%% file: reviewed_jmlr_main.tex
\documentclass[twoside,11pt]{article}
\usepackage[dvipsnames]{xcolor}

\usepackage{jmlr2e}
\usepackage{bm}
\usepackage{enumitem}
\usepackage{algorithm}
\usepackage{algorithmic}
\usepackage{wrapfig}
\usepackage{caption}
\usepackage{changepage}
\usepackage[font=small,labelfont=bf]{caption}

\input{macros}

\newcommand{\Ea}[1]{\E\left[#1\right]}
\newcommand{\Eb}[2]{\E_{#1}\left[#2\right]}

\def\hR{\widehat{\mathcal{R}}}
\DeclarePairedDelimiterX{\abs}[1]\lvert\rvert{\ifblank{#1}{\: \cdot \:}{#1}}
\renewcommand{\vec}[1]{\bm{#1}}

\newcommand{\review}[1]{\textcolor{black}{ #1}}

\newtheorem{assumption}{Assumption}

\firstpageno{1}

\usepackage{lastpage}
\jmlrheading{25}{2024}{1-\pageref{LastPage}}{11/23; Revised
10/24}{11/24}{23-1543}{Yatin Dandi, Florent Krzakala, Bruno Loureiro, Luca Pesce, and Ludovic Stephan}

\ShortHeadings{How Two-Layer Neural Networks Learn, One (Giant) Step at a Time}{Dandi, Krzakala, Loureiro, Pesce, and Stephan}

\begin{document}

\title{How Two-Layer Neural Networks Learn,\\ One (Giant) Step at a Time}

{
\renewcommand{\thefootnote}{\fnsymbol{footnote}}
\author{\name Yatin Dandi\footnotemark[1]\,\,\footnotemark[4] \email yatin.dandi@epfl.ch
        \AND
        \name Florent Krzakala\footnotemark[1] \email florent.krzakala@epfl.ch
        \AND
        \name Bruno Loureiro\footnotemark[2]\email bruno.loureiro@di.ens.fr
        \AND
        \name Luca Pesce\footnotemark[1] \email luca.pesce@epfl.ch
        \AND
        \name Ludovic Stephan\footnotemark[3] \email ludovic.stephan@ensai.fr
        \AND
        \footnotemark[1] \addr Information, Learning and Physics (IdePHICS) Laboratory\\
       École Polytechnique Fédérale de Lausanne\\
       Route Cantonale, 1015 Lausanne, Switzerland 
       \AND
       \footnotemark[2] \addr Département d'Informatique\\
       École Normale Supérieure - PSL \& CNRS\\
       45 rue d’Ulm, F-75230 Paris cedex 05, France 
       \AND 
       \footnotemark[3] \addr Univ Rennes, Ensai, CNRS, CREST \\
       UMR 9194 F-35000 Rennes, France
       \AND 
       \footnotemark[4] \addr
       Statistical Physics Of Computation (SPOC) Laboratory\\
       École Polytechnique Fédérale de Lausanne\\
       Route Cantonale, 1015 Lausanne, Switzerland
       }
       
}

\editor{Mahdi Soltanolkotabi}

\maketitle

\begin{abstract}
For high-dimensional Gaussian data, we investigate theoretically how the features of a two-layer neural network adapt to the structure of the target function through a few large batch gradient descent steps, leading to an improvement in the approximation capacity with respect to the initialization.  First, we compare the influence of batch size to that of multiple (but finitely many) steps. For a single gradient step, a batch of size \( n = \mathcal{O}(d) \) is both necessary and sufficient to align with the target function, although only a single direction can be learned. In contrast, \( n = \mathcal{O}(d^2) \) is essential for neurons to specialize in multiple relevant directions of the target with a single gradient step. Even in this case, we show there might exist ``hard'' directions requiring \( n = \mathcal{O}(d^\ell) \) samples to be learned, where \( \ell \) is known as the leap index of the target. Second, we show that the picture drastically improves over multiple gradient steps: a batch size of \( n = \mathcal{O}(d) \) is indeed sufficient to learn multiple target directions satisfying a staircase property, where more and more directions can be learned over time. Finally, we discuss how these directions allow for a drastic improvement in the approximation capacity and generalization error over the initialization, illustrating a separation of scale between the random features/lazy regime and the feature learning regime. Our technical analysis leverages a combination of techniques related to concentration, projection-based conditioning, and Gaussian equivalence, which we believe are of independent interest. By pinning down the conditions necessary for specialization and learning, our results highlight the intertwined role of the structure of the task to learn, the detail of the algorithm (the batch size), and the architecture (i.e., the number of hidden neurons), shedding new light on how neural networks adapt to the feature and learn complex task from data over time.
\end{abstract}


\section{Introduction}
\label{sec:main:intro}
\input{reviewed_section/intro.tex}
\section{Statement of main theoretical results}
\label{sec:main:theory}
\input{reviewed_section/theory.tex}

%\input{sections/investigation.tex}
% \paragraph{Conclusion ---}
% \label{sec:main:conclusion}
% \input{sections/conclusion.tex}
%%%%%%%%%%%%%%%%%%%%%%%%%%%%%%%%%%%%%%%%%%%%%%%%%%%%%%%%%%%%

\clearpage 

\clearpage 
\appendix
\input{reviewed_section/appendix/numerics.tex}
\input{reviewed_section/appendix/proofs.tex}
\newpage

\vskip 0.2in
\bibliography{biblio}

\end{document}

%% file: macros.tex
\usepackage{mathtools}
\usepackage{etoolbox}
\usepackage{amsmath, amssymb}
\usepackage{xfrac}
\usepackage{wrapfig}

\renewcommand{\vec}{\mathbf}
\newcommand{\dR}{\mathbb{R}}

\newcommand{\dN}{\mathbb{N}}

\newcommand{\dE}{\mathbb{E}}
\newcommand{\dP}{\mathbb{P}}

\newcommand{\R}{\mathbb{R}} % Reals
\newcommand{\N}{\mathbb{N}}
\newcommand{\cA}{\mathcal{A}}
\newcommand{\cB}{\mathcal{B}}
\newcommand{\cC}{\mathcal{C}}

\newcommand{\cH}{\mathcal{H}}

\newcommand{\cL}{\mathcal{L}}

\newcommand{\cN}{\mathcal{N}}
\newcommand{\cO}{\mathcal{O}}
\newcommand{\cP}{\mathcal{P}}
\newcommand{\cR}{\mathcal{R}}
\newcommand{\cS}{\mathcal{S}}

\newcommand{\ind}{\mathbf{1}}

\newcommand{\dS}{\mathbb{S}}

\newcommand{\E}{\mathbb{E}}

\providecommand{\given}{}
\DeclarePairedDelimiterXPP{\Pb}[1]{\mathbb{P}}{\lparen}{\rparen}{}{\renewcommand{\given}{\nonscript{}\:\delimsize\vert\nonscript{}\:\mathopen{}} #1}

\DeclarePairedDelimiterX{\Set}[1]\lbrace\rbrace{\renewcommand{\given}{\nonscript{}\:\delimsize\vert\nonscript{}\:\mathopen{}} #1}

\DeclareMathOperator{\Var}{Var}
\DeclareMathOperator{\Cov}{Cov}

\DeclareMathOperator{\vect}{span}

\DeclareMathOperator*{\argmin}{arg\,min}

\DeclarePairedDelimiterX{\norm}[1]\lVert\rVert{\ifblank{#1}{\: \cdot \:}{#1}}

\DeclareMathOperator{\Unif}{Unif}

\DeclareMathOperator{\He}{He}

\newcommand{\bilingualcommand}[3]{%
	\newcommand{#1}[1][\ ]{%
		##1%
		\iflanguage{english}{\text{#2}}{%
			\iflanguage{french}{\text{#3}}{}%
		}%
		##1%
	}%
}

\bilingualcommand{\where}{where}{où}
\bilingualcommand{\textif}{if}{si}
\bilingualcommand{\textand}{and}{et}
\bilingualcommand{\textiff}{if and only if}{si et seulement si}
\bilingualcommand{\otherwise}{otherwise}{sinon}

\newcommand{\eps}{\varepsilon}

\newcommand{\quand}{\quad \textand \quad}

\newcommand{\bg}{{\boldsymbol{g}}}

\newcommand{\bW}{{{\boldsymbol{{W}}}}}
\newcommand{\bw}{{\boldsymbol{w}}}

\newcommand{\bv}{{\boldsymbol{v}}}
\newcommand{\ba}{{\boldsymbol{a}}}

\newcommand{\bZ}{{\boldsymbol{Z}}}

\newcommand{\btheta}{{\boldsymbol{\theta}}}

\newcommand{\bmu}{{\boldsymbol{\mu}}}

\newcommand{\bz}{{\boldsymbol{z}}}

\newcommand{\bI}{{\boldsymbol{I}}}

\newcommand{\by}{{\boldsymbol{y}}}

\newcommand{\bu}{{\boldsymbol{u}}}

\DeclareMathOperator{\polylog}{polylog}

%% file: reviewed_section/intro.tex
%{\bf TODO  : 
%\begin{itemize}
%\item Mention the link with saddle to saddle, "one (near) saddle at a time".
%\item More ref to information exponent, and leap exponentnt, uniftyin a bit
%\item Kernel/Random feature conjecture needed clarification, add the example we had in the talk
%\item Mention the tradeoff in computation al complexity : with respect to SGD we loose a factor d BUT we win a time order one
%\end{itemize}
%}

A central property behind the success of neural networks is their capacity to adapt to the features in the training data. Indeed, many of the classical machine learning methods, e.g. linear or logistic regression, are specifically designed to a restricted class of functions (e.g. generalized linear functions). Others, such as kernel methods, can adapt to larger function classes (e.g. square-integrable functions), but sometimes at prohibitively many samples. Despite the limitations, these methods enjoy well-understood theoretical guarantees: they are convex (hence easy to train) and given a target function, it is well-understood how many samples are needed to achieve a target accuracy. The situation is dramatically different for neural networks: despite being universal approximators, little is known on how to optimally train them or how many hidden units and/or samples are required to learn a given class of functions. Nevertheless, they have proven to be flexible, efficient and easy to optimize in practice, properties which are often attributed to their capacity to adapt to features in the data. Curiously, most of our current theoretical understanding of neural networks stems from the investigation of their lazy regime where features are {\it not} learned during training. In this work, we take some (giant) steps forward from the lazy regime.

% A number of theoretical efforts have been directed toward understanding the dynamics of training neural networks with gradient descent, aiming to uncover the underlying principles that govern their learning capabilities. By iteratively adjusting the network's weights based on the gradients of a chosen objective function, gradient descent seeks to minimize the discrepancy between predicted and true outputs. In this process, the network adapts to the data, learning which features in the input data are most important for explaining the labels. This adaptivity is in contrast to other machine learning pipelines such as random feature or kernel methods,  and is behind the success of neural networks in efficiently solving problems from image classification to text generation.

%In this paper, we delve into the training dynamics of shallow neural networks, focusing on the conditions that facilitate effective feature learning beyond the kernel or random features regime. We aim to answer the following fundamental question: can we precisely capture the presence of feature learning in the {\it early phase} of gradient descent training? We shall thus investigate how the {\it very first few gradient steps} on the first-layer parameters impact the representation abilities. Given we only perform  a few steps, it will have to be giant ones, and we will use large, aggressive, learning rates.

% ADD SOMETHING ALONS THE LINE OF THE 

Our central goal is to paint a complete picture of how two-layer neural networks adapt to the features of training data $(\vec{z}^\nu, y^\nu)_{\nu=1}^{n}\in\mathbb{R}^{d+1}$ in the {\it early phase} of training after the first few steps of gradient descent. We recall the reader that for data is supported in a high-dimensional space, the curse of dimensionality prevents efficient learning even under standard regularity assumptions on the target function such as Lipschitzness \citep{devroye2013probabilistic}.
Hence, understanding the efficient learning performance of neural networks observed in practice requires additional assumptions on the data distribution. In this work, we focus on a popular synthetic data model consisting of: a) independently drawn standard Gaussian covariates $\vec{z}^{\nu}\sim\cN(0, I_d)$; b) a target function $y^{\nu}=f^\star(\vec{z}^{\nu})$ depending only on a finite number of relevant directions, also known as a \emph{multi-index model}. In other words, there exists a finite number of orthogonal \emph{teacher vectors} $\vec{w}_1^\star, \dots, \vec{w}_r^\star$ such that
    \begin{equation}
        y = f^\star(\vec{z}) \coloneqq g^\star(\langle\vec{w_1}^\star,\vec{z}\rangle,\dots, \langle\vec{w_r}^\star,\vec{z}\rangle).
    \end{equation}
 Note that in this model the features are isotropic, with all the structure in the data being in the target. In contrast to other popular models for structured data , such as low-dimensional support of the inputs or smoothness of the target function, kernel methods do not adapt to target functions depending on low-dimensional projections \cite{bach2017breaking}. This makes it an ideal playground for quantifying the adaptativity of neural networks in the feature-learning regime. Given this class of structured targets, we consider supervised learning with the simplest universal approximator neural network: a fully-connected two-layer network with first and second layer weights $W \in \dR^{p\times d}$ and $\vec{a}\in \dR^p$ and activation function  $\sigma:\mathbb{R}\to\mathbb{R}$:
\begin{align}\label{eq:def_2lnn}
    \hat f(\vec{z}; W, \vec{a}) = \frac1{\sqrt{p}}\sum_{i = 1}^{p} a_i \sigma(\langle \vec{w}_i, \vec{z} \rangle).
\end{align}
The primary focus of this work is to elucidate how a two-layer neural network adapts to a low-dimensional target structure during its training. We aim to understand the interplay among the structure of the task (specifically, the complexity of the hidden true function), the details of the algorithm (here the batch size), and the architecture (the number of hidden neurons) in the process of learning from data \citep{zdeborova2020understanding}. In particular, we will be interested in quantifying how much data is required for the relevant directions of the target to be learned, and how this feature learning translates into the approximation capacity of the network with respect to kernel methods.

%\fk{Few words on the fact that we do single pass, and batch size $n=d^l$, Maybe motivate the setting by federated learning and modern practice, maybe cite work such as 
%\cite{goyal2017accurate,li2020review} and others.}.... 

%Our results highlight the fundamental role of batch-size in feature learning, and its interplay with the number of epochs and the complexity of the underlying target function. We uncover the presence of a hierarchy in the way two-layer neural networks learn different relevant directions of the target, in terms of their relative weights with respect to the Hermite decomposition of the target, as illustrated in Fig.~\ref{fig:stairway} (for a single steps) and Fig.~\ref{fig:stairway2} (for multiple steps). More precisely, our results can be summarized as follows.

%\bl{Expand this in terms of "learning representations" vs. "learning features".}

%\newpage
\section{Summary of main results}
\begin{wrapfigure}{r}{0.45\textwidth}
  \vspace{-1.cm}
\includegraphics[width=0.44\textwidth]{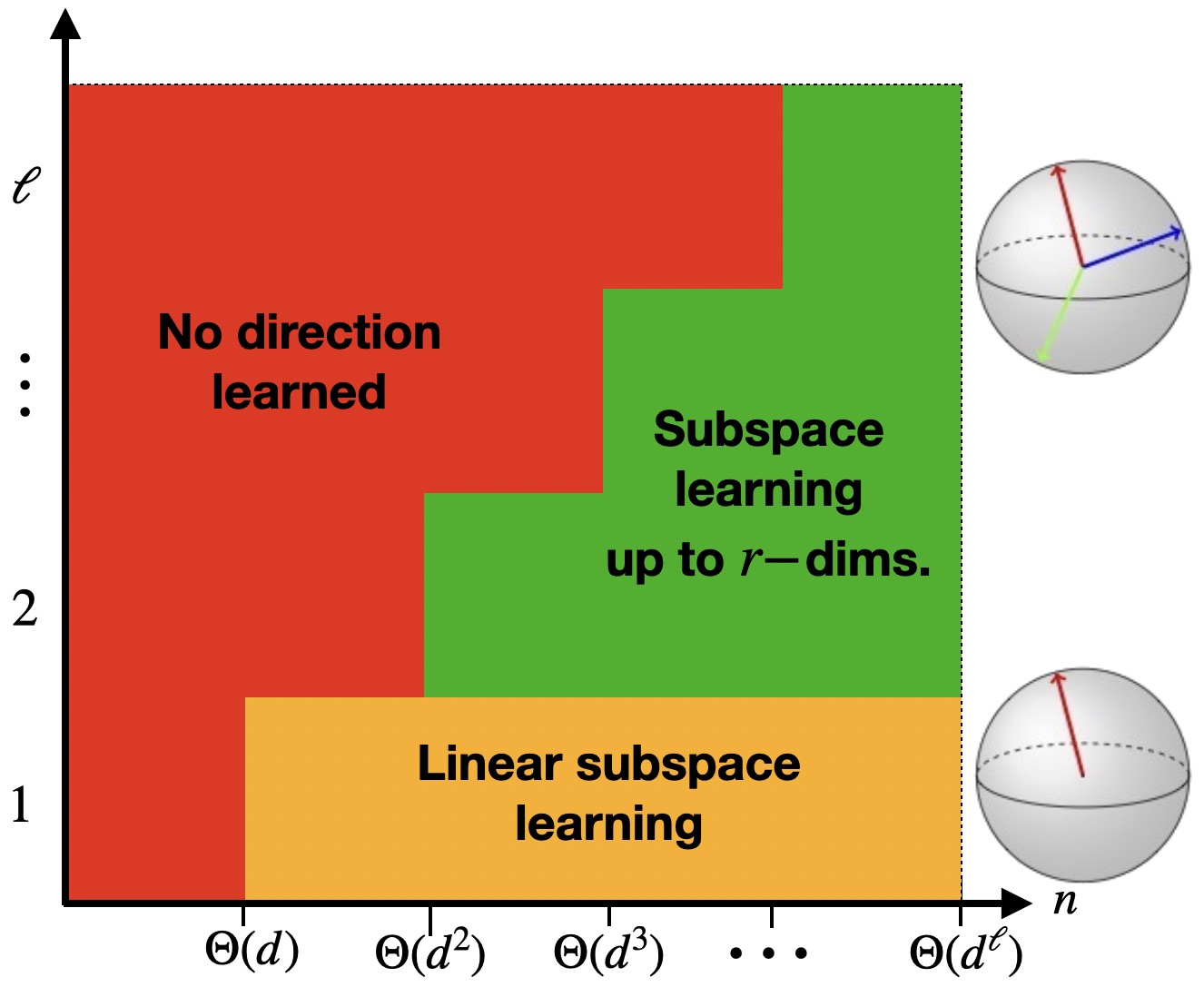}
  \caption{\textbf{Learning with a single gradient step.} 
  Illustration of the relationship between batch size and target function complexity for learning multi-index functions with a single giant step in the  $n=\Theta(d^k)$ regime (Theorems.~\ref{thm:one_step_lower_bound} \& \ref{thm:one_step_learning}). 
  % What can be learned after a single giant gradient steps with $n=\Theta(d^r)$ samples (see text and theorems \ref{thm:one_step_lower_bound} \& \ref{thm:one_step_learning}).
  }
  \label{fig:fig_1}
  \vspace{-0.5cm}
\end{wrapfigure}
%In section \ref{res:generalisation} we show that 
To outperform the network at initialization (which can be regarded as a kernel method) the network must adapt to the data distribution. Mathematically, this translates to developing correlation in the first layer weights with the target directions $\vec{w}^{\star}_{1},\ldots,\vec{w}^{\star}_{r}$. Our first set of results thus precisely focus on {\bf feature learning}, i.e. how the subspace
\[ V^\star \coloneqq \vect(\review{\vec{w}^\star_1, \dots, \vec{w}^\star_r})\]
is learned during training.

\subsection{A single Gradient step} 
First, we discuss the case of {\bf a single gradient step} of {\bf full-batch gradient descent}, which turns out to be already non-trivial \citep{ba2022high,damian2022neural} 
 and consider the update: \begin{equation}\label{eq:first_layer_gd1}
     \vec{w}_i^{1} = \vec{w}_i^0 -  \frac{\eta}{2n} \sum_{\nu = 1}^n \nabla_{\vec{w}_{i}} \left(y^\nu-\hat f(\vec{z}^\nu; W^0, \vec{a}^{0})\right)^{2}\,.
    \end{equation} 
Theorems \ref{thm:one_step_lower_bound} and \ref{thm:one_step_learning} identify a fundamental interplay between batch size and the complexity of the underlying target function. They are summarized in Fig.~\ref{fig:fig_1} (and a particular numerical example is shown in Fig.~\ref{fig:example}). More precisely:

\begin{itemize}
[noitemsep,leftmargin=1em,wide=0pt]
\item We show that developing meaningful correlation with the target function requires a large batch size $n = \cO(d)$ and learning rate $\eta=\Theta(p)$ when $p,d,n$ are large. However, feature learning remains limited in this regime since we prove only a {\it single} direction can be learned. Thus, if the target depends on several relevant directions, only a "single neuron" approximation can be learned.
%\end{itemize}
%\vspace{-0.4cm}

%\begin{itemize}[noitemsep,leftmargin=1em,wide=0pt]
\item Surpassing the single direction approximation with a single step {\it requires} a larger batch size with {\it at least} $n = \cO(d^2)$ samples. This allows for the network weights to {\it specialize} to multiple target directions. 
%Applying a pre-processing that removes the main learned direction, we show that with a batch size of $n = \cO(d^2)$ the gradient spans the space of many directions, allowing for the network weights to {\it specialize} to multiple target directions. 
\item Nonetheless, we show that there might be {\it hard} directions in the target which cannot be learned with $n = \cO(d^2)$. Learning these directions necessitates a batch size of at least $n = \cO(d^\ell)$, as well as suppressing the directions learned at  $n = \cO(d^{\ell-1})$, where $\ell$ is the \emph{leap index} of the target (precisely defined in Def. \ref{def:leap}) that \review{informally} speaking corresponds to the lowest non-zero degree of the Hermite polynomials in the expansion of the target in this directions.  
%\item This result allows to contrast the performances of SGD with full-batch SGD. Indeed, \cite{BenArous2021} showed that learning a single-index ($k=1$) target function with one-pass SGD requires $n=\tau=\cO(n^{\ell-1})$ samples and iterations, where the \emph{information exponent} $\ell$ of the target is given by the degree of the first non-zero polynomial in its Hermite decomposition. In contrast, here we see that a full batch gradient  (which can be efficiently paralelized by distributing over many computers) can achieve the same results {\it in a single iteration}, at the price of using $n=\cO(d^{\ell})$ samples instead.
\end{itemize}
This description paints a clear picture on how the role of the batch size, of the learning rate, and the structure of the hidden function are intertwined. The complexity of learning a low-dimensional target function  is a topic that recently saw a surge of interest, and it is thus interesting to contrast these rates with recent results in the litterature.

For a single-index function with leap index $\ell$, one-pass SGD has a sample complexity of $\cO(d^{\ell - 1})$ \citep{BenArous2021}, and it has been recently shown that a smoothed version of SGD achieves a sample complexity of $\cO(d^{\ell/2})$ \citep{damian_2023_smoothing}. This matches a lower bound from the correlation statistical query family, which encompasses all SGD-like methods. In our single step setting,  the sample complexity for large batch learning is $\cO(d^\ell)$, which is worse than both the aforementioned methods. However, the time complexity of each algorithm paints a different picture. Both SGD algorithms are sequential, and require $\cO(d)$ operations per step, which leads to a total time complexity of $\cO(d^{\ell/2 + 1})$ at least. On the other hand, the computation of the update in Eq. \ref{alg:gd_training} is simply an average of independent terms, which is easy to parallelize. Including the time to compute the average of each term e.g. using a \emph{Gossip algorithm} \citep{boyd_2006_randomized}, this sums up to a time complexity of $\cO(d + \log(n))$. Such an algorithm is also amenable to decentralized learning schemes, where each agent only has access to a fraction of the overall data. 

\subsection{Learning over many GD iterations}
The situation drastically improves when taking for {\bf multiple gradient steps}. In this case, focusing on the linear batch size $n={\cal O}(d)$ regime, and using a fresh batch of data at each GD iteration: 
\begin{equation}\label{eq:first_layer_gd}
        \vec{w}_i^{t+1} = \vec{w}_i^t -  \frac{\eta}{2n} \sum_{\nu = 1}^{n} \nabla_{\vec{w}_i} \left(y^\nu-\hat f(\vec{z}^\nu; W^t, \vec{a}^{0})\right)^{2}\,,
    \end{equation}
Note that splitting the training of the first and second layers and using a fresh batch of data at each iteration is a common approximation in this literature \citep{damian2022neural, ba2022high, Bietti2022}. In contrast to the recent works considering the population limit, we stress that here we take the batch size $n$ to scale with the dimension $d$. In the realm of distributed and federated learning, scenarios with large batches, a single pass, and few iterations are often the norm \citep{goyal2017accurate,li2020review} (for instance this is the case when training large language models), further underlining the relevance of this scenario. In this case, Theorem \ref{thm:staircase} shows that more complex subspaces of the target directions {\it may} be progressively learned at each iteration, as we illustrate in Fig. \ref{fig:fig_2}. More precisely:

\begin{itemize}[noitemsep,leftmargin=1em,wide=0pt]
\item Each additional gradient step allows for learning new perpendicular directions {\bf upon the important condition that they are linearly connected to the previously learned directions} (see \review{Sec.~\ref{sec:few-giant-steps}} for precise definitions of this staircase property). Therefore, in contrast with a single step, taking multiple steps allows to learn a multiple-index target with only $n={\cal O}(d)$ samples.
\item Nonetheless, directions that are not coupled through the staircase property and with zero first Hermite coefficient cannot be learned in any finite number of steps. In fact, as discussed previously, they require a batch size of at least $n={\cal O}(d^2)$. In other words, while multiple steps help specialization, it cannot help learning ``hard'' target directions.
\end{itemize}

\begin{figure}[t]
\centering
\includegraphics[width=0.95\textwidth]{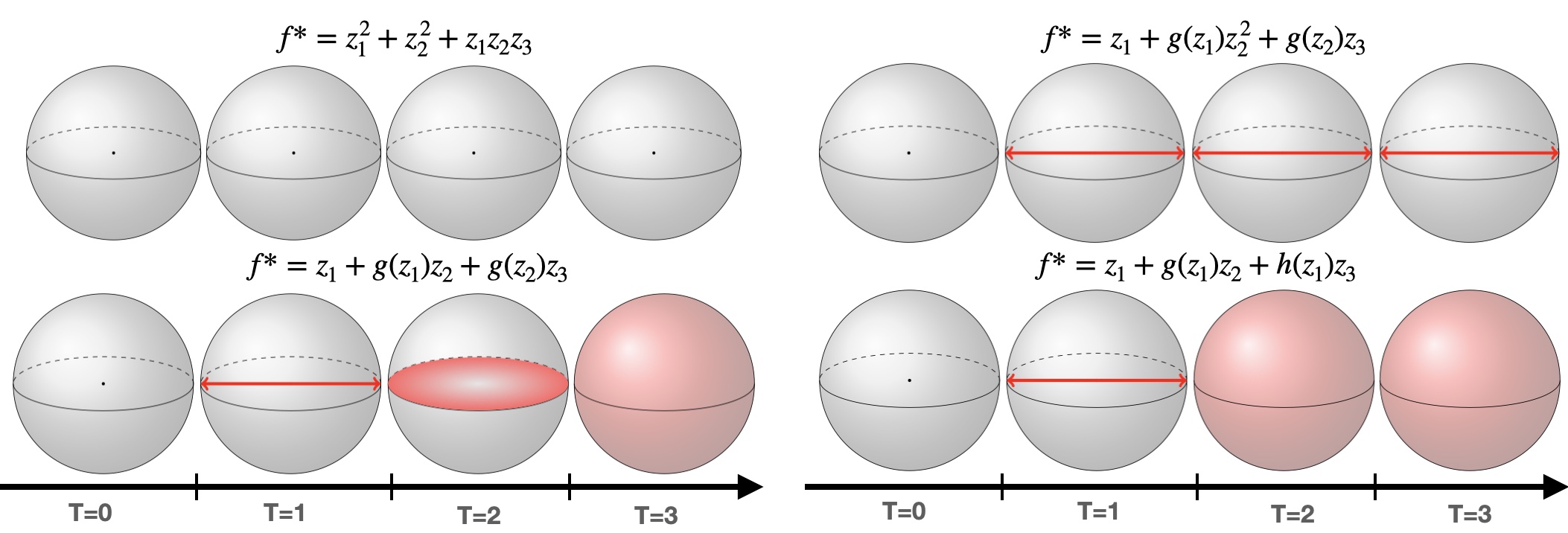}
\caption{\textbf{Learning with multiple gradient steps.} An illustration of how neural networks trained using $n=\cO(d)$ batches learn relevant directions after multiple GD training steps, allowing to learn more complex functions over time (see Thm. \ref{thm:staircase}). Here we represent the space $V^\star$ of the relevant direction of the target function, and the (normalized) projection of learned direction $W^t$ by the networks for different task $f^\star(\cdot)$ - the function $g(\cdot)$ and $h(\cdot)$ are assumed to have zero first two Hermite coefficients. During the early stage of training, neural networks first learn first a single direction associated to the linear part of the target, and then can learn over time other directions that are linear conditioned on the previous learned ones. Let $\{\vec e_i\}_{i \in [d]}$ be the standard basis of $\mathbb{R}^d$, the four examples show cases where: \textbf{Top left:} No directions is learned at all. \textbf{Top right:} The network can only learn a single direction $\vec e_1$ (single index regime). \textbf{Bottom left:} The network learns a new direction each time, $\vec e_1$, then $\vec e_2$ and finally $\vec e_3$. \textbf{Bottom right:} The network learns $\vec e_1$ at the first step and both $\vec e_2$ and $\vec e_3$ at the second steps.}
\label{fig:fig_2}
\end{figure}

These results warrant the following comments in context of the state of art. \review{The staircase functions were introduced and analyzed for Boolean covariates in \cite{abbe2021staircase,abbe2022merged}. In particular, \cite{abbe2022merged} considered one-pass SGD in two settings: (a) $\mathcal{O}(d)$ iterations and batch-size one; (b) $\mathcal{O}(1)$ iterations with batch-size $\mathcal{O}(d)$ - the latter being closer to the multiple gradient steps setup considered in this manuscript.
These works employ a dimension-free characterization of the mean-field limit \citep{chizat2018global,mei2018mean,rotskoff2018trainability,mei2019mean} with diverging width (constant with respect to $d$) to show that the staircase structure is necessary and nearly sufficient for learning the target with $\mathcal{O}(d)$ sample complexity. Our work differs in two important points: we consider Gaussian data and carry out a basis-independent analysis}

\review{Recently, Gaussian data was also considered by \cite{abbe2023sgd}, who extended the martingale analysis of \cite{BenArous2021}, showing that leap $1$ staircase target functions are learned by one-pass SGD with $\mathcal{O}(d)$ iterations / samples. Using, as we do here, large batches $n_d=O(d)$ gives the same  $O(d)$ dependence in terms of sample complexity, but allows us to learn them with $O(1)$ iterations instead. This is a nice illustration of the speed-up provided by large-batch SGD over vanilla SGD (see Table \ref{table} for a summary).}

% \comment{Need to take this paragraph back as Abbe has already $n_b = O(d)$ (Theorem 9 of their $2022$ paper. Rephrase highlighting our differences.) [Abbe introduces staircase function and then mention $2022$ [mean field O(1) neurons but for boolean and then go through for $2023$ for Gaussian. Analogous to  Theorem ... of Abbe 2022,our analysis focuses on the large batch O(d) setting, but for Gaussian data.] }

Secondly, we note that similar to the \emph{saddle-to-saddle} dynamics under gradient flow \citep{jacot2021saddle,abbe2023sgd,boursier2022gradient}, the dynamics described through Theorem~\ref{thm:staircase} involves sequential learning of directions.  We note, however, that in the one-sample SGD regime \citep{BenArous2021,abbe2023sgd} the gradient has vanishing correlation with new directions, thus requiring a polynomial number of updates to escape saddles. In contrast, the large-batch gradient updates contain a finite fraction of components along the new directions, allowing their learning through a single step. Moreover, we show that each update leads to a $\cO(1)$ change in the components along \review{directions} in $V^\star$, obviating the need for coordinate-wise projections in \cite{abbe2023sgd}.

The set of results described above provide a mathematical theory on how neural networks learn representations of the data over training. They corroborate, among others, the findings of \cite{kalimeris2019sgd}, who observed that neural networks \review{first fit the best linear classifier and subsequently learn} functions of increasing complexity.

\begin{figure}[t]
    \centering    
\includegraphics[width = 0.48\textwidth]{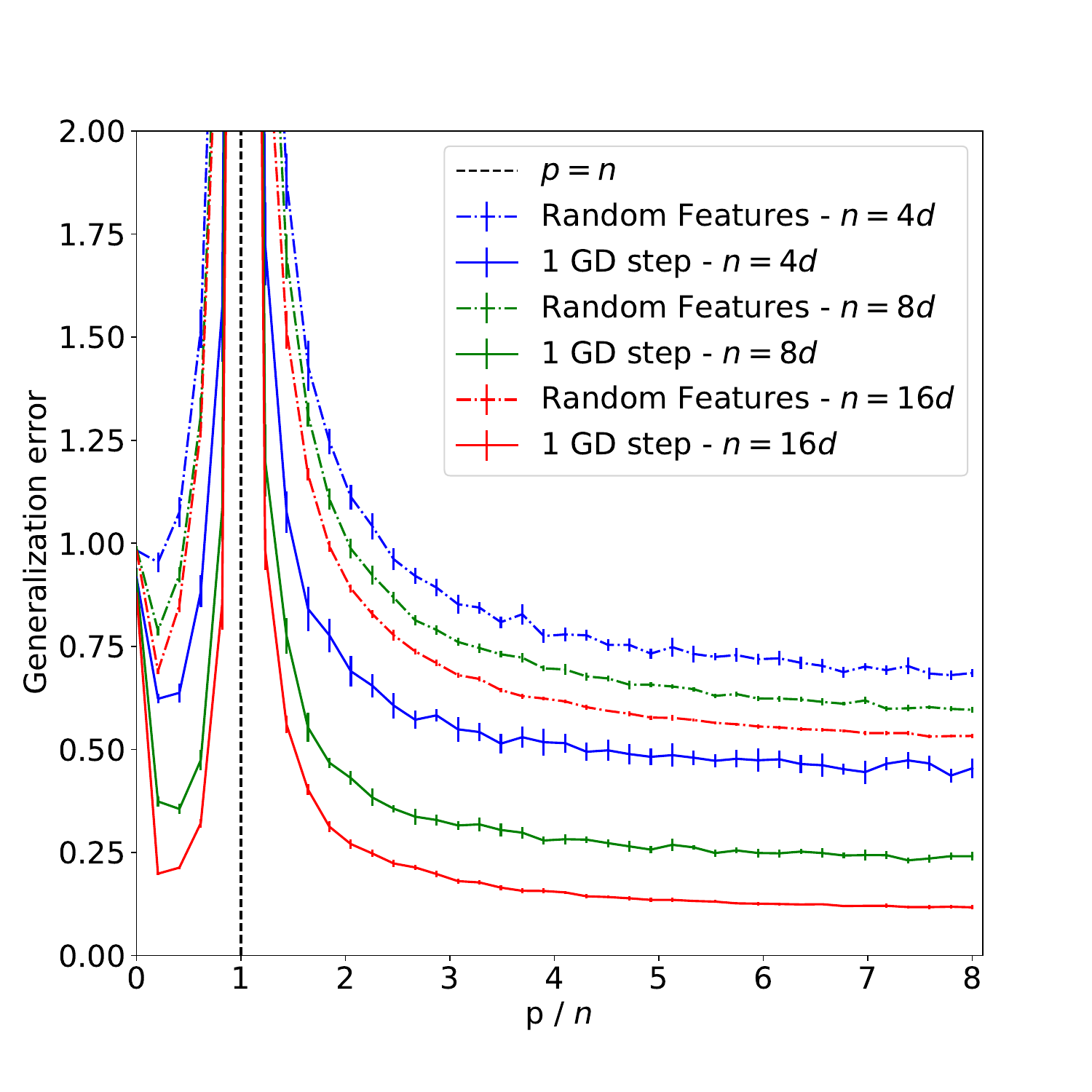}
\includegraphics[width = 0.48\textwidth]{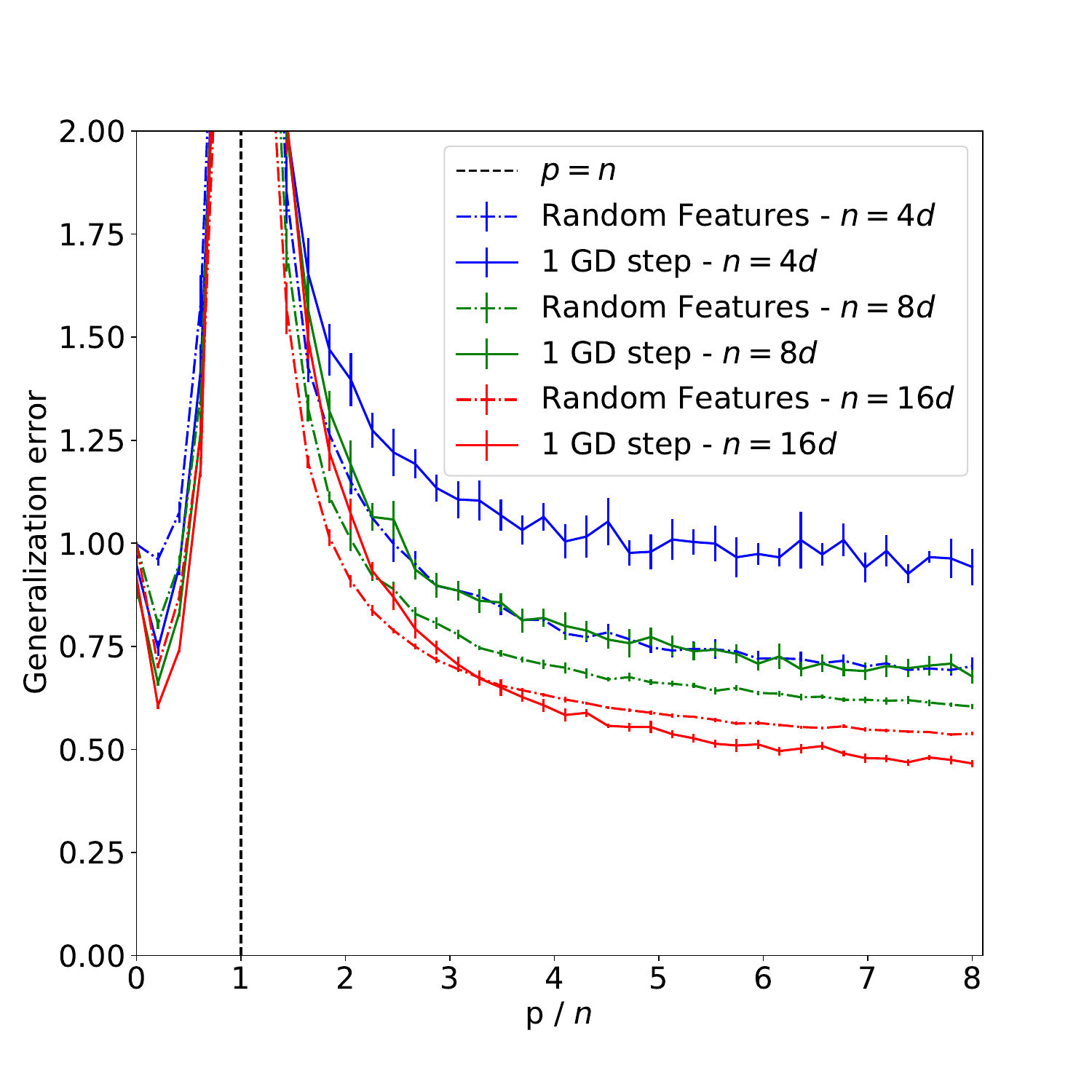}
    \caption{{\bf Feature learning and generalization:} We illustrate how one step of gradient descent may (or may not) improve generalization over random features. The plot shows the generalization error as a function of the number of hidden neurons $p$ normalized by the number of samples used for the first layer training $n=\{4d,8d,16d\}$ with fixed dimension $d=2^8
$. 
    \textbf{Left}: $f^\star(\vec z)=z_1+z_1 z_2$. While random features can only fit a linear model in the proportional regime considered, one step of gradient descent over $W$ allows to fit the $z_1 z_2$ part, resulting in a much lower generalization error with respect to random features, despite only the direction $z_1$ being learned in $W$.
    \textbf{Right}: $f^\star(\vec z) = z_1+ z_2z_3$. In this case, since the nonlinear part does not depend on $z_1$, one step of gradient descent on the $W$ does not allow to improve generalization over random features (see Theorem \ref{thm:cget} and Corollary \ref{corr:lower_bound}). We refer to App.~\ref{sec:appendix:numerics} for details on the numerics.}
    \label{fig:features}
\end{figure}

\subsection{From features to generalization}
Our last set of results connects feature learning with the approximation capacity of the network and illustrates that feature learning improves the learning of the target function $f^{\star}$ over random initialization.

\begin{itemize}[noitemsep,leftmargin=1em,wide=0pt]
\item We show that a two-layer network with a finite second layer can {\it only} learn the part of the target function in the learned subspace (see Proposition \ref{prop:fixed_width_lower_bound}). In fact, we conjecture that with $p$ large enough (but still finite), it should be possible to approximate this part of $f^{\star}$ up to arbitrary precision (Conj. \ref{conj:approxim_conj}).
\item As (possibly) universal kernels, large-width two-layer networks at initialization \review{enjoy} better approximation capacity. Indeed, \cite{mei_generalization_2022} proved that at initialization $W^{t=0}$, with $n=\Theta(d^k)$ only a degree $k$ approximation of the target function can be learned. Our results characterize how feature learning allows us to improve over this sample complexity. In particular, we show that in the directions learned by the first layer, the target function $f^{\star}$ can be learned with less samples, while the component of $f^{\star}$ orthogonal to the learned features still requires $n=\cO(d^{k})$. While a complete mathematical control of the generalization error rates remains a difficult problem (see Conjecture \ref{conj:kernel+feature}), our results provide a clear separation on scales between two-layer networks and NTK-like methods, improving over the state-of-the art in the literature. 

\item In particular, in Corollary \ref{corr:lower_bound} we prove that with a single gradient step and $n,p=\mathcal{O}(d)$, one can {\it only} learn features in this one-dimensional subspace, doing no better than kernels in the orthogonal direction. This is illustrated in Fig.\ref{fig:features} where we give an example where one step of gradient drastically improves generalization, and one where it does not. To prove these results, we provide a stronger {\it conditional} version of the Gaussian equivalence theorem \citep{mei_generalization_2022,goldt_gaussian_2021,hu2022universality}, which is the backbone of Theorem \ref{thm:cget}, and we believe is of independent interest. 
\end{itemize}

The code to reproduce our figures is available on   \href{https://github.com/lucpoisson/GiantStep}{GitHub}, and we refer to App.~\ref{sec:appendix:numerics} for details on the numerical implementations. Proofs are detailed in App.~\ref{sec:appendix:gd_proofs} and App.~\ref{sec:appendix:cget_proofs}.

\paragraph{Other related works ---}
The analysis of high-dimensional asymptotics of kernel regression has provided valuable insights into the advantages and limitations of kernel methods \citep{Dietrich1999, Opper2001, Ghorbani2019, Ghorbani2020, Donhauser2021, Mei2023, Spigler2020, bordelon20a, Canatar2021, simon2022eigenlearning, Cui2021, Cui2022, xiao2022precise}. In particular, a similar stairway picture as in Fig.~\ref{fig:fig_2} emerged from these works \cite{xiao2022precise}. The key difference, however, is that at each regime $n={\cal O}(d^{k})$, kernels can only learn the $k$-th Hermite polynomial of the target. This should be contrasted with our feature learning regime where, once a direction is learned, all its Hermite coefficients are learned. Feature learning corrections to kernel methods have been investigated in \cite{pmlr-v75-dudeja18a, Naveh2021, Seroussi2023, atanasov2022neural, Bietti2022, bordelon2023dynamics, petrini2022learning}. On a complementary line of work, exact asymptotic results for the the random features models have been derived in the literature \citep{mei_generalization_2022, gerace_generalisation_2020, Dhifallah2020, loureiro_learning_2021, loureiro22a, schroder2023deterministic, bosch2023precise}. A large part of these results are enabled by the Gaussian equivalence property \citep{goldt_gaussian_2021, hu2022universality, Montanari2022, dandi2023universality}.

Closer to us are \citep{ba2022high, damian2022neural}. 
% \cite{ba2022high} showed that a sufficiently large single gradient step allows to beat kernel methods. 
\cite{ba2022high}
showed that a single gradient step yields an approximately rank-one change on the weights which is enough to beat kernel methods, but did not characterize the impact on the generalization error. In our work, we prove that with a single gradient step and $n,p=\mathcal{O}(d)$, one can only learn features in this one-dimensional subspace, doing no better than kernels in the orthogonal direction. Additionally, their results are limited to single-index target and to a single gradient step. 
In contrast, \cite{damian2022neural} showed that with $n = \omega(d^2)$ samples, two-layer neural networks can specialize to more than one direction of a multi-index target function with zero first Hermite coefficient ($\ell$=2), and were again limited to a single step.
Our work extends their sufficient conditions on the sample complexity to general $\ell\geq 1$. We also show they are also necessary, i.e. $\forall \epsilon >0$, with less data $n = \Theta(d^{\ell-\epsilon})$ one cannot do better than random features. Thus, our results prove a clear separation between the class of functions learned within the $\Theta(d)$ batch-size setting of \cite{ba2022high} and the $\Theta(d^2)$ batch-size setting of \cite{damian_2023_smoothing} and establish a general hierarchy of functions requiring increasing batch-size to be learned with a gradient step. Additionally, we characterize which class of multi-index targets can be instead learned with $n=\cO(d)$ with {\it multiple} steps.

%Here, we build on their results and consider multi-index targets and multiple steps. \cite{damian2022neural} proved that a batch size of $n=\omega(d^2)$ is \emph{sufficient} for learning multiple directions of a multi-index target. However, they did not establish whether this batch size is \emph{necessary}, and were again limited to a single step. Our results reveal that a batch size of $n=\cO(d^2)$ is indeed necessary for learning multiple directions in a single step, and we characterize which class of multi-index targets can be instead learned with $n=\cO(d)$ with multiple steps.

\review{In a related but different vein, \cite{abbe2022merged, abbe2021staircase} showed how the ``staircase" property of target functions characterizes the sample complexity for two-layer networks trained with one-pass SGD, both for small and large batch sizes. Their focus, however, was on the case of {\it sparse boolean functions}. Recently, \cite{abbe2023sgd} provided a partial extension of these results to two-layer neural networks trained with batch one SGD on isotropic Gaussian data. Our work differs in important points.
First, we consider general multi-index target functions on isotropic Gaussian data. Second, we consider large batch SGD. Our Theorem 7 operates in a similar setting as Theorem 9 in \cite{abbe2022merged} with $p=\mathcal{O}(1)$, but without assuming the knowledge of the basis or the validity of the mean-field limiting equations.
Our result shows that a ``directional staircase" behavior arises when iterating a few giant gradient steps, while a related, but different, picture arises with a single step depending on the batch size.} We also provide a sharp characterization of when this phenomenon appears for multi-index targets and networks trained under large batch SGD, and provide a bound on the resulting generalization error. \review{Furthermore, our Theorem 12 precisely characterizes the effect of $p=\mathcal{O}(d)$ neurons for a single gradient step.} 
\\ Akin to the ``summary statistics" approach in \cite{saad.solla_1995_line,BenArous2021,ben2022high}, our analysis is based upon the concentration of the overlaps of the neurons with the target subspace and their norms, instead of the concentration of the full gradient vector considered in recent works such as \cite{abbe2022merged,damian2022neural}, removing any requirements on the constants in the sample complexity
for alignment along the target subspace. 

%% file: reviewed_section/theory.tex
\subsection{Preliminaries}
%\bl{I think this should go to the appendix}

Before stating our main results, we recall a few \review{definitions} and useful facts.
\paragraph{Hermite expansion ---}  Given the Gaussian measure $\gamma_m$ on $\dR^m$, we can build a scalar product on $\ell^2(\dR^m, \gamma_m)$ as
\begin{equation}
    \langle f, g \rangle_\gamma = \int_{\dR^m} f g \ \mathrm{d} \gamma_m = \dE_{\vec{z}\sim \cN(\vec{0}, I_m)}[f(\vec{z})g(\vec{z})].
\end{equation}
It turns out that there is a specific orthonormal basis of interest for this scalar product, that we present in tensor form:
\begin{definition}[Hermite decomposition]\label{def:hermite}
    Let $f: \dR^m \to \dR$ be a function that is square integrable w.r.t the Gaussian measure. There exists a family of tensors \review{$(C_j(f))_{k\in\dN}$} such that $C_j(f)$ is of order $j$ and for all $\vec{x} \in \dR^m$,
    \begin{equation}
        f(\vec{x}) = \sum_{j \in \dN} \langle C_j(f), \cH_j(\vec{x}) \rangle
\label{eq:hermite_expansion}
    \end{equation}
    where $\cH_j(\vec{x})$ is the $j$-th order Hermite tensor \citep{grad_1949_note}.
\end{definition}

\paragraph{Higher-order singular value decomposition ---}

The higher-order singular value decomposition (HOSVD) \citep{de2000multilinear} of a tensor is defined as follows:
\begin{definition}[Higher-order SVD \citep{de2000multilinear}]
    Let $C \in \dR^{m^k}$ be a symmetric tensor of order $k$. A higher-order SVD of $C$ is an orthonormal set $(\vec{u}_1, \dots, \vec{u}_r)$ of \review{$r \leq m$} vectors, as well as a tensor $S \in \dR^{r^k}$ such that
    \begin{equation}
        C = \sum_{j_1, \dots, j_k = 1}^r S_{j_1, \dots, j_k} \vec{u}_{j_1} \otimes \dots \otimes \vec{u}_{j_k},
    \end{equation} 
where $r$ is chosen to be minimal.
\end{definition}
\review{The rank $r$ and the singular values tensor $S$ are unique while, as in the case of SVD for matrices, the vectors $(\vec{u}_1, \dots, \vec{u}_r)$ are only unique up to signs, and in the case of identical singular values up to rotations within the corresponding singular value subspace. Furthermore, the core tensor $S$ satisfies all-orthogonality \citep{de2000multilinear}.
}

% \comment{Reviewer 1 asks for additional clarity on this definition: a) is $m<k$ and not $r<k$?; b) is r minimal in some sense?; c) why also higher-order SVD unique up to rotations. [Yes $r$ is minimal. Add reference and correct that is only true up to permutations and sign not rotations] }
\subsection{Setting and assumptions}
% We consider fully connected neural networks with $p$ hidden neurons: given $W \in \dR^{p\times d}$ and $\vec{a}\in \dR^p$, the network function is defined as
% \begin{equation}\label{eq:def_2lnn}
% \hat f(\vec{z}; W, \vec{a}) = \frac1{\sqrt{p}}\sum_{i = 1}^{p} a_i \sigma(\langle \vec{w}_i, \vec{z} \rangle) = \frac{1}{\sqrt{p}} a^\top \sigma(W \vec{z}),
% \end{equation}
% where $\sigma$ is the non-linear activation function applied elementwise. We study the performance of learning a distribution $\rho$ on pairs $\vec{z}, y$, which corresponds to minimizing the \emph{population loss}
% \begin{equation}
%     \cR(W, \vec{a}) = \dE_{(\vec{z}, y) \sim \rho}\left[ \cL(\hat f(\vec{z}; W, \vec{a}), y) \right],
% \end{equation}
% where $\cL$ is the chosen loss function. In this work, we will focus on the case of the square loss
% \[ \cL(\hat y, y) = \frac12 \left(\hat y - y \right)^2 \]
Before stating our main results, we introduce the setting and main assumptions required. The first concerns the class of target functions we consider.
\begin{assumption}[Data model]
\label{ass:data}
The training inputs $\vec{z}^{\nu}\in\mathbb{R}^{d}$ are independently drawn from the Gaussian distribution $\cN(0, I_d)$. Further, we assume that the target function $y^{\nu}=f^{\star}(\vec{z})$ depends only on a few relevant directions. In other words, there exists a fixed number of orthonormal vectors $(\vec{w}_1, \dots, \vec{w}_r)$ and a fixed function $g^\star: \dR^r \to \dR$ such that  
    \begin{equation}
        y = f^{\star}(\vec{z}) \coloneqq g^\star(\langle\vec{w_1}^\star,\vec{z}\rangle,\dots, \langle\vec{w_r}^\star,\vec{z}\rangle).
    \end{equation}
\end{assumption}
As we will show later, learning with GD can be seen as a hierarchical process, where depending on the batch size different directions of the target are progressively learned. Next, we define the leap index, a fundamental quantity which precisely parametrizes what are the first directions to be learned.
\begin{definition}[Leap index] \label{def:leap}
Since the input data is Gaussian $\vec{z}\sim\mathcal{N}(0,I_{d})$, the target function admits a decomposition in terms of the Hermite decomposition (see Definition \ref{def:hermite}). We define the \emph{leap index} of the target function $f^{\star}$ as the first integer $\ell > 0$ such that \review{ $C_\ell^{\star} = C_\ell(f^\star) \neq 0$}:
\begin{align}
   \ell =  \min \{j\in\mathbb{N}: \langle f^{\star},\cH_j\rangle_{\gamma}\neq 0 \}
\end{align}
\end{definition}
% \comment{[Refer to Information exponent \cite{BenArous2021}, this is just its generalization. Refer to IsoLeap of \cite{abbe2022merged} and our results in next section.]}
\review{For single-index models, the above definition reduces to Information Exponent defined in \cite{BenArous2021}. A generalization of the above exponent to sequential learning of directions is defined in \cite{abbe2023sgd} under ``Leap complexity" and ``Isotropic Leap complexity". % We discuss such sequential learning in Section \ref{sec:few-giant-steps}.
} Given a batch of training data $(\vec{z}^\nu, y^\nu)_{\nu=1}^{n}\in\mathbb{R}^{d+1}$ drawn from the model \eqref{ass:data} defined above, we now define how the network weights $(W,a)$ are initialized and updated. 
\begin{assumption}[Training procedure]\label{ass:training}
Consider the following random initialization for the weights:
\begin{equation}\label{eq:sample_archit}
    \sqrt{p} \cdot a_i^0 \overset{i.i.d}{\sim} \Unif([-1, 1]) \!\!\quand\!\! \vec{w}_i^0  \overset{i.i.d}{\sim} \Unif(\cS^{d-1}).
\end{equation}
The distribution of the $a_i$ can be replaced by any other continuous distribution with positive variance. Note that for $p=\cO(1)$, we have $\hat f(\vec{z}; W^{0}, \vec{a}^{0})\neq 0$. To further simplify the analysis, we assume $p$ to be even and further impose the following symmetrization at initialization: 
\begin{equation}
    a_i^0 = -a_{p-i+1}^0 \quand \vec{w}_i^0 = \vec{w}_{p-i+1}^0 \quad \text{for all } i \in [p/2],
\end{equation}
which ensures $\hat f(\vec{z}; W^{0}, \vec{a}^{0})=0$. Note that this simplification is common in the related literature, e.g. \cite{chizat_2019_lazy, damian2022neural}, and is mainly necessary when $p$ is small.
% We consider a symmetric initialization such that the network function is identically zero at the start of training. More precisely, assuming $p$ is even we take:
% \begin{equation}
%     a_i^0 = -a_{p-i+1}^0 \quand \vec{w}_i^0 = \vec{w}_{p-i+1}^0 \quad \text{for all } i \in [p/2]. 
% \end{equation}
% Note that this assumption is common in the related literature, e.g. \cite{chizat_2019_lazy, damian2022neural}, and is mainly necessary when $p$ is small, as the initial network function might not concentrate around zero. Moreover, we consider the following random initialization for the weights:
% \begin{equation}
%     \sqrt{p} \cdot a_i^0 \overset{i.i.d}{\sim} \Unif(\{-1, 1\}) \!\!\quand\!\! \vec{w}_i^0  \overset{i.i.d}{\sim} \Unif(\cS^{d-1}).
%     \label{eq:sample_archit}
% \end{equation}
% The distribution of the $a_i$ can be replaced by any other distribution with positive variance, at least when $p$ grows. 
Given the initial conditions, the weights are trained with the following two-step full-batch gradient descent:
\begin{enumerate}[noitemsep,wide=0pt]
    \item \emph{First layer training}: for every gradient step $t\leq T$, a fresh batch of training data $\{(\vec{z}^\nu, y^\nu)\}_{\nu=1}^n$ is drawn from the model in Assumption \ref{ass:data}, and the first layer weights are updated according to:
    \begin{equation}\label{eq:first_layer_gd_again}
        \vec{w}_i^{t+1} = \vec{w}_i^t -  \frac{\eta}{2n} \sum_{\nu = 1}^n \nabla_{\vec{w}_i} \left(y^\nu - \hat f(\vec{z}^\nu; W^t, \vec{a}^{0})\right)^{2},
    \end{equation}
    Hence, the total sample complexity for this step is $Tn$.

    \item \emph{Second layer training}: once the first layer is trained for $T$ steps, the second layer weights $\vec a$ are trained to optimality on \review{an independent batch} of data by performing ridge regression with the features learned in the first step: \begin{equation}\label{eq:second_layer_training}\hat{\vec{a}} = \argmin_{\vec{a}\in \dR^p} \frac{1}{2n} \sum_{\nu = 1}^n \left(y^\nu - \hat f(\vec{z}^\nu; W^{T}, \vec{a})\right)^{2} + \lambda \norm{\vec{a}}^2.
    \end{equation}
\end{enumerate}
Such a separation of the training between the first and second layer is a common setup for the theoretical study of training \citep{damian2022neural, abbe2023sgd, berthier2023learning}, and allows for a more tractable study of convergence. 
\review{Note that we assume resampling of data for the training of the second layer for convenience, while we expect the general picture to hold without resampling albeit with a slightly worse sample complexity. See for instance Theorems $1$ and $3$ in \cite{damian2022neural}.}

% \comment{[Clarify that we do assume resampling for second layer training. However, cite Damian that discusses what happens if you do not do it.] }
\end{assumption}
We are now ready to state our main technical results.

\subsection{Single gradient step}
\label{sec:1step}
Our starting point is a {\it single} giant gradient step, and the phenomenology described in Fig.~\ref{fig:fig_1}. The main goal is to determine under which conditions the relevant directions $\Pi^{\star}\vec{z}$ of the target function $f^{\star}$ can be learned with the training procedure introduced in Assumption \ref{ass:training}. Hence, a crucial object in our analysis is given by the projection of the network weights in the space spanned by the target relevant directions:  
\begin{equation}
    \vec{\pi}_i^t = \Pi^{\star} \vec{w}_i^t;
\end{equation}
where $\Pi^\star$ is the orthogonal projection on $V^\star$.

%\bl{Link to generalization result}

Our first result is of a negative nature, showing the impossibility of learning in the data-scarce regime:
\begin{theorem}\label{thm:one_step_lower_bound}
    Let $\ell$ be the leap index of $f^{\star}$ \eqref{def:leap}, and assume that $n = \cO(d^{\ell - \delta})$ for some $\delta > 0$. Then, with probability at least $1 - cpe^{-c(\delta)\log(d)^2}$, there exists a universal constant $c$ such that for any $i\in [p]$,
    \begin{equation}
        \frac{\norm{\vec{\pi}_i^1}^2}{\norm{\vec{w}_i^1}^2} \leq c\,\frac{\polylog(d)}{d^{(1 \wedge \delta)/2}}. 
    \end{equation}
\end{theorem}
In other words, for \emph{every} neuron $i$, only a vanishing fraction of the weight $\vec{w}_i^1$ lies in the target subspace $V^{\star}$. In particular, if $\delta > 1$, this large gradient step does not improve over the initial random feature weights.

On the other hand, when $n = \Omega(d^\ell)$, we are able to characterize exactly what is being learned in one gradient step.

\begin{theorem}\label{thm:one_step_learning}
    Assume that the $\ell$-th Hermite coefficient $\mu_\ell$ of $\sigma$ is nonzero, and set the learning rate
    \begin{equation}
        \eta = p d^{\frac{\ell - 1}{2}}.
    \end{equation}
    Then, with probability at least $1 - ce^{-c\log(d)^2}$, there exists a random variable $X$ independent of $d$ with positive expectation such that
    \begin{equation}
        \frac{\norm{\vec{\pi}_i^1}^2}{\norm{\vec{w}_i^1}^2} \geq  X_i,
    \end{equation}
    where $X_1, \dots, X_p$ are i.i.d copies of $X$. Further, let $\vec{u}_1^{\star}, \dots, \vec{u}_{r_\ell}^{\star}$ be the higher-order singular vectors of $C_\ell^{\star}$, and define
    \begin{equation}
        V_\ell^{\star} = \vect(\vec{u}_1^{\star}, \dots, \vec{u}_{r_\ell}^{\star}).
    \end{equation}
    Then, the projections $\vec{\pi}_i^1$ asymptotically belong to $V_\ell^{\star}$, in the sense that there exists a constant $c$ such that
    \begin{equation}
        \norm{(I - \Pi_{V_\ell^{\star}}) \vec{\pi}_i^1} \leq c\, \frac{\polylog(d)}{\sqrt{d}},
    \end{equation}
    \review{and  for $p \geq r_\ell$, they  span the space $V_\ell^\star$. Concretely, for every $\delta>0$ there exists a constant $C_\delta$ such that for large enough $d$ with probability $1-\delta$:
    \begin{equation}
        \inf_{\vec v \in  V_\ell^\star}\norm{W^1 \vec v}_2 \geq C_\delta.
    \end{equation}}
\end{theorem}

    Note that in the case $\ell = 1$ the learned subspace $V^\star_{\ell}$ is one dimensional and is identified by the first Hermite coefficient of the target $C_1(f^\star)$: this corresponds to the ``linear subspace learning'' regime exemplified in Fig.~\ref{fig:fig_1}. Some aspects of the results above were already present in previous works, with key differences: \cite{ba2022high} shows the existence of a rank-one property of the gradient at initialization for $n = \Theta(d)$, and \cite{damian2022neural} implies the positive part of our result for $n = \Theta(d^2)$, provided that $V_2^{\star} = V^{\star}$ (which corresponds to their Assumption 2). Our theorem allows us to obtain the matching lower bounds, demonstrating their tightness, and provides the generic picture for {\it any higher} powers of $d$. In particular, our results prove a clear separation between the class of functions learned within the $\Theta(d)$ batch-size setting of \cite{ba2022high} and the $\Theta(d^2)$ batch-size setting of \cite{damian2022neural} and establish a general hierarchy of functions requiring increasing batch-size to be learned with a single gradient step. We refer to Appendix \ref{sec:appendix:gd_proofs} for the proofs of the theorems and a more detailed technical discussion.
    \review{We note that the weak recovery of the subspace $V^\star_{\ell}$ in Theorem \ref{thm:one_step_learning} alone does not suffice towards achieving vanishing generalization error since $V^\star_\ell$ might be a strict subspace of $V^\star$.
    Even in the case when $V_{\ell}^\star=V^\star$, precise generalization bounds depend on the variability of the weights along $V^\star$ and approximation capacity of $\sigma$. 
    However, we believe that for $V_{\ell}^\star=V^\star$ weak recovery guarantees can be translated to perfect recovery and vanishing generalization errors by using $\alpha d^{\ell})$ samples with $\alpha \rightarrow \infty$ and a suitable choice of $\sigma$. For instance, see \cite{damian2022neural} for such an analysis in the case $\ell=2$, at the cost of additional logarithmic factors. We discussion the relationship between weak recovery and generalization errors further in Section \ref{res:generalisation}. }
    
    % \comment{Reviewer 1 asks clarification on the implications that the result has on: a) the full recovery of $V^\star_\ell$ and b) the rest of the target function. [Our definition of weak recovery is consistent. However, we agree that strong recovery is not included. (Mention that one could cover strong recovery using techniques similar to Damian.]}

% \fk{Any other advantage that we have? Did we wrote anything to the referee about this?} 

\begin{figure}[t]
    \centering\includegraphics[width = \textwidth]{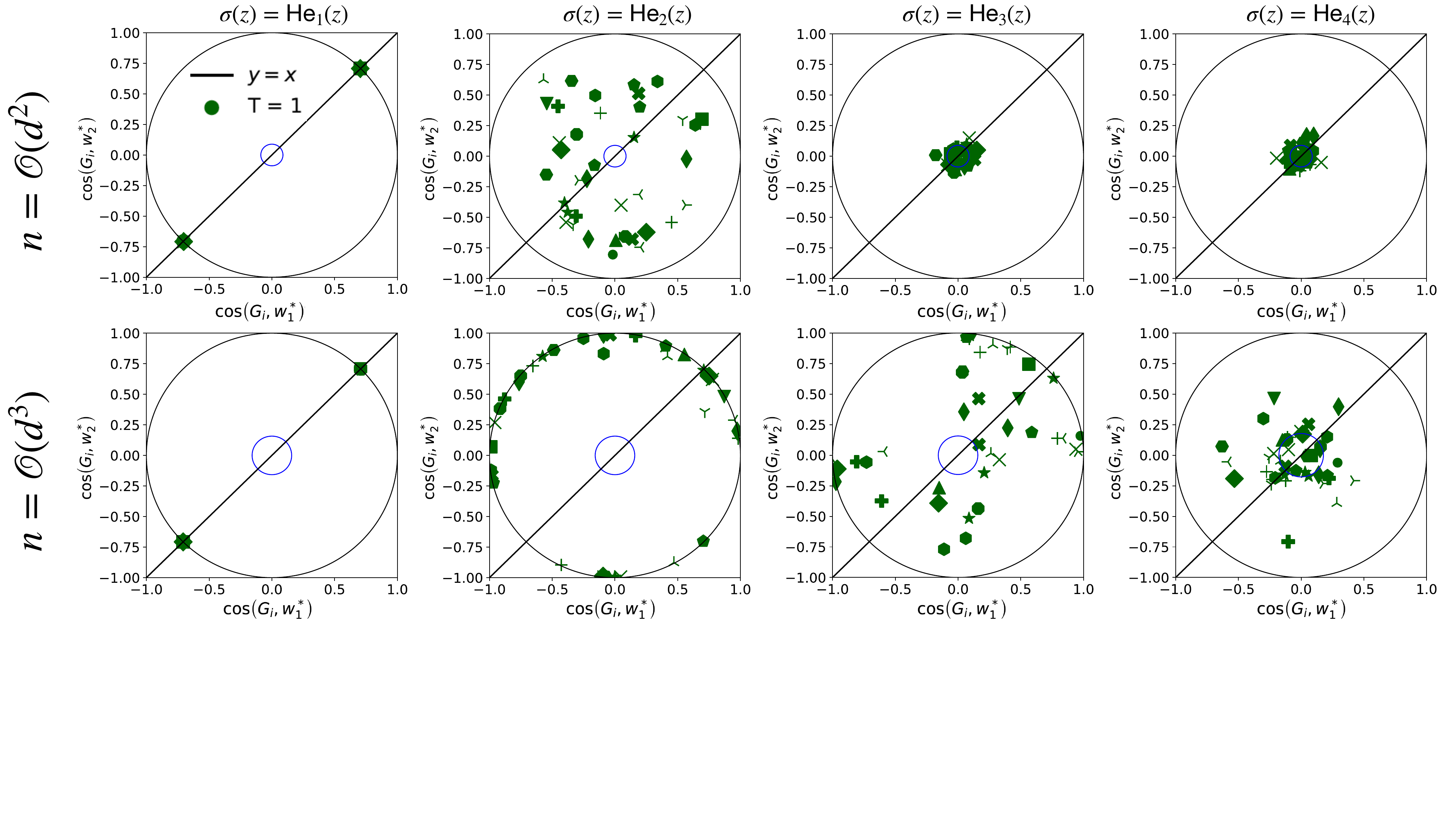}
    \vspace{-7.em}
    \caption{{\bf Feature learning after a single step.}  Specialization of hidden units in the $n = {\cal O}(d^{k})$ regime $(k = 2, 3)$. The plots show the cosine similarity of the gradient with respect to the target vectors $(\vec{w}^{\star}_1,\vec{w}^{\star}_2)$ for $p = 40$ different neurons, identified by different markers. The bisectrix of the first quadrant is shown as a continuous black line, the circle of unitary radius in black, and the circle of radius $\sfrac{2}{\sqrt{d}}$ in blue. In the upper panel, $(n,d) = (2^{18},2^9)$, and $(n,d) =(2^{21},2^7)$ in the lower one. 
    \\ We use a $2$-index target $f^{\star}(\vec z ) = \sigma^{\star}(\langle\vec{w_1}^{\star},\vec{z}\rangle) + \sigma^{\star}(\langle\vec{w_2}^{\star},\vec{z}\rangle)$, with matching student: $\sigma(z) = \sigma^{\star}(z)$. \textbf{Left:} $\sigma(z) =  \He_1(z)$. \textbf{Center-Left:} $\sigma(z) = \He_2(z)$. \textbf{Center-Right:} $\sigma(z) =  \He_3(z)$. \textbf{Right:} $\sigma(z) = \He_4(z)$. We observe that if the leap index $\ell = 1$, we only learn a single direction, no matter the data quantity, while for $\ell > 1$ we learn every direction as soon as we reach $n = {\cal O}(d^\ell)$. The small spread observed for $\sigma(z) = \He_4(z)$ and $n = {\cal O}(d^3)$ is due to the small value of $d$ used for the experiments. See details in App.~\ref{sec:appendix:numerics}.} 
    \label{fig:d2_d3_regimes}
    \vspace{-1em}
\end{figure}

\subsection{Learning with many steps}
\label{sec:few-giant-steps}
We now move to the effect of multiple gradient steps, as described in Fig.~\ref{fig:fig_2}. As we shall see, while the effect of multiple steps is limited to a subclass of functions, \review{they can be efficiently learned} with few iterations. For simplicity, we restrict ourselves to the case $n = {\cal  O}(d)$, although we expect the findings to hold in the other regimes. We unveil the intertwined dependence between representation learning efficiency and a ``directional staircase'' condition. Informally: once a direction is learned by the first-layer weights, the next gradient step uses it as a ladder to learn the next ones, upon the assumption that they are linearly connected (as we will make precise next). The resulting hierarchical learning picture extends to the large batch case \review{(analyzed in \cite{abbe2022merged} for Boolean inputs)} \review{and Gaussian data} the observations of \cite{abbe2023sgd} \review{specialized in the single-pass SGD with batch one}, using basis-dependent projections. Unlike \cite{abbe2023sgd}, both our definition \ref{def:subspace_conditioning} and the training algorithm are basis-independent. Lastly, our analysis incorporates the effect of the learned output $\hat f(\vec{z}; W^t, \vec{a})$ on the gradient updates, leading to interaction between neurons. 
\review{
Such an interaction was incorporated in the mean-field analysis of \cite{abbe2022merged} for boolean inputs, but supressed in \cite{abbe2023sgd}.} 

% \comment{Reviewer 2 claims that Abbe $2022$ proved that the interaction between neurons can be ignored.}

We formalize the hierarchical learning  by introducing the notion of \emph{subspace conditioning}:
\begin{definition}[Subspace conditioning]\label{def:subspace_conditioning}
    Let $V$ be a vector space, and $U \subseteq V$ a subspace. For any function $f: V\to\dR$, and $\vec{x}\in U$, we define the \emph{conditional} function $f_{U, \vec{x}}: U^\bot \to \dR$ as
    \begin{equation}
        f_{U, \vec{x}}(\vec{x}^\bot) = f(\vec{x} + \vec{x}^\bot)
    \end{equation}
\end{definition}
In short, the function $f_{U, \vec{x}}$ corresponds to $f$ ``conditioned'' on the projection of its argument in $U$. Its first Hermite coefficient will be denoted as
\begin{equation}
    \mu_{U, \vec{x}}(f) = \dE_{\vec{x}^\bot \sim \cN(0, I)}\left[\nabla_{\vec{x^\bot}} f_{U, \vec{x}}\right]
\end{equation}
We are now equipped to state our main result describing sequential learning of directions under multiple gradient steps with batch size $n=\Theta(d)$. We denote by $W^{\star} \in \R^{r \times d}$ the matrix with rows $\vec{w_1}^{\star},\cdots, \vec{w_r}^{\star}$. To avoid degeneracy issues, we restrict ourselves to polynomial activations and target functions.
% \comment{Reviewer 1: Assumption $3$ (polynomial activations) and Assumption $5$ in the appendix (bounded derivatives on $\mathbb{R}$) are incompatible [Fixed in the appendix but need to change the exponent of the Orlicz norm everywhere.]} 
\begin{assumption}\label{ass:activ}
Both the student activation $\sigma: \R \rightarrow \R$ and the teacher function $g^{\star}: \R^r \rightarrow \R$ are fixed polynomials with degrees independent of $d,n$.  
Furthermore, $\operatorname{deg}(\sigma) \geq \operatorname{deg}(g^{\star})$.
\end{assumption}

% \comment{Rev 2: Abbe $2022$ Theorem 7 has stronger results on the negative results concerning generalization error when there is no staircase property. The positive part is implied by Theorem 9 (always Abbe $2022$). Reviewer argues that also this statement is stronger as it is proven that the network will fit the function $f_\star$ for almost all function. [Add comparison to Abbe and give it to them.] }
\begin{theorem}
\label{thm:staircase}
    Assume that $n = {\Theta}(d)$, and $\eta > 0,p \in \N$ are fixed.
    Define a sequence of nested subspaces $U^{\star}_0 \subseteq U^{\star}_1 \subseteq \dots \subseteq U^{\star}_t \subseteq \dots$ as 
\begin{itemize}
    \item $U^{\star}_0 = \{0\}$,
    \item for any $t \geq 0$, $U_{t+1}^{\star} = U_t^{\star} \oplus \vect\left( \{ \mu_{U_t^{\star}, \vec{x}}(f^{\star}): \vec{x} \in U_t^{\star}\} \right)$.
\end{itemize} 
Then, under assumptions \ref{ass:training} and \ref{ass:activ}, after $t$ gradient steps of the form \eqref{eq:first_layer_gd}, $W^t$ satisfies the following with high-probability, for all $i \in [p]$:
\begin{enumerate}
\item \review{Then there exists an almost surely positive random variable $X_{t,\ba}$, independent of $d,n$ such that for $p > \operatorname{dim}(U_t^{\star})$,  
\[ \inf_{\bv \in U_{t}^{\star}:\norm{\bv}=1} \norm{W^t \bv} \geq X_{t,\ba}+ O(\frac{\polylog (d)}{\sqrt{d}}).\]}

\item For any $\bv \in {U}_{t}^{\star\bot} \cap V^{\star}$, $\abs{\langle \bw_i,\bv \rangle}=O(\frac{\polylog (d)}{\sqrt{d}})$.
\end{enumerate}
Informally, the above statements imply that almost surely over $\ba$, the first layer learns exactly $U^{\star}_{t}$.
\end{theorem}
% \comment{Rev 1: claims that Abbe et al. are able to recover the ``hidden directions'' while this work overclaims his results. [The phrasing is consistent with what we mean]}
It is easy to check that the above Theorem is consistent with Theorems \ref{thm:one_step_lower_bound}  and \ref{thm:one_step_learning} since \review{$f_{\{0\}, \vec{0}} = 0$}, and $\vec v^{\star}$ is exactly the first Hermite of $f^{\star}$.
\review{We note that unlike Theorem \ref{thm:one_step_learning}, the multiple gradient steps in Theorem \ref{thm:staircase} lead to interaction between neurons. Our analysis shows, however, that the effect of the independence across $a_i$ is preserved across time, allowing different neurons to span different directions in $U^{\star}_{t}$.}

\paragraph{Examples ---} As an illustration, we work out the subspaces $U_{t}^{\star}$ for simplified versions of the examples in Figure \ref{fig:multiple_steps}. For simplicity, we will take $\vec{w}_k^{\star} = \vec{e}_k$, the $k$-th vector of the standard basis, so that $\langle \vec{w}_k^{\star}, \vec{z} \rangle = z_k$.
\begin{itemize}[wide=0pt]
    \item $f^{\star}(\vec{z}) = z_1 + z_2 + z_1^2 - z_2^2$: the normalized first Hermite coefficient of $f^{\star}$ is $\review{\vec{v}^{\star} = (\vec{e}_{1} + \vec{e}_{2})/\sqrt{2}}$, so $U_1^{\star} = \vect(\vec{e}_{1} + \vec{e}_{2})$. Then, we can rewrite $f^{\star}$ in the new basis $(\vec{v}^{\star}, \vec{v}^\bot)$, to get
    \[f^{\star}(\vec{z}) = \sqrt{2}\langle \vec{v}^{\star}, \vec{z}  \rangle + 2 \langle \vec{v}^{\star}, \vec{z} \rangle \langle \vec{v}^\bot, \vec{z} \rangle. \]
    Hence, if $\vec{x} = \lambda \vec{v}^{\star}$, we have $\mu_{U_{t}^{\star}, \vec{x}}(f^{\star}) = 2\lambda \vec{v}^\bot$, and hence $U_2^{\star} = V^{\star}$.

    \item $f^{\star}(\vec{z}) = z_1 + z_2 + z_1^2 + z_2^2$: just as the above example, $U_1^{\star} = \vect(\vec{e}_{1} + \vec{e}_{2})$, and hence we perform the same change of basis. However, this time,
    \[f^{\star}(\vec{z}) = \sqrt{2}\langle \vec{v}^{\star}, \vec{z}  \rangle + \langle \vec{v}^{\star}, \vec{z} \rangle^2 +  \langle \vec{v}^\bot, \vec{z} \rangle^2. \]
    This implies that for any $\vec{x}\in U_1^{\star}$, $\mu_{U_{t}^{\star}, \vec{x}}(f^{\star}) = \vec{0}$, and the direction of $\vec{v}^\bot$ is never learned.
\end{itemize}

We provide a proof of Theorem \ref{thm:staircase} in Appendix \ref{sec:appendix:gd_proofs}. The notion of subspace conditioning (Definition \ref{def:subspace_conditioning}) arises through an inductive decomposition of the projections of the gradient along the target subspace. Furthemore, our theory leads to a precise prediction of the orientation of the gradient in the target subspace after a finite number of steps. This is illustrated in Figure \ref{fig:multiple_steps}. Note that the non-linearity of the activation function allows different neurons to simultaneously specialize along different directions in contrast to the rank-one increase at each saddle in deep linear networks \cite{jacot2021saddle} under vanishing initialization. We refer to Appendix \ref{sec:appendix:numerics} for additional discussion.

To conclude this section, we provide an overview of our present understanding of the effect of batch-size (sample complexity), number of gradient-steps (iteration complexity), and the number of neurons (overparameterization) on learning different multi-index target functions in table \ref{table}. For simplicity, when discussing the leap index $\ell$, we restrict to functions where all the directions in the target subspace have the same leap i.e when $V^*_\ell= V^*$. We define ``staircase functions" as target functions such that the sequence of subspaces defined through subspace conditioning in Theorem \ref{thm:staircase} eventually span $V^*$.

\begin{table}[ht]
\renewcommand{\arraystretch}{1.0}
\begin{center}
 \begin{adjustwidth}{-0cm}{}
\begin{tabular}{||c|| c | c | c ||}
 \hline
 Complexity  & SGD  & One-step GD & Multi-\review{step} GD  \\ 
 of $f^*$  & with $n=1$&with $n=\Theta(d^{\ell})$ &with $n=\Theta(d)$\\ 
 \hline Single-index, $\ell\!=\!1$& \textcolor{blue}{$\tau\!=\!n_{\rm{T}}\!=\!d$} & \textcolor{ForestGreen} {$\tau\!=\!1, n_{\rm{T}}\!=\! n \!=\!\Theta(d)$} & $\tau\!=\!1,n_{\rm{T}}\!=\!\Theta(d)$ \textcolor{magenta}{$(\star)$} \\ 
 \hline Single-index, $\ell\!=\!2$ & \textcolor{blue}{$\tau\!=\!n_{\rm{T}}\!=\!d\log d$} & \textcolor{orange}{$\tau\!=\!1,n_{\rm{T}}\!=\!n\!=\!\Theta(d^2)$\!\!\!} & \textcolor{red}{\!\!\!$\tau\!=\!\Theta(\log d),n_{\rm{T}}\!=\!\Theta(d\log d)$}\!\!\! \\
 \hline
 Single-index, $\ell>2$ & \textcolor{blue}{$\tau=n_{\rm{T}}=d^{\ell-1}$} & \!\!\!$\tau=1,n_{\rm{T}}\!=\!n\!=\!\Theta(d^\ell)$\!\!\! & \textcolor{red}{$\tau\!=\!\Theta(d^{\ell-2}),n_{\rm{T}}\!=\!\Theta(d^{\ell-1})$} \\
 \hline 
 Staircase & \textcolor{blue}{$\tau=n_{\rm{T}}=d$} & $\tau\!=\!1,n_{\rm{T}}\!=\!n\!=\!\Omega(d^2)$ & $\tau\!=\!\Theta(1),n_{\rm{T}}\!=\!\Theta(d)$ \textcolor{magenta}{$(\star)$} \\
 \hline 
\end{tabular}
\end{adjustwidth}
\caption{\label{table} Number of iterations $\tau$ and sample complexity $n_{\rm{T}}$ needed to learn the features/directions of the function $f^*$ with SGD one-sample at a time, using a single step-gradient descent, or using one-pass GD with large $n=O(d)$ batches. Results in \textcolor{blue}{(blue)} are from \cite{BenArous2021,abbe2023sgd}. 
% \review{while entries in \textcolor{magenta}{(magenta)} refers to \cite{abbe2022merged} in the Boolean setting}. 
The upper bounds on the sample complexity are known to hold from \textcolor{ForestGreen}{(green)} \cite{ba2022high} and \textcolor{orange}{(orange)} \cite{damian2022neural}, while the corresponding lower bound, as well as the results in black, are proven in the present paper. \textcolor{red}{(red)} are educated guesses, which are out of the scope of the paper since we only focus on a finite number of iterations $\tau$. For staircase functions, one-step GD may require more than $\Theta(d^2)$ samples depending on the maximum leap across directions in the target subspace. 
\review{The marker \textcolor{magenta}{$(\star)$} refers to results in \cite{abbe2022merged} for the Boolean case and extended in the present work for the Gaussian setting.}
}
\end{center}
\end{table}

\subsection{From feature learning to generalization bounds} 
\label{res:generalisation}

\begin{figure}[t]
    \centering    
\includegraphics[width = 1\textwidth]{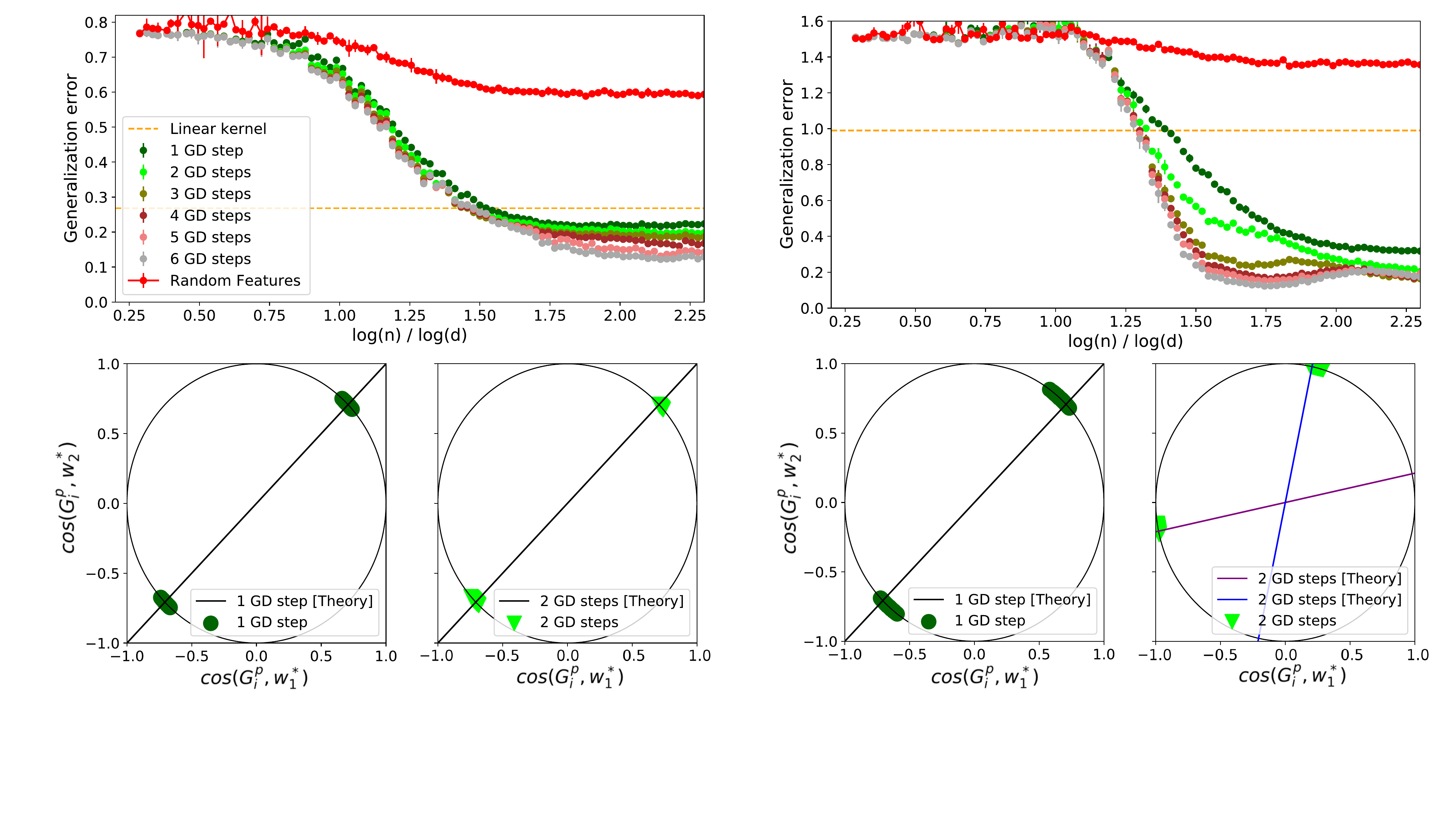}
   \vspace{-5em}
    \caption{{\bf Feature learning with multiple gradient steps.} 
    \textbf{Top:} Generalization error as a function of $n$ ($d=512,p=256$) after iterating the training procedure for six steps. \textbf{Bottom:} Cosine similarity of the projected gradient matrix $G^p$ inside the target subspace for all the $p$ neurons at a fixed ratio $\sfrac{n}{d} = 4$, plotted at different stages of the training. The blue and purple lines are the theoretical predictions for the orientation of the gradient in the second step. \\
        We fix a \rm{relu} student and consider two different 2-index target functions $f^{\star}(\vec z) = \sigma^{\star}_1(\langle\vec{w_1}^{\star},\vec{z}\rangle) + \sigma^{\star}_2(\langle\vec{w_2}^{\star},\vec{z}\rangle)$. {\bf Left:} $\sigma^{\star}_1(z) = \sigma^{\star}_2(z) = \He_1(z) + \sfrac{\He_2(z)}{2} + \sfrac{\He_4(z)}{4!}$ \quad{\bf Right:} $\sigma^{\star}_1(z) = z-z^2$ and $\sigma^{\star}_2(z) = z+ z^2$. In accordance with Theorem \ref{thm:staircase}, the difference between the two cases is clear already after the first GD step: while on the left the gradient is stuck around the predicted rank-one spike after the first step (black line), on the right the gradient changes orientation in the second step, allowing to learn multiple features. See details in App.~\ref{sec:appendix:numerics}.}
    \label{fig:multiple_steps}
    \vspace{-1em}
\end{figure}

% \comment{Mention that also Abbe $2022$ has analyzed consequences on generalization. More precisely, Theorem 7, shows the consequences on generalization when training both layers. This is then used to prove the necessity of the staircase property. [Give it to them by mentioning finite $p$ and Boolean setting.]}
We now investigate the consequence of Theorem \ref{thm:one_step_learning} in the actual performance of the neural network. At a high-level, one expects that learning a subspace $U$ by the first-layer effectively reduces the input-dimensionality to that of $U$, allowing the network to learn a function dependent on $U$ using a significantly fewer number of examples and neurons. However, for learning functions dependent on directions not yet learned by $W$, we expect a requirement of a large number of neurons as well as samples, analogous to the setting of random features. In the following, we make this intuition precise by proving that the generalization error of training the second layer on the top of the learned features is lower bounded precisely by the directions which were not learned by the first-layer. Our results and conjectures incorporate the dependence on both the number of samples and the number of hidden neurons.

% This is a fairly complex question: indeed,  \fk{This should be said later, in the finite p regime we cannot learn anything.}

\paragraph{Finite $p$ regime ---} This is the simplest case: in this setting, there is simply no way to fit anything beyond the learned directions, even with an infinite amount of data. This is formalized by the following proposition:
\begin{proposition}\label{prop:fixed_width_lower_bound}
    Assume that $p$ is bounded as $n, d \to \infty$, and that the first layer $W$ only learns a subspace $U \subseteq V^{\star}$, \review{i.e for any $v \in V^{\star}, v \perp U$, $\abs{\langle \vec{w}_i, \vec{v} \rangle} = o(1)$.}
    For any choice of second layer $\vec{a}$ such that $\norm{\vec{a}}_\infty \leq {\cal  O}(1)$, we have
    \begin{equation}
    \dE\left[\left(f^{\star}(\vec{z}) - \hat f(\vec{z}; W, \vec{a})\right)^2\right] \geq \dE\left[\Var\left( f^{\star}(\vec{z})   \mid P_U \vec{z}\right)\right] - o(1)
    \end{equation}
    where $P_U$ is the orthogonal projection on $U$.
\end{proposition}
% \comment{Reviewer 1: What does it mean to ``learn'' a subspace $U$? Look at Appendix $C.1$ it means $|| P_{U^\perp} w_i|| < \varepsilon_d$ but that would be a stronger notion of recovery than the ``weak recovery'' one guaranteed by the above theorems. [It does not assume strong recovery it only assumes absence of weak recovery in other directions in the teacher subspace. We will revise the notation for $U^\perp$ as it is not the true complement but only restricted to the teacher subspace.]}

We refer to Appendix \ref{sec:appendix:cget_proofs} for the proof of the proposition.
\review{In particular, Proposition \ref{prop:fixed_width_lower_bound} implies that with bounded $p$, the model cannot achieve vanishing generalization error unless $U=V^\star$. For Boolean data, such a lower bound in the absence of recovery of the full subspace was proven in Theorem 9 of  \cite{abbe2022merged}. Proposition \ref{prop:fixed_width_lower_bound}, however, provides a fine-grained lower bound for any $U \subseteq V$. Intuitively}, the right-hand side of the above inequality corresponds to the predictor $\hat f(\vec{z}) = \dE\left[ f^{\star}(\vec{z}) \mid P_U \vec{z} \right]$, which is the best predictor that depends only on $P_U \vec{z}$. Conversely, we expect that with enough neurons, it is possible to approximate this predictor:
\begin{conjecture}
\label{conj:approxim_conj}
\review{
Suppose that $n=\alpha d^\ell$ in Theorem \ref{thm:one_step_learning} or $n=\alpha d$ in Theorem \ref{thm:staircase}. For any  $\delta > 0$, there exists a $p_0 \in\dN$ and $\alpha_0 \in \R^+$ such that if $p \geq p_0$, $\alpha \geq \alpha_0$,} there exists a choice of second layer $\vec{a}$ satisfying $\norm{\vec{a}}_\infty \leq {\cal O}(1)$, such that as $n, p \to \infty$
    \begin{equation}
        \dE\left[\left(f^{\star}(\vec{z}) - \hat f(\vec{z}; W, \vec{a})\right)^2\right] \leq \dE\left[\Var\left( f^{\star}(\vec{z})   \mid P_U \vec{z}\right)\right] + \delta + o(1),
    \end{equation}
\review{where $U$ denotes the corresponding learned subspace.}
\end{conjecture}

%\fk{We need to say why we are not proving it, and why it hard, and how Damian et al are VERY SPECIFIC!! Again, use what we say to the referee}
% \comment{Mention that Abbe $2022$ has already proven this phenomenon in the Boolean setting (Theorem 9) and that numerical investigation also in the Gaussian cases was provided. [It is finite $p$ so we should give it to Abbe.]}
Proving the above conjecture for generic target and activation functions requires proving a sufficient spread of $W$ along the learned subspace $U$ together with an approximation result for finite-dimensional neural networks having activation function $\sigma$. 
Such a conjecture is proven in \cite{damian2022neural} for the specific case of $U = V^{\star}$ (in which case the first term in the RHS is zero) with $\sigma = \operatorname{relu}$,  an additional bias term, and at the cost of logarithmic factors, \review{in Theorem $9$ of \cite{abbe2022merged} for Boolean inputs with activation $\sigma(x)=(1+x)^L$ for a particular choice of $L$, and in \cite{abbe2023sgd} for a particular class of targets.} For generic $\sigma$ with Gaussian inputs, such approximation results are scarcer. In particular, since the projection of the weights $W$ along $U$ cannot be chosen arbitrarily, one cannot directly invoke classical approximation results such as the ones based on Barron spaces. However, the components of $W$ along $U$ are still randomly distributed across neurons. Therefore, one possible approach could be to use results on the generalization of random feature models on finite-dimensional inputs such as \cite{rudi2017generalization}. We leave such an analysis for future work.

\begin{wrapfigure}{r}{0.46\textwidth}
  \vspace{-0.6cm}
% \begin{figure}[t!]
%        \centering
%        \begin{minipage}[c]{0.46\textwidth}
            \includegraphics[width=0.45\textwidth]{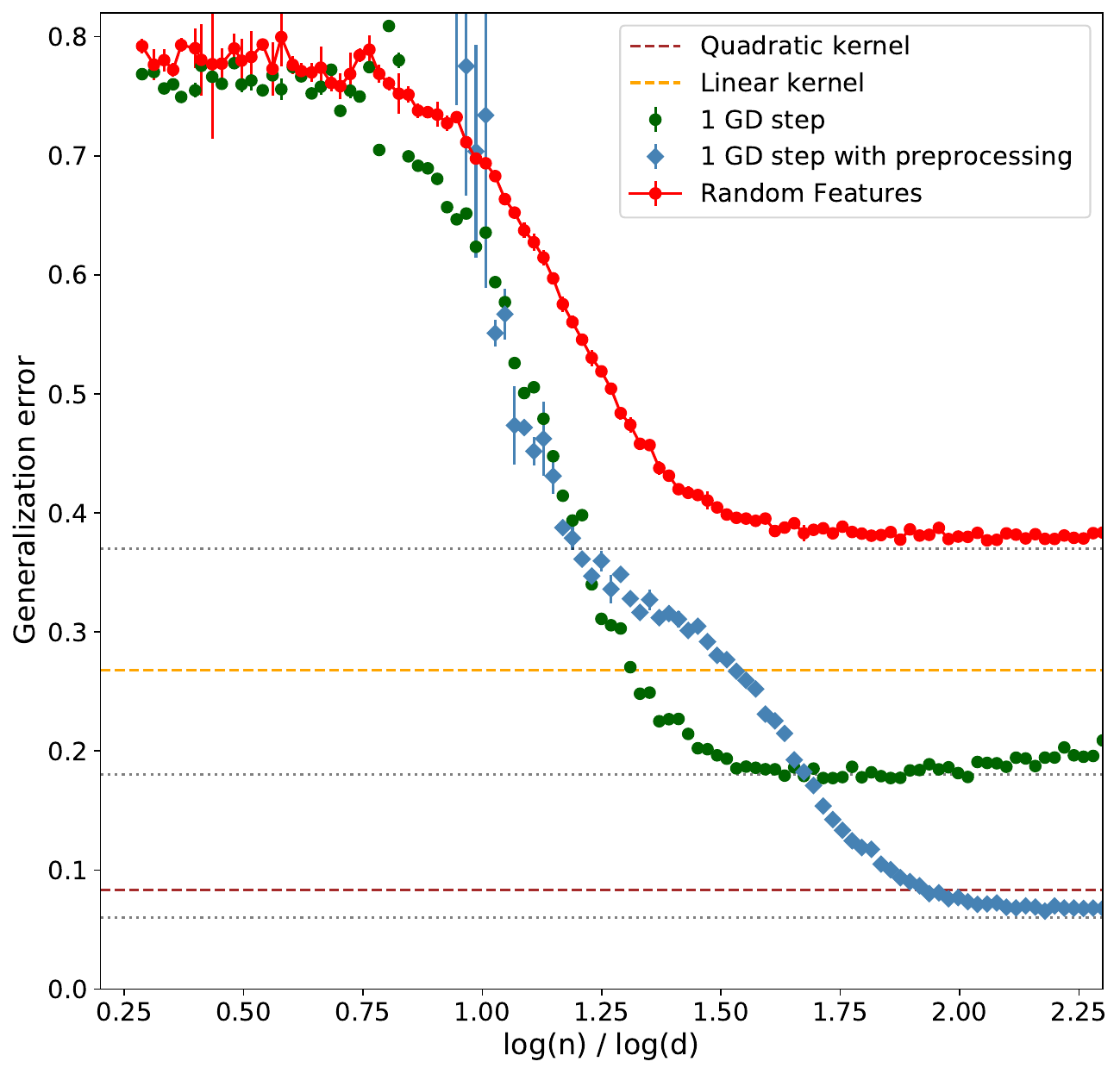}
%        \end{minipage}
%        \hfill
%        \begin{minipage}[c]{0.52\textwidth}
%            \vspace{-0.2cm}
            \caption{{\bf Learning with training of the second layer.} Simulation illustrating the different regimes in Fig.~\ref{fig:fig_1}, using $d\!=\!512, p\!=\!1024$, a symmetric two-index target function $f^{\star}(\vec z) = \sigma^{\star}(\langle\vec{w_1}^{\star},\vec{z}\rangle) + \sigma^{\star}(\langle\vec{w_2}^{\star},\vec{z}\rangle)$ with activation $\sigma^{\star}(z)=\He_1(x)+\He_2(x)/2!+\He_4(x)/4!$, and a \rm{relu} student. (a) The first algorithm (green) applies a giant step and then learns the second layer. When $n\!\gg\!d$, its performance goes beyond the linear predictor that would be obtained with a kernel method and reach the ``linear subspace learning'' regime in Fig.~\ref{fig:fig_1}. (b)  To go beyond this regime, the second algorithm (blue) preprocesses the data to remove a plug-in estimate of the first Hermite coefficient. It reaches a lower plateau as $n\!\approx\!d^2$, now beating the quadratic kernel. We contrast this behavior with the one of the random feature model (red). Details can be found in App.~\ref{sec:appendix:numerics}.}
        \label{fig:example}
%        \end{minipage}
%        \hfill
       \vspace{-2em}
%    \end{figure}
\end{wrapfigure}

\paragraph{Beyond the fixed width regime ---} 
When the number of neurons is allowed to diverge, it is no longer impossible to learn funtions along the directions that are not present in $U$. Indeed, \cite{Mei2023} shows that even in the absence of feature learning (i.e. when the first layer is random), it is possible to learn a polynomial approximation of $f^{\star}$ of degree $k$ as long as $n, p = \Omega(d^{k + \delta})$ for some $\delta > 0$. The exact performance of the network may heavily depend on the training process for the second layer. Intuitively, we expect the following behavior: 
\begin{enumerate}[noitemsep,leftmargin=1em,wide=0pt]
    \item On the directions present in $W$, with enough variety in the neurons, the model behaves similarly as a random feature one in this {\it finite}-dimensional space, and are thus able to learn ``everything''.
    \item On the orthogonal directions, however, the models still behave as a random feature model in high dimension, and thus the polynomial limitations of \cite{Mei2023} still apply.
\end{enumerate}
To formalize this conjecture, we introduce the following definition:
\begin{definition}
    Let $V$ be a vector space, and $U \subseteq V$ a subspace. For any $k \geq 0$, we define the space $\cP_{U, k}$ of functions $f: V \to \dR$ such that for any $\vec{x}\in U$, the function $f_{U, \vec{x}}$ introduced in Definition \ref{def:subspace_conditioning} is a polynomial of degree at most $k$.
\end{definition}

Simply put, the space $\cP_{U, k}$ consists of polynomials in $\vec{x}^\bot$ of degree at most $k$, whose coefficients can be functions of $\vec{x}$. We will denote by $P_{U, \leq k}$ and $P_{U, >k}$ the projections on $\cP_{U, k}$ and $\cP_{U, k}^\bot$ in $\ell^2(\dR, \gamma)$, respectively, where $\gamma$ is the Gaussian measure. 
This allows us to write the above intuition in the following precise conjecture:
\begin{conjecture}\label{conj:kernel+feature}
    Assume that $p = {\cal O}(d^{\kappa_1})$ and $n = {\cal O}(d^{\kappa_2})$, and that the first layer $W$ only learns a subspace $U^{\star} \subseteq V^{\star}$.    Then, if $\hat{\vec{a}}$ is obtained as in Eq. \eqref{eq:second_layer_training} for any value of $\lambda$,
    \begin{equation}
        \dE\left[\left(f^{\star}(\vec{z}) - \hat f(\vec{z}; W, \hat{\vec{a}})\right)^2\right] \geq \norm{P_{U^{\star}, >\kappa} f^{\star}}^2_\gamma - o(1),
    \end{equation}
    where $\kappa = \min(\kappa_1, \kappa_2)$ and $\norm{}_\gamma$ is the norm in $\ell^2(\dR, \gamma)$.
\end{conjecture}
While a general proof for this conjecture would require significant additional work, we can prove some particular cases. First, it is easy to check that Proposition \ref{prop:fixed_width_lower_bound} corresponds to the $\kappa = \kappa_1 = 0$ case of the above conjecture. Second, in the next section, we prove the fairly delicate $\kappa_1 = \kappa_2$ case.

\paragraph{The linear case ---} We provide a proof of the $\kappa_1 = \kappa_2 = 1$ case of the above conjecture, also known as the \emph{proportional regime}. In this setting, \citep{mei_generalization_2022, gerace_generalisation_2020, goldt_gaussian_2021, hu2022universality} have shown that for the random features $W^{0}$ case, the generalization error of the minimizer for the second layer \eqref{eq:second_layer_training} is equivalent to the generalization error of an equivalent linear model; this result was extended in \cite{ba2022high} to a trained first layer $W^{1}$ for small, ``lazy'' stepsizes $\eta = {\cal O}(1)$.
\cite{ba2022high} has also shown that when $\eta = \Theta(p)$, the equivalent linear description breaks down, due to the appearance of a spike term proportional to $\eta$ in the trained weight matrix $W^{1}$; here, we show that this spike only allows the learning of non-linear functions parallel to the spike. Notice that the optimization problem \eqref{eq:second_layer_training} is equivalent to
\begin{equation}\label{eq:ck_kernel}
    \min_{\vec{a} \in \dR^p} \frac1n \sum_{\nu = 1}^n \left( \langle \vec{a}, \phi_{\text{CK}}(\vec{z}^\nu)\rangle - f^{\star}(\vec{z}^\nu) \right)^2 + \lambda \norm{\vec{a}}^2
\end{equation}
where $\phi_{\text{CK}}(\vec{z}) = \sigma(W\vec{z})$ is the feature map corresponding to the \emph{conjugate kernel} of the neural network. For a given direction $\vec{v}$, we define the \emph{conditional linear} equivalent map $\phi_{\text{CL}}(\vec{z}; \vec{v})$ as follows: given the decomposition $\vec{z} = z_{\vec{v}} \vec{v} + \vec{z}^\bot$,
\begin{equation}
    \phi_{\text{CL}}(\vec{z}; \vec{v}) = \mu\left(z_{\vec{v}}\right) + \Psi(z_{\vec{v}}) z^\bot +  \Phi(z_{\vec{v}})\vec{\xi}
\end{equation}
where $\vec{\xi} \sim \cN(0, I_p)$, and $\mu\in\dR^p, \Psi\in \dR^{p\times d}, \Phi \in \dR^{p \times p}$ are chosen to match the first two conditional moments of $\phi_{\text{CK}}$:
\begin{align*}
    \mu(z_{\vec{v}}) &= \dE\left[ \phi_{\text{CK}}(\vec{z}) \mid z_{\vec{v}} \right], \quad\Psi(z_{\vec{v}}) = \dE\left[ \phi_{\text{CK}}(\vec{z})(z^\bot)^\top \mid z_{\vec{v}} \right],\\ 
    \Phi(z_{\vec{v}}) &= \Cov\left[ \phi_{\text{CK}}(\vec{z}) \mid z_{\vec{v}} \right] - \Psi(z_{\vec{v}})\Psi(z_{\vec{v}})^\top
\end{align*}

Then, the following theorem holds:

\begin{theorem}\label{thm:cget}(informal)
     Assume that $n, p = \Theta(d)$, and that $V_1^\star$ defined in Theorem \ref{thm:one_step_learning} is non-zero. Let $v^\star=C_1(f^\star)$, so that $V_1^\star = \vect(v^\star)$.Then, after one gradient step, there exists a vector $\vec{v} \in \dR^d$ such that:
    \begin{itemize}[noitemsep,wide=0pt]
        \item the projection of $\vec{v}$ on $V^{\star}$ is proportional to $\vec{v}^{\star}$,
        \item the solutions $\hat{\vec{a}}$, $\tilde{\vec{a}}$ to the optimization problem in \eqref{eq:ck_kernel} with the feature maps $\phi_{\text{CK}}(\vec{z})$ and $\phi_{\text{CL}}(\vec{z}; \vec{v})$ yield the same generalization error.
    \end{itemize}
\end{theorem}
The full statement of Theorem \ref{thm:cget} with the relevant hypotheses is discussed App.~\ref{sec:appendix:cget_proofs}. In particular, as we show in App.~\ref{sec:appendix:cget_proofs}, Theorem \ref{thm:cget} provides a proof of the case $\kappa_1 = \kappa_2 = 1$ of Conjecture \ref{conj:kernel+feature}:
\begin{corollary}\label{corr:lower_bound}
Conjecture \ref{conj:kernel+feature} holds for $\kappa_1 = \kappa_2 = 1$.
\end{corollary}
Informally, as for the random features model, features in the subspace orthogonal to $\vec{v}$ map to a linear model, and therefore only the linear part of $f^{\star}$ can be learned in this subspace. This precisely corresponds to the statement in Conjecture \ref{conj:kernel+feature}. This is illustrated in particular in Fig.\ref{fig:features} where the generalization after one step of gradient descent is improved  drastically over initialization for $f^\star(\vec z)=z_1+z_1 z_2$, while it is not for $f^\star(\vec z)=z_1+z_2 z_3$. 
% Fig.\ref{fig:postive_cget_appendix} illustrates this phenomenon for more complex functions $f^\star(\vec z)=z_1+z_1z_2+g_{3}(z_1)z_3$ vs $f^\star(\vec z)=z_1+z_2z_3+g_{3}(z_2)z_3$, where $g_{\ell}$ is an auxiliary function with leap index $\ell$ (See App.~\ref{sec:appendix:numerics}). 
We refer to App.~\ref{sec:appendix:numerics} for additional discussion on the numerical validation of these claims. 

\paragraph{General overview ---} We conclude this section by offering a summary on the generalization properties of two-layer networks. In table \ref{table2} we focus on the effect of overparameterization, i.e. number of neurons in the hidden layer. 
% \begin{table}[h]
% \renewcommand{\arraystretch}{1.20}
% % \begin{center}
% \begin{tabular}{|c||c | c | c ||}
%  \hline
%  $f^*$ & Random Features & One-step GD & Multi-step GD \\
%  \hline  $z_1+z_1z_2$ &$ \min (p,n) \!=\! \tilde{\Theta}(d^2)$ & $p\!=\!\Theta(d),n\!=\!\Theta(d)$ &$\tau\!=\!2,p\!=\!\Omega(1),n\!=\!\Theta()$\\ 
% \hline  $z_1+z_2z_3$ &$ \min (p,n) \!=\! \tilde{\Theta}(...)$ & $p\!=\!\Theta(...),n\!=\!\Theta(d)$ &$\tau\!=\!3,p\!=\!\Theta(...),n\!=\!\Theta(...)$\\ 
%  \hline
%   $z_1+z_1z_2+g_{3}(z_2)z_3$ &$ \min (p,n) \!=\! \tilde{\Theta}(d^3)$ & $p\!=\!\Theta(d),n\!=\!\Theta(d)$ & $\tau\!=\!2,p\!=\!\Theta(1),n\!=\!\Theta(d)$\\ 
%  \hline  $z_1+z_2z_3+g_{3}(z_2)z_3$ &$ \min (p,n) \!=\! \tilde{\Theta}(...)$ & $p\!=\!\Theta(...),n\!=\!\Theta(d)$ &$\tau\!=\!3,p\!=\!\Theta(...),n\!=\!\Theta(...)$\\ 
%  \hline 
% \end{tabular}
% \caption{\label{table2}: \textcolor{blue}{LP: Running the third row, to include a new positive example looking at 2GD steps.} Sample complexity $n_{\rm{T}}$ required to fit a target function upto a desired accuracy. Random Features correspond to $\tau=0$ i.e the network at initialization.}
% % \end{center}
% \end{table}

%%% New attempt %%%% 
\begin{table}[h]
\renewcommand{\arraystretch}{0}
 \begin{adjustwidth}{0cm}{}
\begin{tabular}{||c||c || c | c ||}
 \hline
 $f^*(\cdot)$ & Random  & GD & GD   \\
  & Features & $p\!=\!\mathcal{O}(1), n\!=\!\mathcal{O}(d)$& $p\!=\!\mathcal{O}(d), n\!=\!\mathcal{O}(d)$ \\
 \hline 
  \hline  $z_1\!+\!z_1z_2$ &$ \min (p,n) \!=\! \tilde{\Theta}(d^2)$ & $\tau = 2$ &$\tau = 1$\\ 
\hline  $z_1\!+\!z_2z_3$ &$ \min (p,n) \!=\! \tilde{\Theta}(d^2)$ &\!No learn. in $\tau\!=\!\Theta(1)$\!&\!No learn. in $\tau\!=\!\Theta(1)$\!\\
 \hline $z_1\!+\!z_1z_2\!+\!\He_{3}(z_1)z_3$ &$ \min(p,n) \!=\! \tilde{\Theta}(d^4)$ & $\tau = 2$ & $\tau = 1$\\ 
\hline  $z_1\!+\!z_1z_2\!+\!\He_{3}(z_2)z_3$ &$ \min (p,n) \!=\! \tilde{\Theta}(d^4)$ & $\tau = 3$ &$\tau = 2$ \\
\hline  $z_1\!+\!z_2z_3\!+\!\He_{3}(z_2)z_3$ &$ \min (p,n) \!=\! \tilde{\Theta}(d^4)$ &\!No learn. in $\tau\!=\!\Theta(1)$\!&\!No learn. in $\tau\!=\!\Theta(1)$\!\\
 \hline 
\end{tabular}
\end{adjustwidth}
\caption{\label{table2}\textbf{Overparametrization helps learning multi-index targets.} Table summarizing the number of iterations needed to learn different multi-index targets $f^\star(\cdot)$  with underparametrized $\left(p = \mathcal{O}(1)\right)$ and overparametrized $\left(p=\mathcal{O}(d)\right)$ two-layer networks trained in the proportional batch-size regime $n=\mathcal{O}(d)$. We contrast this behaviour with the sample complexity of Random Features trained with $n$ samples ($n_{\rm{T}}=n$), corresponding to $\tau=0$, i.e, the network at initialization. The time complexity can be significantly reduced by considering overparametrized networks, in accordance with Theorem~\ref{thm:cget}.}
\end{table}
Here, the function $f^\star(\vec{z})=z_1+z_1z_2$ can either be learned in two steps through the staircase property or in a single step by increasing the number of hidden neurons. This is an example of a multi-index target function that can be learned with a single step of $\mathcal{O}(d)$ batch-size through overparameterization, exemplified in the left panel of Fig.~\ref{fig:features}. On the other hand, $f^\star(\vec z) = z_1 + z_2z_3$ is not linear conditioned on $z_1$ (learned at the first step), and this leads to a generalization pattern similar to Random Features in the overparametrized regime, see right panel of Fig.~\ref{fig:features}. We refer to App.~\ref{sec:appendix:numerics} for the analysis of the last two example of Table~\ref{table2} (see Fig.~\ref{fig:postive_cget_appendix}). 
\\Conjecture.~\ref{conj:kernel+feature} extends the claims of Thm~\ref{thm:cget} to the polynomial scaling regime of $p = \mathcal{O}(d^{\kappa_1})$ and $d = \mathcal{O}(d^{\kappa_2})$. We exemplify the effect of overparametrization in this setting by considering the following target function:
\begin{equation}
f^\star(\vec{z})=z_1+\He_3(z_1)z_2+\He_2(z_1)\He_2(z_3).
\end{equation}
Here, we can illustrate the intertwined dependence between batch-size, number of hidden neurons, and time iteration in determining learning efficiency. Indeed, changing any of these quantities can affect the components of the target function that are fitted by the network. Consider the following concrete scenarios, where $\hat{f}(\vec{z})$ denotes the predicted output function, after training the second layer:
\begin{enumerate}
    \item One-step GD with $n =\Theta(d), p = \mathcal{O}(1)$: $\hat{f}(\vec{z})=z_1$.
    \item One-step GD with $n =\Theta(d), p = \mathcal{O}(d)$: $\hat{f}(\vec{z})=z_1+\He_3(z_1)z_2$.
    \item One-step GD with $n = \Theta(d^2)$, $p = \Theta(d^2)$: $\hat{f}(\vec{z})=z_1+\He_3(z_1)z_2+\He_2(z_1)\He_2(z_3)$. 
    % Here one gradient step is performed with batch-size $\Theta(d)$ and subsequently the second layer is trained using $\Theta(d^2)$ samples.
\end{enumerate}

\section*{Conclusion}
% \comment{[Cite the ICML work Arnaboldi et al. that originated from the conjecture in Table 2]. [Add discussion on strong recovery and that one could expand using techniques coming from Damian]}
In this work, we investigated the dynamics of two-layer neural networks as they learn a multi-index model using a single-pass, large-batch, gradient descent algorithm. We shed light on the intricate interactions between the task's structure, notably the complexity of the target function, the hyperparameters of SGD, such as the batch size and learning rate, and the network architecture, such as how the hidden layer width impact the approximation capacity of the network at fixed data budget. We highlight three key findings: a) The pronounced influence of a single gradient step on feature learning, underlining the nexus between batch size and the target's information exponent; b) The amplification of the network's approximation capacity over successive gradient steps and the learning of increasingly complex functions over time; c) The improvement in generalization when contrasted with the random feature/kernel regime. In conclusion, we presented a thorough mathematical framework detailing many nuances of data representation learning in two-layer neural networks during their early training phase. Finally, we note that while certain aspects of those results are left as conjectures, we believe that they capture both the dynamics of multiple gradient steps, as well as the structure of the remaining directions. Proving those conjectures requires fairly delicate concentration and random matrix arguments, which may be of independent mathematical interest.

\section*{Acknowledgements}
We thank Denny Wu, Theodor Misiakiewicz, Loucas Pillaud-Vivien, Joan Bruna \& Lenka Zdeborov\'a for insightful discussions. We also acknowledge funding from the Swiss National Science Foundation grant SNFS OperaGOST, $200021\_200390$ and the \textit{Choose France - CNRS AI Rising Talents} program. This work was completed while Ludovic Stephan was a postdoc in the IdePHICS laboratory at EPFL. 

%% file: reviewed_section/appendix/numerics.tex
\section{Numerical investigation}
\label{sec:appendix:numerics}

In this section we explain the procedures to get the different figures in the main text, along with the details behind the numerical experiments. We provide as well additional plots corroborating the theoretical results presented in the main manuscript. The code is available on \href{https://github.com/lucpoisson/GiantStep}{GitHub}.

\paragraph{Description of training algorithm and hyperparameters:} First, we describe the training protocol reported in Alg.~\ref{alg:gd_training}: we separately update the first layer with $T-$GD steps of learning rate $\eta$, followed by training with standard ridge regression for the second layer with fixed regularization strength $\lambda$. We vary adaptively the learning rate to satisfy the hypothesis of Thm.~\ref{thm:one_step_learning}, i.e. $\eta = \mathcal{O}(p\sqrt{\frac{n}{d}})$, and we take noiseless labels. If not stated otherwise, we consider fixed regularization strenghth $\lambda = 1$. We average over $10$ different seeds to get the mean performance, and we use standard deviation for giving confidence intervals.

% \begin{figure}
%       \centering\includegraphics[width = \textwidth]{figs/separate_train.pdf}
%     \vspace{-10.5em}
%     \caption{{\bf Seaparate phases of training.} Learning curves plotted varying the number of neurons in the second layer $p \in (128,256,512)$ separating the training stage in Algorithm.~\ref{alg:gd_training}. The parameters are equivalent } 
%     \label{fig:n2large}
% \end{figure}
\subsection{Learning with a single giant step}  

A sizable part of our results concerns the feature learning efficiency of two layer neural networks after one giant step of GD. We provide a toy illustration of the phenomenology in Fig.~\ref{fig:fig_1}, and rigorously characterize this in Theorems \ref{thm:one_step_lower_bound} and \ref{thm:one_step_learning}. Moreover, in section \ref{res:generalisation} we provide a plethora of results analyzing the consequences of the theorems above in the actual learning performance of the network. Here, we perform a detailed numerical investigation of the different claims in these results. 

\paragraph{Learning single-index targets:} We start by analyzing the generalization performance of different two-layer networks after one giant GD step when learning single-index teacher functions (See Fig.~\ref{fig:example}). We compare the generalization performances of linear and quadratic kernel methods - horizontal lines marked by different colors computed at $n=n_{max} \sim d^{2.25}$ - with three networks: a) \textit{random features} (red points) random network with fixed weight matrix $W_0$ at initialization; b) \textit{1 GD step} (green points) two-layer network trained using one step in the protocol of Alg.~\ref{alg:gd_training}; c) \textit{1 GD step with preprocessing} (blue points) 
two-layer network trained using a preprocessing step in Alg.~\ref{alg:gd_training}. The introduction of a preprocessed algorithm is linked with the theoretical results of Theorem \ref{thm:one_step_learning} and we provide a detailed analysis in the next paragraph.

\paragraph{The importance of preprocessing:} 
As Theorem.~\ref{thm:one_step_learning} provably states, it is not possible to get fully specialized hidden units with one giant step of GD in the $n = \mathcal{O}(d^k)$ regime (with $k>1$), if the directions associated to teacher Hermite coefficients lower than $k$ are not suppressed, or equivalently, if the leap index of the target is lower than $k$ - see Fig.~\ref{fig:fig_1}. 
We can circumvent this issue by using a preprocessing step. Given a batch of size $n = \mathcal{O}(d^k)$, we preprocess the labels in Alg.~\ref{alg:gd_training} using a method introduced in \cite{damian2022neural} for the case $n = \mathcal{O}(d)$:
% \begin{align}
% \label{eq:app:preprocess}
% y_{\nu} \leftarrow y_{\nu} - \sum_{m=0}^{l-1}\langle \hat{\vec{c}}_{m},\He_m(\Vec{z}_{\nu}) \rangle  \qquad  \text{with} \qquad \Vec{\hat{c}}_{m} \leftarrow \frac 1n \sum_{\nu=1}^n y_{\nu}  \He_m(\Vec{z}_{\nu})
% \end{align}
\begin{align}
\label{eq:app:preprocess_coeff}
\hat{c}_{j_1,\dots,j_d} &\leftarrow \frac1n \sum_{\nu=1}^n y_{\nu}  \He_{j_1}(\langle \Vec{e}_1, \Vec{z}_{\nu})\rangle \cdots \He_{j_d}(\langle \Vec{e}_d, \Vec{z}_{\nu})\rangle
 \\
 \label{eq:app:preprocess_label}
y_{\nu} &\leftarrow y_{\nu} - \sum_{j_1,\dots,j_d: j_1 + \cdots + j_d <k } \frac{\hat{c}_{j_1,\dots,j_d}}{j_1 ! \cdots j_d!} \He_{j_1}(\langle \Vec{e}_1 ,\Vec{z}_{\nu})\rangle \cdots \He_{j_d}(\langle \Vec{e}_d ,\Vec{z}_{\nu})\rangle 
\end{align}
where we denoted with $\hat{c}_{j_1,\cdots,j_d}$ the plug-in estimates from data of the teacher Hermite coefficients, and with $\{\Vec{e}_i\}_{i \in [d]}$ the canonical basis in $\mathbb{R}^d$. By standard concentration arguments \cite{Gotze2019ConcentrationIF} the plug-in estimation of the coefficients is accurate only in the $n = \omega(d\polylog(d))$ regime. Indeed, in Fig.~\ref{fig:example} the inefficient estimation of eq.~\eqref{eq:app:preprocess_coeff}  in the $n = o(d)$ sample regime generates a noisy learning curve for the preprocessed algorithm (blue points). The ridge estimator $\hat{\Vec{a}}$ is consequently found by training on the processed labels  defined in eq.~\eqref{eq:app:preprocess_label} and the suppressed part is injected back in the predictor only at test time:
\begin{align}
\hat{f}(\Vec{z}_{\nu}) = \frac{1}{\sqrt{p}} \Vec{\hat{a}}^\top \sigma(W \vec{z}_{\nu}) + \sum_{j_1,\cdots,j_d: j_1 + \dots + j_d <k } \frac{\hat{c}_{j_1,\dots,j_d}}{j_1 ! \cdots j_d!} \He_{j_1}(\langle \Vec{e}_1, \Vec{z}_{\nu})\rangle \cdots \He_{j_d}(\langle \Vec{e}_d ,\Vec{z}_{\nu})\rangle   
\end{align} 
% \begin{align}
% \hat{f}(\Vec{z}_{\nu}) = \frac{1}{\sqrt{p}} \Vec{\hat{a}}^\top \sigma(W \vec{z}_{\nu}) + \sum_{m=0}^{l-1} \langle\hat{\vec{c}}_{m},\He_m(\Vec{z}_{\nu})\rangle\end{align}

\paragraph{Comparison of different methods:}
The results presented in Fig.~\ref{fig:example} verify the theoretical predictions of Thm.~\ref{thm:one_step_learning}: 
% it is suboptimal to use preprocessed Alg.~\ref{alg:gd_training} in the $n=\mathcal{O}(d)$ regime with respect to vanilla Alg.~\ref{alg:gd_training}, since the latter attains the ``linear subspace learning'' regime (see Fig.~\ref{fig:fig_1}) and beats the linear kernel, while the former cannot.
in the $n=\mathcal{O}(d)$ regime vanilla Alg.~\ref{alg:gd_training} attains the ``linear subspace learning'' regime (see Fig.~\ref{fig:fig_1}) and beats the linear kernel, while the preprocessed version cannot. 
However, implementing preprocessing turns out definitely beneficial in the $n = \mathcal{O}(d^2)$ region. Indeed, while the vanilla Alg.~\ref{alg:gd_training} remains stuck on the linear subspace learning plateau, the preprocessed Alg.~\ref{alg:gd_training} reaches a lower test error than the quadratic kernel. This is achieved by effectively raising the leap index of the target function. More precisely, given a target with leap index $\ell=1$ as in Fig.~\ref{fig:example}, the manipulation in eq.~\eqref{eq:app:preprocess_label} aims exactly at the removal of the first Hermite coefficient of the target by estimating it from the data, allowing feature learning in the $n = \mathcal{O}(d^2)$ regime in accordance with Thm.~\ref{thm:one_step_learning}. 
We complement the above picture by analyzing the influence of the number of neurons $p$ on the generalization performance (see Fig.~\ref{fig:changep}): by increasing the expressive power of the network, we attain the single-index regime by using a single giant step of Alg.~\ref{alg:gd_training} (in accordance with Conj.~\ref{conj:approxim_conj}). Moreover, we note that it is necessary to use $p = 2d$ in order to be able to beat the performance of the quadratic kernel in this learning task (rightmost section). 

\begin{figure}[t]
      \centering\includegraphics[width = \textwidth]{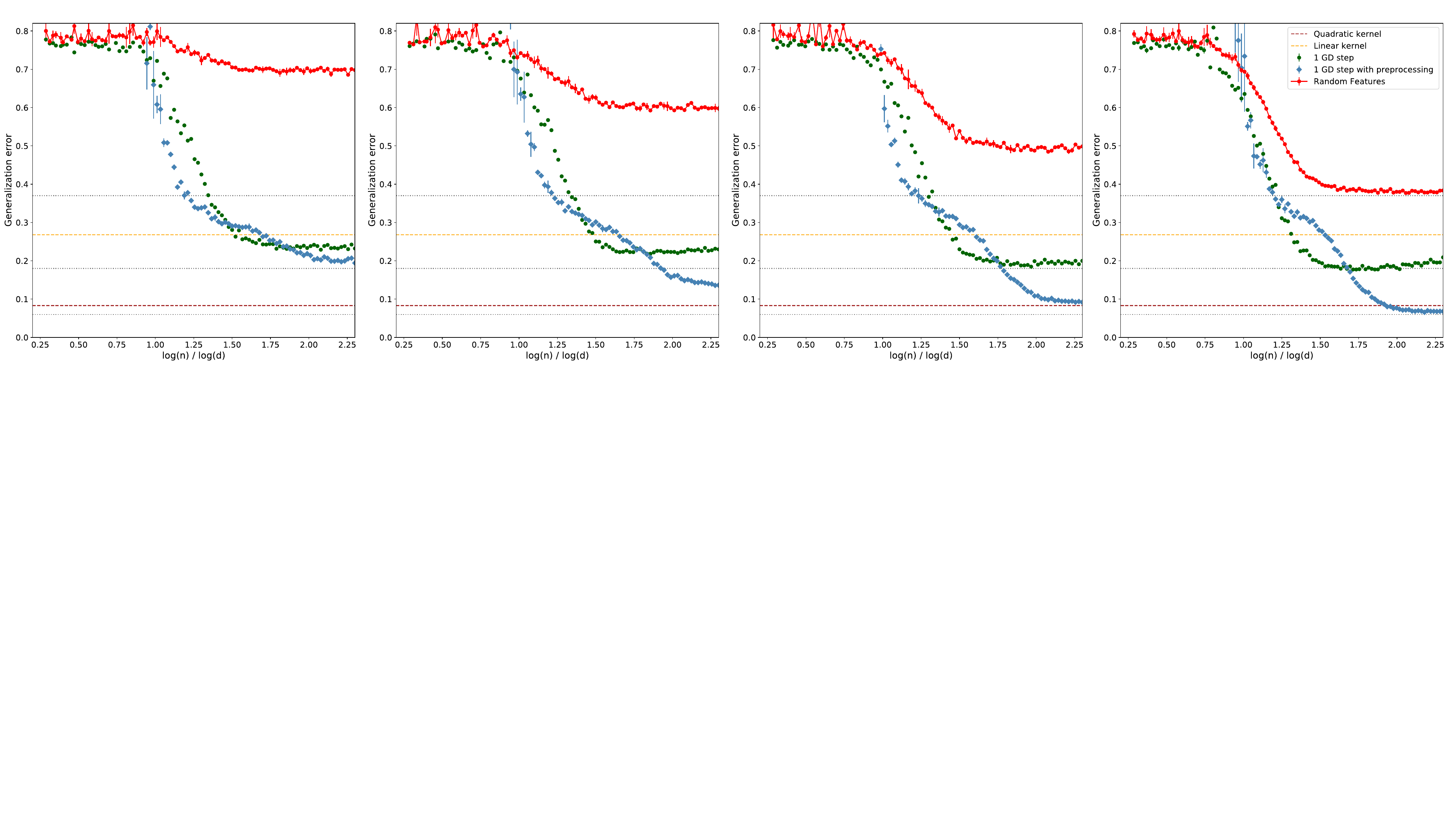}
    \vspace{-13.5em}
    \caption{{\bf Learning as a function of the number of neurons.} Simulations illustrating the different regimes in Fig.~\ref{fig:fig_1}, using $d\!=\!512$,  a symmetric two-index target function $f^\star(\vec z) = \sigma^\star(\langle\vec{w_1}^\star,\vec{z}\rangle) + \sigma^\star(\langle\vec{w_2}^\star,\vec{z}\rangle)$ with activation $\sigma^\star(z)=\He_1(x)+\He_2(x)/2!+\He_4(x)/4!$, a \rm{relu} student, and changing the value of $p \in (128,256,512,1024)$ from left to right. (a) The first algorithm (green) applies a giant step and then learns the second layer. When $n\!\gg\!d$, its performance goes beyond the linear predictor that would be obtained with a kernel method and reach the ``linear subspace learning'' regime exemplified in Fig.~\ref{fig:fig_1}. (b)  To go beyond this regime, the second algorithm (blue) preprocesses the data to remove a plug-in estimate of the first Hermite coefficient. The dotted black lines are a guide for the eyes referencing to the different plateaus of the rightmost plot.}
    \label{fig:changep}
\end{figure}

% \lp{Put?} We note that the preprocessed Alg.~\ref{alg:gd_training} performs poorly in a region where we cannot safely estimate the first Hermite coefficient, given the limited amount of data in the batch. However, if we separate the two stages of training, disentangling the training of the first and second layers we obtain the results in the lower panel of Fig.~\ref{fig:changep}. More precisely, we consider a situation in which we first train the first layer with a batch of size $n$, and we evaluate the learned representation in a separate step accessing a larger batch of $n_2 \gg n$. 
\paragraph{Investigating representation learning efficiency:} We move to an additional numerical investigation of feature learning efficiency, as characterized by Theorems \ref{thm:one_step_lower_bound} and \ref{thm:one_step_learning}. We again consider a single step of Alg.~\ref{alg:gd_training}, focusing now on the analysis of the gradient - see Fig.~\ref{fig:d2_d3_regimes}. We compute the gradient matrix $G \in \mathbb{R}^{p \times d}$ and plot the cosine similarities of all the rows $\{\vec{G}_i \in \mathbb{R}^d\}_{i=1}^p$ with the teacher vectors $(\vec{w}^\star_1,\vec{w}^\star_2)$. The figure clearly illustrates the claims of Thm.~\ref{thm:one_step_learning}: in the $n=\mathcal{O}(d^k)$  regime (with $k > 1$) is necessary to analyze targets with leap index $k$ in order to obtain specialized hidden units. Moreover, the leftmost section of Fig.~\ref{fig:d2_d3_regimes} completes the picture offered by Figs.~\ref{fig:example}\&\ref{fig:changep} about the lack of specialization in presence of teacher functions with non-zero first Hermite coefficient: the gradient is stuck in the single-index regime theoretically predicted by Thm.~\ref{thm:one_step_learning}, regardless of the sample regime considered, preventing feature learning. To produce the plot, we consider the initialization of the second layer to be Gaussian, i.e. $\vec a^0 \sim \sqrt{p}\mathcal{N}(\vec 0, I_p)$. This choice helps the spreading of the neurons in the teacher subspace, improving the figure's visualization clarity.
\vspace{-0.5em}
\begin{figure}[t]
    \centering\includegraphics[width = \textwidth]{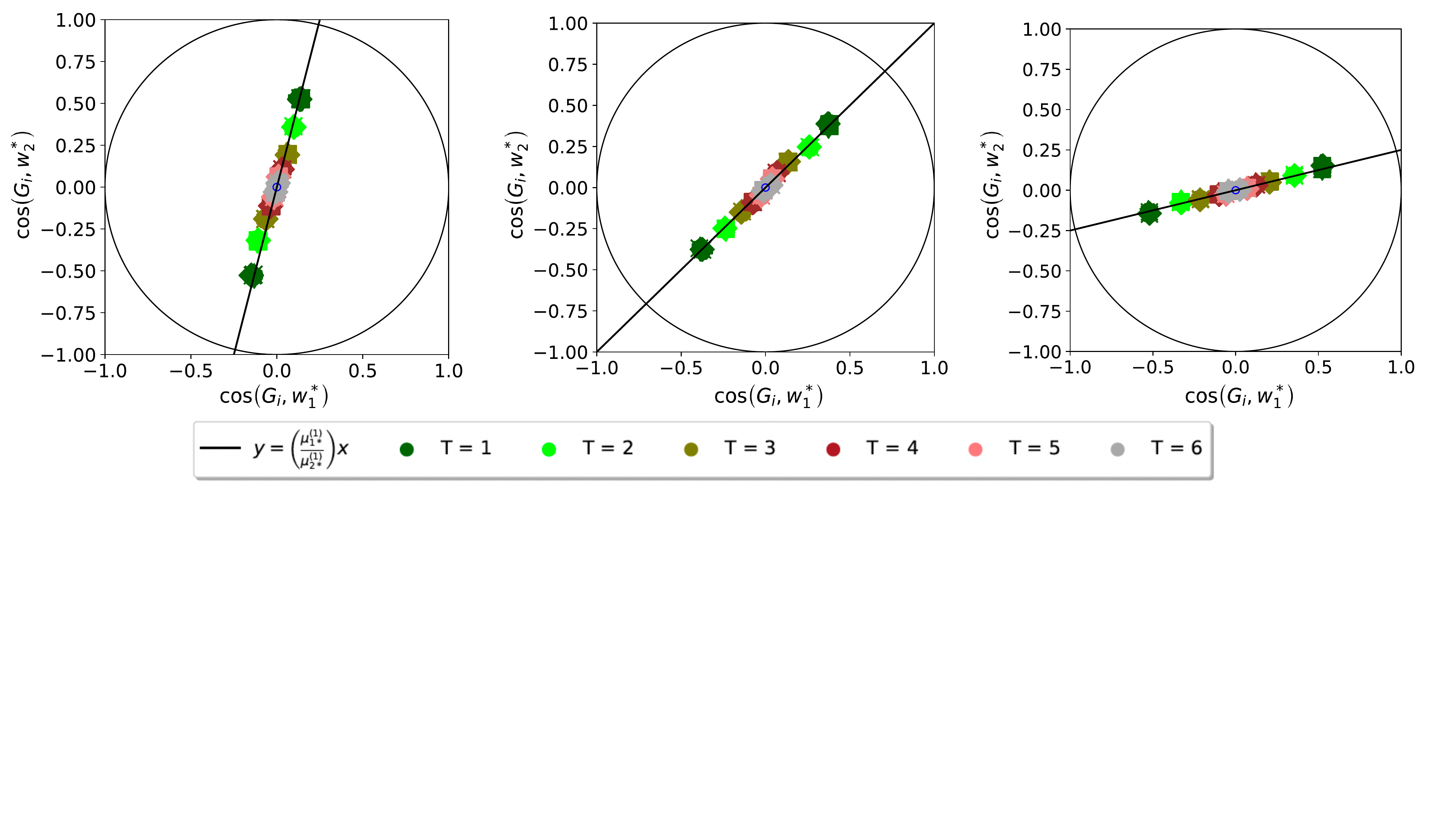}
    \vspace{-9.5em}
    \caption{{\bf Lack of feature learning after few GD steps.} The plots show the cosine similarity with respect to the teacher vectors $(\vec{w}^\star_1,\vec{w}^\star_2)$ for the gradient at different stages of the training. The predicted orientation (Thm.~\ref{thm:one_step_learning}) of the gradient is shown as a continuous black line, the circle of unitary radius in black, and the circle of radius $\sfrac{2}{\sqrt{d}}$ in blue. We fix $n = d = 2^{13}$, the learning rate $\eta = p$, and we use a $\rm{relu}$ student. We vary the teacher functions: \textbf{Left:} $\sigma^\star_1(z) =  4z^2 + z$, $\sigma^\star_2(z) = z$  \textbf{Center:} $\sigma^\star_1(z) = \sigma^\star_2(z) =  4z^2 + z$, \textbf{Right:}  $\sigma^\star_2(z) =  4z^2 + z$, $\sigma^\star_1(z) = z$. 
 The orientation of the gradient does not change after $T=6$ steps preventing specialization, in agreement with Thm.~\ref{thm:staircase}.}
\label{fig:multiple_steps_non_staircase}
    \vspace{-1em}
\end{figure}

\subsection{Learning with multiple steps}
We move now the numerical investigation of the learning behavior of two layer networks after multiple gradient steps. The general picture of the phenomenology is offered in Fig.~\ref{fig:fig_2}, following the theoretical characterization of Theorem \ref{thm:staircase}.

\begin{figure}[t]
    \centering    
\includegraphics[width = 0.48\textwidth]
{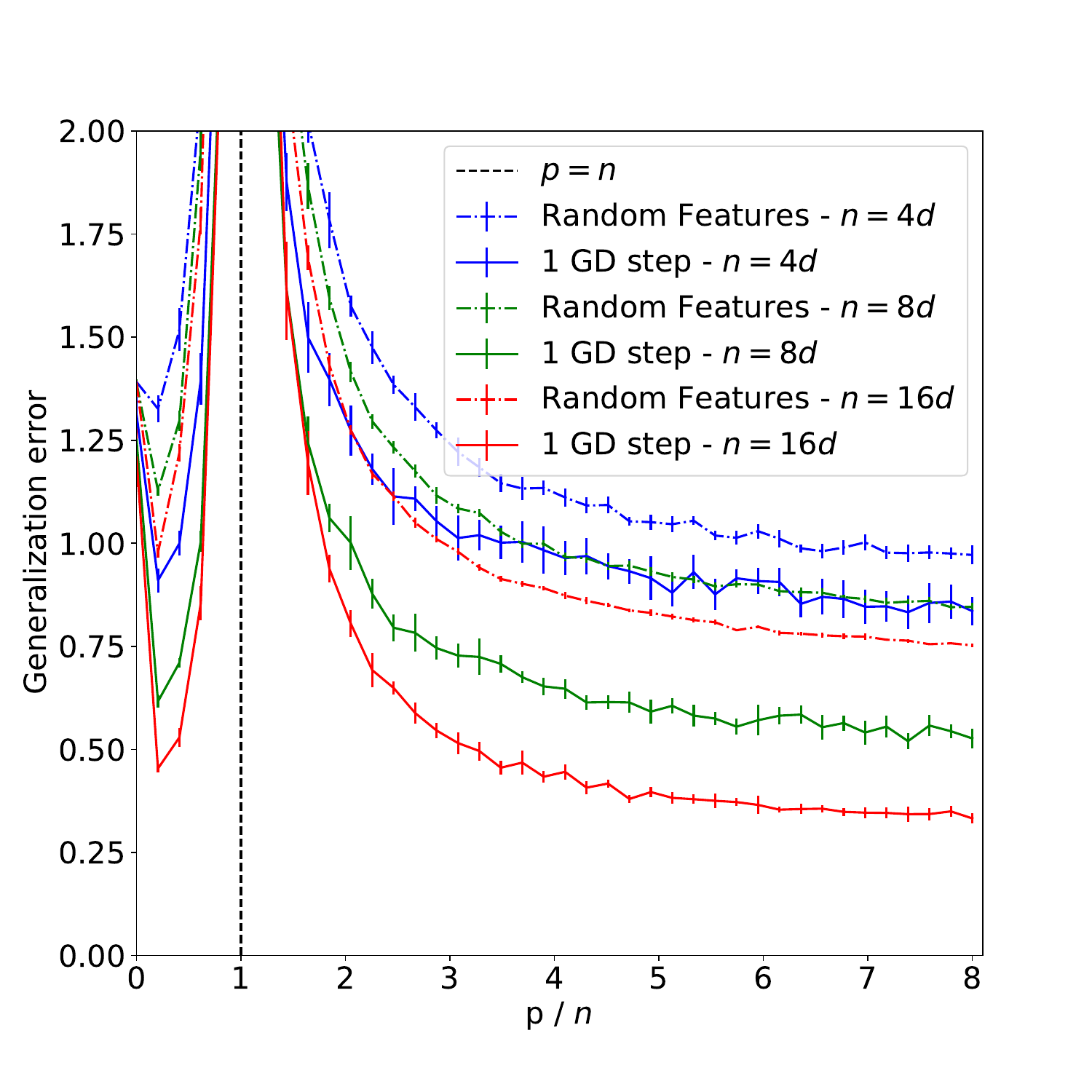}
\includegraphics[width = 0.48\textwidth]
{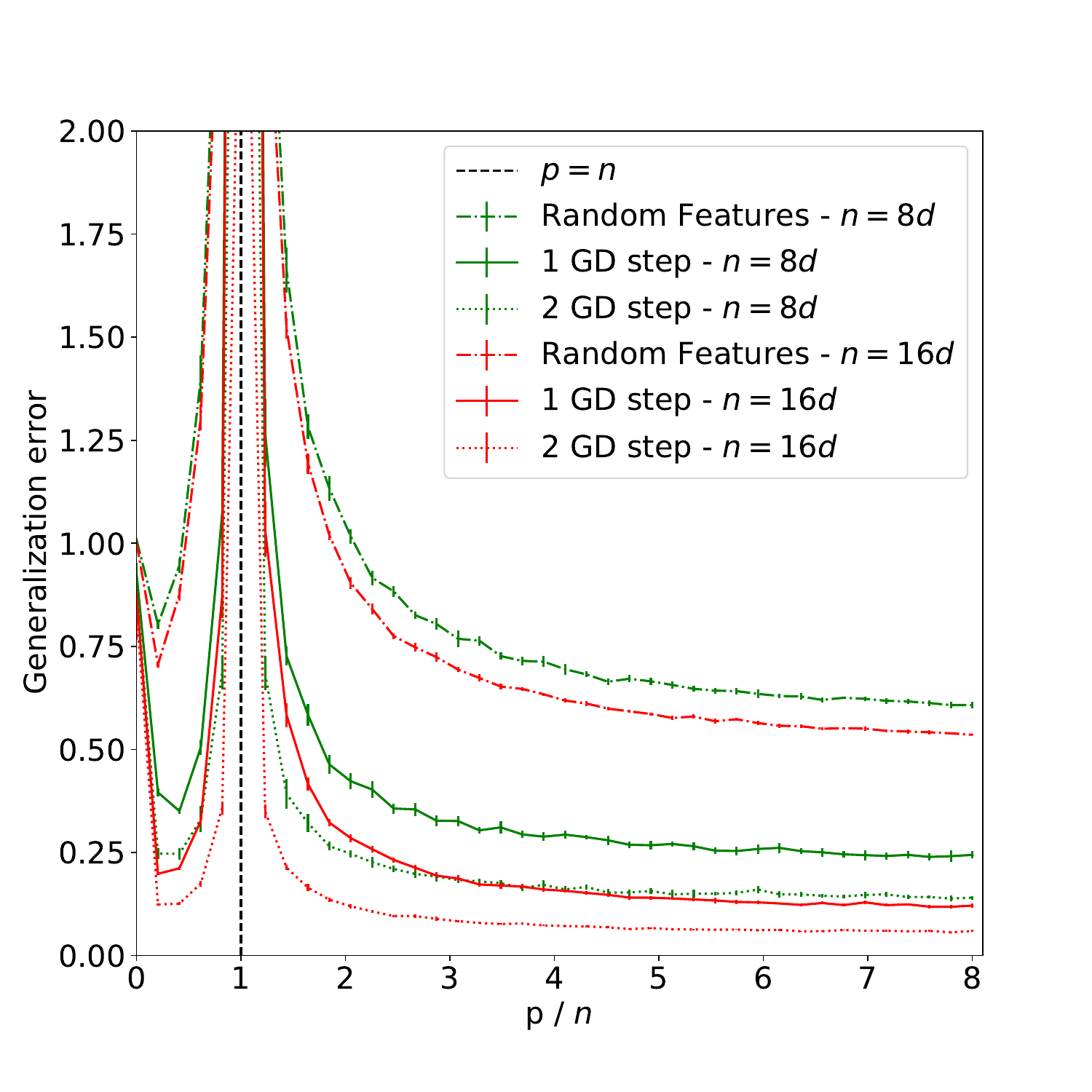}
        \caption{{\bf Learning 3-index targets with overparametried networks}. We illustrate how overparametrization helps improving generalization over random features in accordance with Table~\ref{table2}. The plot shows the generalization error as a function of the number of hidden neurons $p$ normalized by the number of samples used for the first layer training $n=\{4d,8d,16d\}$ with fixed dimension $d=2^8$. \textbf{Left}: $f^\star(\vec z)=z_1 + z_1 z_2 + g_3(z_1)z_3$, where the auxiliary function has leap index $\ell = 3$, here $g(z) = \tanh{z} - z\mathbb{E}_{\xi \sim \mathcal{N}(0,1)}\left[\xi \tanh{\xi}\right]$. While random features can only fit a linear model in this case, one step of gradient descent over $W$ allows to fit the $z_1 z_2$ and $g(z_1)z_3$ parts, resulting in a much lower generalization error with respect to random features, despite only the direction $z_1$ being learned in $W$. \textbf{Right}: $f^\star(\vec z) = z_1+ z_1z_2 + g_3(z_2)z_3$, where $g(\cdot)$ is defined above. In this case, two steps of gradient descent allow to improve generalization over both Random Features and one GD step by allowing the network to fit the $g(z_2)z_3$ term, linealry connected to $z_2$ that is learned only at the second step through the staircase property (Thm.~\ref{thm:staircase}).}
    \label{fig:postive_cget_appendix}
\end{figure}
\paragraph{Investigating the generalization performance} First, we investigate the generalization behavior in the upper panel of Fig.~\ref{fig:multiple_steps}. We modify slightly the training procedure in Alg.~\ref{alg:gd_training} to
perform the numerical experiments: at every gradient step on the first layer weights we train the second layer sequentially with ridge regression. The analysis of the test error behavior in the upper panel of Fig.~\ref{fig:multiple_steps} sheds light on the consequences of Thm.~\ref{thm:staircase} on the generalization performance of two-layer networks. Indeed, we observe a clear benefit in performing multiple gradient steps if the teacher function has a direction linearly connected to the rank-one spike in the gradient identified by $C_1(f^\star)$ (right panel), while if such linearly connected direction does not exist (left panel) the generalization performance does not improve relevantly over time, and the network is stuck on the ``linear subspace learning'' (see the upper right plot of Fig.~\ref{fig:fig_2}). These results are in perfect agreement with Thm.~\ref{thm:staircase}. 

\paragraph{Investigating representation learning efficiency:}In this paragraph we further analyze the claims of Thm.~\ref{thm:staircase} in the context of feature learning. The experiments done in the lower panel of Fig.~\ref{fig:multiple_steps} are closely related to the ones of Fig.~\ref{fig:d2_d3_regimes}. However, contrary to the previous setting, we study the cosine similarity of the \textit{projected gradient} $G^p =  G \Pi^\star$ in the teacher subspace. This quantity differs from the cosine similarity of the full gradient,  plotted in Fig.~\ref{fig:d2_d3_regimes}, as we lose completely the information about the share of the gradient lying in the subspace orthogonal to the teacher one. This divergence in the choices is due to the different illustrative goals of the figures: while in Fig.~\ref{fig:multiple_steps} we highlight the change in orientation of the gradient inside the teacher subspace after a few steps, hence not caring about the relative magnitude, in Fig.~\ref{fig:d2_d3_regimes} we contrast the magnitude of the true gradient with the one of a random object (blue circles) in order to claim the presence (or lack) of feature learning after a single step. The results in the lower panel of Fig.~\ref{fig:multiple_steps} are obtained iterating $2$ steps of the training procedure in Alg.~\ref{alg:gd_training}: in accordance with Thm.~\ref{thm:staircase} we observe delocalization of the projected gradient only if there are linearly connected directions that can be exploited to escape the spike given by the first Hermite coefficient $C_1 (f^\star)$ (right panel). Moreover, we are able to theoretically predict the orientation of the gradient at the second step as well (see Appendix.~\ref{sec:appendix:gd_proofs}). On the contrary, when such linearly connected directions do not exist, the gradient is stuck on the spike $C_1(f^\star)$ (Left panel). We elaborate on this last observation by checking that the lack of specialization persists iterating for more than two GD steps. We present the results in Fig.~\ref{fig:multiple_steps_non_staircase}: the gradient is stuck in the single index regime even as the training proceeds, again in agreement with Thm.~\ref{thm:staircase}. Moreover, we illustrate by changing the teacher functions, that the theoretical prediction of Thm.~\ref{thm:one_step_learning} on the gradient orientation, are valid beyond the symmetric teachers.

% Moreover, as we learned already from Fig.~\ref{fig:example}, we necessarily need to suppress the first Hermite coefficient in order  to obtain specialized hidden units in the $n = \mathcal{O}(d^2)$ regime (given the choice of the teacher function on the left panel of Fig.~\ref{fig:multiple_steps}). However, in the left panel of Fig.~\ref{fig:multiple_steps}, we do observe a relevant decrease of the generalization error in the $n = \mathcal{O}(d^2)$ regime without any explicit suppression or preprocessing: the performance after few sequential steps (e.g. grey points) differs from the single step one (dark green points). This suggests that the procedure in Alg.~\ref{alg:learning_curves}, more closely related to real-life applications, is effectively removing the first Hermite coefficient of the teacher. 

\paragraph{Multiple stairs} Exploiting the same visualization framework of Fig.~\ref{fig:multiple_steps}, we complement the picture  by studying a straightforward generalization in order to test Thm.~\ref{thm:staircase} on functions having multiple linearly connected directions to the previously learned one, or informally, "multiple-stairs function". The results are presented in Fig.~\ref{fig:3stairs} by considering the function $f_\star(\Vec{z}) = \sfrac {z_1}{3} + \sfrac{2z_1z_2}{3} + z_2 z_3 $; we consider $3$ steps in the training of Alg.~\ref{alg:gd_training}, the network is able to learn respectively $\Vec{e}_1, \Vec{e}_2, \Vec{e}_3$ after the first three steps of training in the proportional sample regime. This is clearly appreciable by studying the cosine similarity of the projected gradient on the teacher subspace: after the first step it is localized around $\vec{e}_1$, proceeding with training it has projections along $\Vec{e}_2$ while $\Vec{e}_3$ remains hidden, and only at the third step we obtain delocalization of the gradient along $\Vec{e}_3$. These results are in perfect agreement with Thm.~\ref{thm:staircase}. Note that the hierarchical learning framework of Thm.~\ref{thm:staircase} allows neurons to simultaneously specialize along different directions, as exemplified in Fig.~\ref{fig:fig_2} (see the bottom right plot). We observe one instance of this multidirectional staircase learning in Fig.~\ref{fig:120_learning} by considering the target $f_\star(\vec z) = \sfrac{z_1}3 + 2\He_2(z_1)z_2 + z_1 z_3 $: while the results are unchanged in the first step with respect to Fig.~\ref{fig:3stairs} (with only the $\vec e_1$ direction being learned), we observe that both directions $\vec e_2 \, \& \,\vec e_3$ are learned at the second step. 

\paragraph{Benefit of overparametrization}
In this paragraph we investigate benefit of overparametrization for two-layer networks trained with giant steps of GD when learning multi-index teacher functions (See Fig.~\ref{fig:features}). Theorem~\ref{thm:cget} precisely characterizes that in the diverging $p$ limit it is not possible to fit functions that are non-linear conditioned on the knowledge of the spike $C_1(f_\star)$ learned at the first step. However, this conditional form of Gaussian Equivalence does not prevent the learning of functions in orthogonal directions that are linearly connected to $C_1(f_\star)$ (Thm.~\ref{thm:staircase}). The results in Fig.~\ref{fig:features} clearly show these claims: when such linearly connected directions exist (left panel), one (giant) step of gradient descent training surpass the generalization performance of random features, while it is not possible otherwise (right panel). To produce the figure we fix the regularization strength to be infinitesimal $\lambda = 10^{-6}$ and we adapt the learning rate $\eta = 5p\sqrt{\frac{n}{d}}$ for different values of $p$. Moreover, we normalize the squared generalization risk by the variance of the target functions in order to have comparable y-axis for the two panel of Fig.~\ref{fig:features}. We complement these illustrations with an additional plot using the same numerical setup: Fig.~\ref{fig:postive_cget_appendix} extend the claims above to the $3-$index target functions and conclude the analysis of the examples included in Table~\ref{table2}. In both cases displayed in Fig.~\ref{fig:postive_cget_appendix} overparametrization helps improving the generalization with respect to random features and reduces the time iterations needed to learn the target compared to the underparametrized case. Although we do not consider the same example as Table~\ref{table2} to improve numerical stability, the overall claims of the table are left unchanged by substituting the the $\ell-$th Hermite polynomial $\He_{\ell}$ with the auxiliary function $g_{\ell}(\cdot)$ having leap index $\ell$. 
%~% remove gen. plots %%
% The upper panel of Fig.~\ref{fig:3stairs} shows the generalization error learning curve using Alg.~\ref{alg:learning_curves}: in accordance with the theoretical prediction of Thm.~\ref{thm:staircase} we see a decrease of the generalization error as the number of steps in the GD training increases in the $n = \mathcal{O}(d)$ regime. 

 \begin{figure}[t]
      \centering\includegraphics[width = \textwidth]{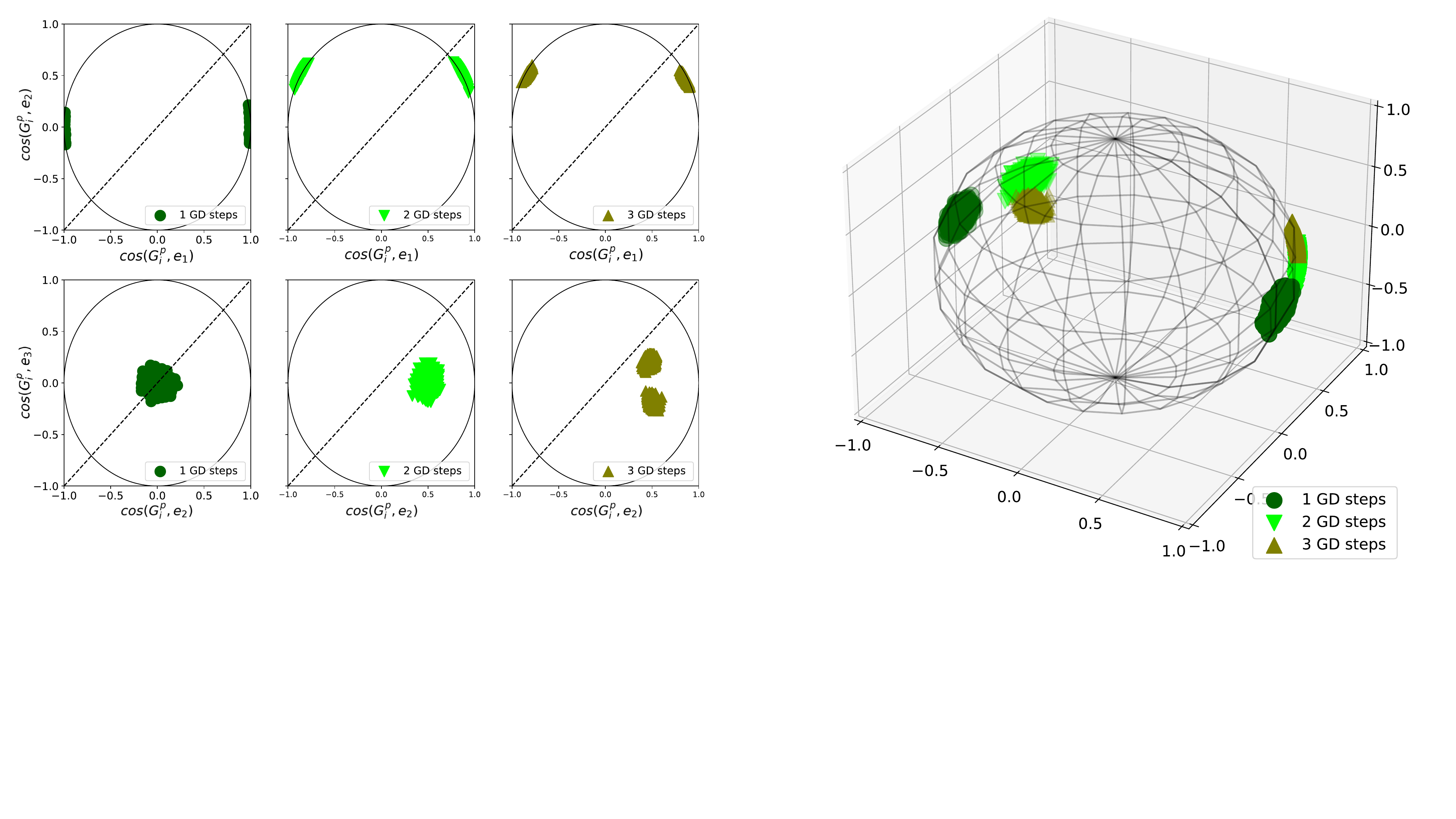}
    \vspace{-8.5em}
    \caption{{\bf Climbing multiple stairs.}  Fix the teacher function $f_\star(\Vec{z}) = \sfrac {z_1}{3} + \sfrac{2z_1z_2}{3} + z_2 z_3 $ and a $\rm{relu}$ student. 
% \textbf{Top:} Generalization error as a function of $\sfrac{\log{n}}{\log{d}}$ ($d = 512, p =256$)  \textbf{Bottom:}
    The plots show the cosine similarity of the projected gradient matrix $G^p$ inside the teacher subspace for all the $p$ neurons at a fixed ratio $\sfrac{n}{d} = 4$, plotted at different stages of the training following Alg.~\ref{alg:gd_training}. The plot shows the similarity in the $3D$ teacher subspace on the right, and two sections of it on the left: \textbf{Up:} $(\Vec{e}_1,\Vec{e}_2)$ plane. \textbf{Bottom:}$(\Vec{e}_2,\Vec{e}_3)$ plane. In accordance with Thm.~\ref{thm:staircase}, the gradient is first localized around $\Vec{e}_1$, then sequentially learns $\Vec{e}_2$, and only at the third step has components along $\Vec{e}_3$. }
        \label{fig:3stairs}
        \vspace{-1em}
    \end{figure}
\vspace{-1em}

\begin{figure}[t]
    \centering\includegraphics[width = \textwidth]{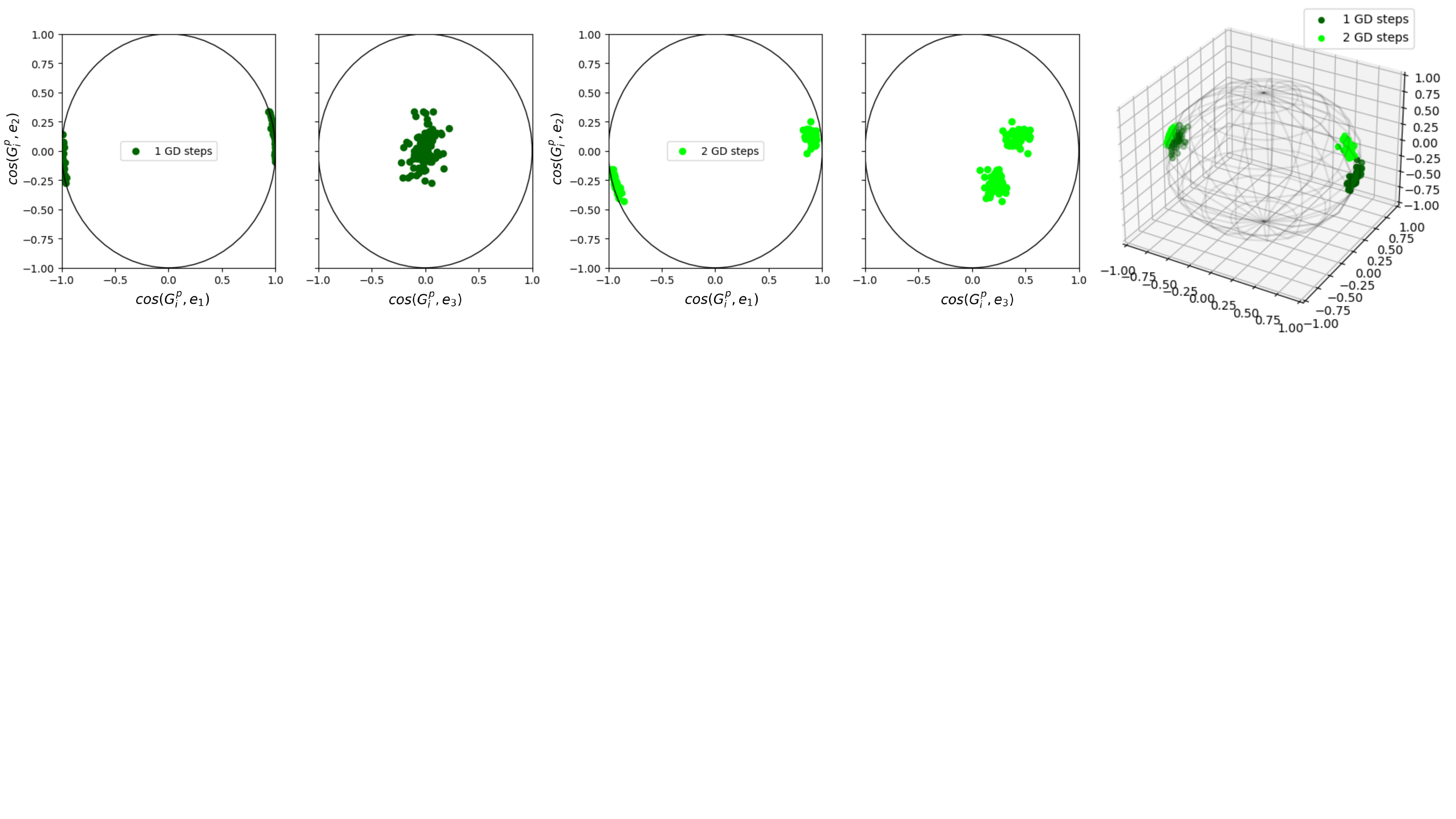}
    \vspace{-14.5em}
    \caption{{\bf Learning multiple directions at a time.}  Fix the teacher function $f^\star(\vec z) = \sfrac{z_1}3 + 2\He_2(z_1)z_2 + z_1 z_3$  and a $\rm{relu}$ student. The plots show the cosine similarity of the projected gradient matrix $G^p$ inside the teacher subspace for all the $p$ neurons at a fixed ratio $\sfrac{n}{d} = 4$, plotted at different stages of the training following Alg.~\ref{alg:gd_training}. The plot shows the similarity measure in different cases. \textbf{Left:} $(\Vec{e}_1,\Vec{e}_2)$ cross section. \textbf{Center:} $(\Vec{e}_3,\Vec{e}_2)$ cross section. \textbf{Right:} $3D$ teacher subspace $(\vec e_1, \vec e_2, \vec e_3)$.
    In accordance with Thm.~\ref{thm:staircase}, the gradient is first localized around the direction $\Vec{e}_1$, and then learns both directions $(\Vec{e}_3,\Vec{e}_2)$ at the second gradient step. }
        \label{fig:120_learning}
        % \vspace{-9em}
    \end{figure}

\begin{algorithm}
\caption{Training procedure}
\begin{algorithmic}[]
\label{alg:gd_training}
    \STATE  {\bfseries Choice of parameters} 
    Fix the data dimension and the width of the second layer $(d,p)$ and sample $(W_0,\Vec{a}_0)$ obeying eq.~\eqref{eq:sample_archit}. Fix a regularization parameter $\lambda$, and a number of GD steps $T_{max}$.
    \FOR{$n$ in a given range}
    \STATE {\bfseries Learning rate tuning} Fix the learning rate $\eta = \mathcal{O}(p\sqrt{\frac{n}{d}})$.
    \FOR{ $t < T_{max}$}
      \STATE {\bfseries Data generation } Sample the data matrix $Z \sim \mathcal{N}(0,I_{n \times d})$ and get the labels $Y = f_\star(Z) \in \mathbb{R}^n$
    \STATE {\bfseries Update first layer} Compute the gradient matrix $G_t =  \{\vec{G}^{(t)}_i\}_{i \in [p]}\in \mathbb{R}^{p \times d}$ and update $W$:
    \begin{align}
        \vec{G}_i^{(t)} &\leftarrow \frac{a_{0,i}}{\sqrt{p}} \cdot \frac1n\sum_{\nu=1}^n \vec{x}^\nu \sigma'(\langle \vec{w}_{i}^{(t)}, \vec{z}^\nu \rangle) \left(\hat{f}(\Vec{z}^{\nu},W_t,\vec{a}_0) - f^\star(\vec{z}^\nu)\right)\\
        W_{t+1} &\leftarrow W_t - \eta G_{t}
    \end{align}
    \IF{$t == T_{max}$}
    \STATE {\bfseries Train second layer} \review{Repeat the data generation step. Get the feature matrix  $X  \leftarrow \sigma(W_t Z)$ and compute the ridge estimator: 
    $$ \hat{\Vec{a}} \leftarrow \begin{cases}
        X^\top \left(XX^\top + \lambda I_n\right)^{-1}Y  & \textif n<p \\
        \left(X^\top X + \lambda I_p\right)^{-1}X^\top Y & \textif n>p \end{cases}  $$}
    \ENDIF
    \ENDFOR
    \ENDFOR
\end{algorithmic}
\end{algorithm}

%% file: reviewed_section/appendix/proofs.tex
\clearpage
\section{Gradient descent on the first layer}
\label{sec:appendix:gd_proofs}

\subsection{Technical assumptions}

We shall show our results under the following assumptions. First, since we assume that the leap index of $f^\star$ is at least one, and we setup the network to zero output the following assumption is unrestrictive:
\begin{assumption}
    The teacher function $f^\star$ and the student activation $\sigma$ both have their zero-th Hermite coefficient equal to 0. 
\end{assumption}

We shall also need a smoothness assumption:
\begin{assumption}\label{assump:smoothness}
    Both the student activation $\sigma$ and $g^*$ are continuous, and differentiable except possibly on a finite set of points. Further, the first two derivatives of $g^\star$ and the first three derivatives of $\sigma$ are bounded in $\dR$. 
\end{assumption}

\review{As we show in the proof of Theorem \ref{thm:staircase}, the above assumption can be relaxed to $\sigma, g^\star$ with polynomially bounded derivatives.}

\subsection{Preliminaries}

\paragraph{More on Hermite expansion}

We recall a few properties of the Hermite tensors of Definition \ref{def:hermite}. Up to symmetry, the tensors $\cH_k(\vec{x})$ are an orthonormal basis of $\ell^2(\dR^m, \gamma)$, in the sense that for any $\vec{i}, \vec{j} \in \dR^k$: 
\begin{equation}\label{eq:hilbert_orthonormal}
    \langle \cH_{k, \vec{i}}(\vec{x}), \cH_{k, \vec{j}}(\vec{x}) \rangle_{\gamma} = \frac{1}{|\mathfrak{o}(\vec{i})|} \ind_{\vec{i} \text{ is a permutation of } \vec{j}}
\end{equation}
where $|\mathfrak{o}(\vec{i})|$ is the number of distinct permutations of $\vec{i}$ \review{and $\mathcal{H}_{k}$ denote the Hermite tensors defined in~\ref{def:hermite}}.
It can be checked from the definition in \cite{grad_1949_note} that the $\cH_k$, and hence the $C_k(f)$, are basis-invariant, and hence represent an actual $k$-linear form on $\dR^m$. Further, the property \eqref{eq:hilbert_orthonormal} yields an immediate expression for the scalar product in $\ell^2(\dR^m, \gamma_m)$:
\begin{equation}\label{eq:app:hermite_scalar}
    \langle f, g \rangle_\gamma = \sum_{k \in \dN} \langle C_k(f), C_k(g) \rangle.
\end{equation}

Further, the Hermite coefficients of low-rank functions are straightforward to compute:
\begin{lemma}\label{lem:app:low_dim_hermite}
    Let $g: \dR^r \to \dR$, and a linear map $A \in \dR^{r \times d}$ such that $AA^\top = I_r$. Then the Hermite coefficients of $f(\vec{x}) = g(A \vec{x})$ are 
    \begin{equation}
        C_k(f) = C_k(g) \cdot (A, \dots, A),
    \end{equation}
    where $\cdot$ is the multilinear multiplication operator \citep{greub2012multilinear}.
\end{lemma}
In particular, this implies that the singular vectors of $C_k^\star$ all belong to $V^\star$.

\paragraph{Concentration in Orlicz spaces} We recall the classical definition of Orlicz spaces:
 \begin{definition}
     For any $\alpha \in \dR$, let $\psi_\alpha(x) = e^{x^\alpha} - 1$. Let $X$ be a real random variable; the \emph{Orlicz norm} $\norm{X}_{\psi_\alpha}$ is defined as
     \begin{equation}
         \norm{X}_{\psi_\alpha} = \inf \left\{t > 0\::\: \dE\left[ \psi_\alpha\left(\frac{|X|}{t} \right)\right] \leq 1\right\}
     \end{equation}
 \end{definition}
 We refer to the monographs \cite{ledoux_1991_probability, vaart_1996_weak} for more information. We say that a random variable is sub-gaussian (resp. sub-exponential) if its $\psi_2$ (resp. $\psi_1$) norm is finite. The main use of this definition is the following concentration inequality: for a variable $X$ with finite Orlicz norm,
 \begin{equation}\label{eq:app:orlicz_concentration}
     \Pb*{\left| X - \dE X \right| > t\norm{X}_{\psi_\alpha}} \leq 2e^{-t^\alpha}.
 \end{equation}
 The Orlicz norms are sub-multiplicative, in the following sense:
 \begin{lemma}\label{lem:app:orlicz_submult}
     Let $X$ and $Y$ be two random variables. Then, for any $\alpha > 0$, there exists a constant $K_\alpha$ such that
     \begin{equation}
         \norm{XY}_{\psi_{\alpha/2}} \leq K_\alpha \norm{X}_{\psi_\alpha} \norm{Y}_{\psi_\alpha}
     \end{equation}
 \end{lemma}
 Finally, we shall use the following theorem:
 \begin{theorem}[Theorem 6.2.3 in \cite{ledoux_1991_probability}] \label{thm:app:orlicz_sum}
     Let $X_1, \dots, X_n$ be $n$ independent random variables with zero mean and second moment $\dE X_i^2 = \sigma_i^2$. Then,
     \begin{equation}
         \norm*{\sum_{i=1}^n X_i}_{\psi_\alpha} \leq K_\alpha \log(n)^{1/\alpha} \left(\sqrt{\sum_{i=1}^n \sigma_i^2} + \max_{i}\norm{X_i}_{\psi_\alpha} \right)
     \end{equation}
 \end{theorem}

\paragraph{Preliminary computations} We begin with a few useful preliminary computations. First, since $\vec{w}_i^0 \sim \Unif(\dS^{d-1})$, the following lemma holds:
\begin{lemma}\label{lem:app:bound_m}
    With probability at least $1 - cpe^{-c\log(d)^2}$, we have for any $i \in [p]$ and $k \in [r]$:
    \begin{equation}
        \norm{\vec{\pi}_i^0} \leq \frac{\sqrt{r}\log(d)}{\sqrt{d}} 
    \end{equation}
\end{lemma}

Let $\vec{g}_i$ be the negative gradient for the $i$-th neuron at initialization:
\begin{equation}
    \vec{g}_i = -\nabla_{\vec{w}_j} \cL\left(\hat f(\vec{z}^\nu; W^0, \vec{a}), f^\star(\vec{z}^\nu)\right).
\end{equation}
Since at initialization the output of the network is exactly zero, we have
\begin{equation}\label{eq:grad}
    \vec{g}_i = \frac{a_i}{\sqrt{p}} \cdot \frac1n\sum_{\nu=1}^n \vec{z}^\nu \sigma'(\langle \vec{w}^{0}_i, \vec{z}^\nu \rangle) f^\star(\vec{z}^\nu)
\end{equation}
Finally, the update equation for $\norm{\vec{w}_i}$ reads
\begin{equation}\label{eq:app:norm_update}
    \norm{\vec{w}_i^1}^2 = 1 + 2 \eta \langle \vec{w}_i^0, \vec{g_i} \rangle + \eta^2\norm{\vec{g}_i}^2
\end{equation}

\subsection{Computing expectations}

We begin by a simple computation of the expectation of $\vec{g}_i$: 
% \comment{Define $c_k$}
\begin{lemma}\label{lem:app:grad_exp}
    For any $i \in [p]$, we have
    \begin{equation}
        \dE[\vec{g}_i] = \frac{a_i}{\sqrt{p}}\left(\sum_{k = 0}^\infty c_{k+2}\, \langle (\vec{w}_i^0)^{\otimes k}, C_{k}^\star \rangle \,\vec{w}_i + \sum_{k = 0}^\infty c_{k+1}\,  C_{k+1}^\star \times_{1 \dots k} (\vec{w}_i^0)^{\otimes k} \right)
    \end{equation}
        where the last multiplication is a product over the first $k$ axes of $C_{k+1}$ (and thus results in a vector) \review{and $(c_k)_{k \leq 0}$ denote the Hermite coefficients of $\sigma$}.
\end{lemma}

\begin{proof}
    By Stein's lemma, for any $\vec{w}$, we have
    \begin{align*}
        \dE\left[ \vec{z} \sigma'(\langle \vec{w}, \vec{z} \rangle) f^\star(\vec{z})\right] &= \dE\left[ \nabla_{\vec{z}}\sigma'(\langle \vec{w}, \vec{z} \rangle) f^\star(\vec{z})\right] + \dE\left[ \sigma'(\langle \vec{w}, \vec{z} \rangle) \nabla_{\vec{z}}f^\star(\vec{z})\right] \\
        &= \vec{w} \dE\left[ \sigma''(\langle \vec{w}, \vec{z} \rangle) f^\star(\vec{z})\right] + \dE\left[ \sigma'(\langle \vec{w}, \vec{z} \rangle) \nabla_{\vec{z}}f^\star(\vec{z})\right] 
    \end{align*}
    From Lemma \ref{lem:app:low_dim_hermite}, the $k$-th Hermite coefficient of $\vec{z} \mapsto \sigma''(\langle \vec{w}, \vec{z} \rangle)$ is $c_{k+2} \, \vec{w}^{\otimes k}$, where the $(c_k)_{k \geq 0}$ are the Hermite coefficients of $\sigma$. By two applications of the scalar product formula \eqref{eq:app:hermite_scalar}, we find
    \begin{align}
    \label{eq:app:herm_corr}
        \dE\left[ \vec{z} \sigma'(\langle \vec{w}, \vec{z} \rangle) f^\star(\vec{z})\right] &= \sum_{k = 0}^\infty c_{k+2}\, \langle \vec{w}^{\otimes k}, C_{k}^\star \rangle \vec{w} + \sum_{k = 0}^\infty c_{k+1}\,  C_{k+1}^\star \times_{1 \dots k} \vec{w}^{\otimes k}.
    \end{align}
\end{proof}

\paragraph{Truncating the expansions} Now, we show that the expectations in Lemma \ref{lem:app:grad_exp} can be truncated at the leap index term.

\begin{lemma}\label{lem:app:grad_exp_bounds}
    With probability at least $1 - cpe^{-c\log(d)^2}$, for every $k \geq 0$ and $i \in [p]$, we have
    \begin{equation}
        \left| \langle C_k^\star, (\vec{w}_i^0)^{\otimes k} \rangle \right| \leq c \, \left(\frac{\sqrt{r} \log(d)}{\sqrt{d}} \right)^{k} \quand \norm*{C_{k+1}^\star \times_{1 \dots k} (\vec{w}_i^0)^{\otimes k}} \leq c \left(\frac{\sqrt{r} \log(d)}{\sqrt{d}} \right)^{k} 
    \end{equation}
    As a result, if $\ell$ is the leap index of $f^\star$,
    \begin{equation}\label{eq:app:grad_exp_truncated}
        \norm*{\dE[\vec{g}_i] - \review{a_i}C_{\ell}^\star \times_{1 \dots (\ell-1)} (\vec{w}_i^0)^{\otimes (\ell-1)}} = \cO\left( \frac{r^{\ell/2}\polylog(d)}{d^{\ell/2}} \right)
    \end{equation}
\end{lemma}
% \comment{Where is the term proportional to $a_i$? It propagates to subsequent equations.}
\begin{proof}
    First, we have by Lemma \ref{lem:app:low_dim_hermite},
    \begin{equation*}
        \left| \langle C_k^\star, (\vec{w}_i^0)^{\otimes k} \rangle \right| = \left| \langle C_k(g^\star), (W^\star\vec{w}_i^0)^{\otimes k}\rangle \right| \leq \norm{C_k(g^\star)}_2 \cdot \norm{\vec{\pi}_i^0}^k,
    \end{equation*}
    where $\norm{C_k(g^\star)}_2$ is the operator norm of $C_k(g^\star)$. Since
    \[ \norm{C_k(g^\star)}_2 \leq \norm{C_k(g^\star)}_F \leq \norm{g^\star}_\gamma, \]
    the first inequality ensues by Lemma \ref{lem:app:bound_m}. Now, let $A_{k+1}$ be the $(k+1)$-th mode unfolding of $C_{k+1}(g^\star)$; then
    \[ \norm*{C_{k+1}^\star \times_{1 \dots k} (\vec{w}_i^0)^{\otimes k}} = \norm*{A_{k+1} (W^\star \vec{w}_i^0)^{\otimes k}} \leq \norm{A_{k+1}}_2  \norm{\vec{\pi}_i^0}^k  \]
    The norm of $A_{k+1}$ is then bounded by the same argument as above. The final equality is obtained by using the above bounds on every term above $k = \ell$ in the first sum, and above $k = \ell-1$ in the second.
\end{proof}

\paragraph{Student norms} We now move on to controlling \eqref{eq:app:norm_update}, in expectation. We begin with the cross-term:
\begin{lemma}\label{lem:app:student_cross_exp}
    With probability at least $1 - cpe^{-c\log(d)^2}$, we have for any $i \in [p]$,
    \begin{equation}
        \dE\left[ \langle \vec{w}_i^0, \vec{g}_i \rangle \right] = \cO\left( \frac{r^{\ell/2}\polylog(d)}{p d^{\ell/2}} \right)
    \end{equation}
\end{lemma}
\begin{proof}
    From eq. \eqref{eq:app:grad_exp_truncated}, we have
    \[  \dE\left[ \langle \vec{w}_i^0, \vec{g}_i \rangle \right] = \frac{a_i}{\sqrt{p}} \left(\langle C_\ell^\star, (\vec{w}_i^0)^{\otimes \ell} \rangle + \cO\left( \frac{r^{\ell/2}\polylog(d)}{d^{\ell/2}} \right)\right)\]
    The first part of Lemma \ref{lem:app:grad_exp_bounds} gives
    \[ \langle C_\ell^\star, (\vec{w}_i^0)^{\otimes \ell} \rangle = \cO\left( \frac{r^{\ell/2}\polylog(d)}{p d^{\ell/2}} \right), \]
    and the lemma follows since $|a_i| \leq 1/\sqrt{p}$.
\end{proof}

\bigskip

The main object of study is therefore $\norm{\vec{g}_i}^2$. We can write it as 
% \comment{$p^2$ in the denominator should be $p$ and propagates}
\begin{align}
    \norm{\vec{g}_i}^2 &= \frac{a^2_i}{n^2 \review{p}} \sum_{\nu, \nu'=1}^n  \langle \vec{z}^\nu, \vec{z}^{\nu'} \rangle \sigma'(\langle \vec{w}_i, \vec{z}^\nu \rangle) f^\star(\vec{z}^\nu)\sigma'(\langle \vec{w}_i, \vec{z}^{\nu'} \rangle) f^\star(\vec{z}^{\nu'})\notag \\
    &= \frac{a^2_i}{n^2\review{p}} \bigg(  \sum_{\nu \neq \nu'} \langle \vec{z}^\nu, \vec{z}^{\nu'} \rangle \sigma'(\langle \vec{w}_i, \vec{z}^\nu \rangle) f^\star(\vec{z}^\nu)\sigma'(\langle \vec{w}_i, \vec{z}^{\nu'} \rangle) f^\star(\vec{z}^{\nu'})\notag  \\
    &\hspace{3em} + \sum_{\nu = 1}^n \norm{\vec{z}^\nu}^2 \sigma'(\langle \vec{w}_i, \vec{z}^\nu \rangle)^2 f^\star(\vec{z}^\nu)^2\bigg) \label{eq:app:grad_norm_decomp}
\end{align}
Since $\vec{z}^\nu, \vec{z}^{\nu'}$ are independent for $\nu \neq \nu'$, this leaves
\begin{equation}\label{eq:app:exp_norm_g_decomp}
    \dE\left[\norm{\vec{g}_i}^2\right] = \frac{n(n-1)}{n^2}\norm*{\dE[\vec{g}_i]}^2 + \frac{a^2_i}{n\review{p}} \dE\left[ \norm{\vec{z}^\nu}^2 \sigma'(\langle \vec{w}_i, \vec{z}^\nu \rangle)^2 f^\star(\vec{z}^\nu)^2 \right]
\end{equation}

We shall only need orders of magnitude for those terms. These are taken care of in the following lemma:
\begin{lemma}\label{lem:app:norm_exp}
    There exists a bounded random variable $X$ independent from $d$ such that, with probability at least $1 - cpe^{-\log(d)^2}$,
    \begin{equation}\label{eq:app:norm_exp_cross}
        \norm*{\dE[\vec{g}_i]}^2 = a^2_i X_i \cdot \frac{\norm*{\vec{\pi}_i^0}^{2(\ell - 1)}}{\review{p}} + \review{\cO\left( \frac{r^{\ell-1/2}\polylog(d)}{\review{p^2}d^{\ell-1/2}} \right)}
    \end{equation}
    where $(X_i)_{i\in [p]}$ are i.i.d copies of $X$.
    Additionally, there exist two constants $c, C$ such that 
    \begin{equation}\label{eq:app:norm_exp_diag}
        c \cdot d \leq  \dE\left[ \norm{\vec{z}}^2 \sigma'(\langle \vec{w}_i, \vec{z} \rangle)^2 f^\star(\vec{z})^2 \right] \leq C \cdot d
    \end{equation}
\end{lemma}
 \begin{proof}
     We begin with \eqref{eq:app:norm_exp_cross}. Define the unit norm vectors
     \[ \vec{r}_i = \frac{W^\star \vec{w}_i^0}{\norm{\vec{\pi_i^0}}}, \]
     since the $\vec{w}_i$ are isotropic, the $\vec{r}_i$ are uniform on $\dS^{r-1}$. Then,
     \[ \norm*{C_{\ell}^\star \times_{1 \dots (\ell-1)} (\vec{w}_i^0)^{\otimes (\ell-1)}}^2 = \underbrace{\norm*{C_\ell(g^\star) \times_{1\dots (\ell-1)} \vec{r_i}^{\otimes (\ell-1)}}^2}_{=: X_i} \cdot \,\norm*{\vec{\pi}_i^0}^{2(\ell-1)}. \]
      The random variables $X_i$ thus defined are i.i.d, independent from $d$, and have positive expectation since $C_\ell(g^\star)$ is nonzero. Equation \eqref{eq:app:norm_exp_cross} then results from the expansion in \eqref{eq:app:grad_exp_truncated}.

\medskip

     We now move on to the second part; first, by Hölder's inequality,
     \begin{equation}\label{eq:norm_holder}
         \dE\left[ \norm{\vec{z}}^2 \sigma'(\langle \vec{w}_i, \vec{z} \rangle)^2 f^\star(\vec{z})^2 \right] \leq \sqrt{\dE\left[ \norm{\vec{z}}^4\right]}\sqrt[4]{\dE\left[ \sigma'(\langle \vec{w}_i, \vec{z} \rangle)^8\right]\dE\left[f^\star(\vec{z})^8 \right]} \leq C \cdot d,
     \end{equation}
     since the last two expectations are independent from $d$. On the other hand, using the same inequality with $\norm{\vec{z}}^2 - d$, we can write
     \begin{equation*}
         \dE\left[ \norm{\vec{z}}^2 \sigma'(\langle \vec{w}_i, \vec{z} \rangle)^2 f^\star(\vec{z})^2 \right] = d \dE\left[ \sigma'(\langle \vec{w}_i, \vec{z} \rangle)^2 f^\star(\vec{z})^2 \right] + \cO(\sqrt{d}).
     \end{equation*}
     Since $\mu_\ell \neq 0$ and $f^\star$ has leap index $\ell$, there exists $\varepsilon > 0$ two subsets $\cA \subseteq \dR, \cB \subseteq V^\star$ of positive measure such that $\sigma'(x)^2 \geq \eps$ if $x \in \cA$ and $f^\star(\vec{z}^\star) > \eps$ if $\vec{z}^\star \in \cB$. From the fact that $\pi_i \leq 1/2$ with high probability, we conclude that the set 
     \begin{equation*}
         \cC := \{ \vec{z}\in\dR^p\ : \ \langle \vec{w}_i, \vec{z} \rangle \in \cA, P_{V^\star} \vec{z} \in \cB  \}
     \end{equation*}
     has positive (Gaussian) measure. It follows that
     \begin{equation}
         \dE\left[ \sigma'(\langle \vec{w}_i, \vec{z} \rangle)^2 f^\star(\vec{z})^2 \right] \geq \gamma(\cC) \eps^2,
     \end{equation}
     which concludes the proof of Eq. \eqref{eq:app:norm_exp_diag}.
 \end{proof}

 \subsection{Concentration}\label{sec:concentration}
 We now move on to concentrating the quantities of interest of the previous section. Our aim will be to show the following proposition:
 \begin{proposition}\label{prop:app:update_concentration}
     With probability at least $1 - Cpe^{-c\log(n)^2} - Cpe^{-c\log(d)^2}$, for any $i\in[p], k\in[r]$,
     \begin{small}
     \begin{align}
         \norm*{\vec{\pi}_i^1 - \dE\left[\vec{\pi}_i^1\right]} &= \cO\left( \frac{\eta \sqrt{r} \log(n)}{p \sqrt{n}} \right) \label{eq:app:m_concentration}\\
         \left|\norm*{\vec{w}_i^1}^2 - \dE\left[\norm*{\vec{w}_i^1}^2\right] \right| &= \cO\left( \frac{\eta \log(n)}{p \sqrt{n}} + \frac{\eta^2 d \log(n)^6}{p^2 n \sqrt{n} } + \frac{\eta^2 \log(d)}{p^2 n \sqrt{d}} + \frac{\eta^2r \log(n)^2}{p^2 n} + \frac{\eta^2 \log(n)^\ell r^{(\ell-1)/2}}{p^2 d^{(\ell-1)/2} \sqrt{n}}\right) \label{eq:app:q_concentration}
     \end{align}
     \end{small}
 \end{proposition}

 Importantly, we do not claim that the whole vector $\vec{w}_i^1$ concentrates; only its norm and its projection on a low-dimensional subspace do. Throughout this section, we define the random vectors
 \begin{equation}
    \vec{X}^\nu = \vec{z}^\nu \sigma'(\langle \vec{w}_i, \vec{z}^\nu \rangle)f^\star(\vec{z}^\nu).
 \end{equation}
 These vectors are i.i.d, with the same distribution as a random vector that we will call $\vec{X}$. 

\paragraph{Concentration of linear functionals} We begin with a simple bound, that implies both Eq. \eqref{eq:app:m_concentration} and the first term of Eq. \eqref{eq:app:q_concentration}.
 \begin{lemma}\label{lem:app:linear_concentration}
     Let $\vec{w}$ be a unit vector in $\dR^d$. There exists a universal constant $c$ such that with probability $1 - 2pe^{-c\log(n)^2}$, for any $i \in [p]$ and $k \in [r]$,
     \begin{equation}
         \left|\langle \vec{w}, \vec{g}_i \rangle - \dE[\langle \vec{w}, \vec{g}_i \rangle]\right| \leq \frac{\log(n)}{p\sqrt{n}}
     \end{equation}
 \end{lemma}

 \begin{proof}
     By Assumption \ref{assump:smoothness}, the function $f^\star$ is Lipschitz, so $f^\star(\vec{z})$ is a sub-gaussian random variable. The same is obviously true for $\langle \vec{w}, \vec{z} \rangle$, and since $\sigma'$ is bounded, the random variable $\langle \vec{w}, \vec{X} \rangle = \langle \vec{w}, \vec{z} \rangle \sigma'(\langle \vec{w}_i^0, \vec{z} \rangle)f^\star(\vec{z})$ is also sub-Gaussian. We can thus apply
      Bernstein's inequality \cite[Corollary 2.8.3]{vershynin2018high} with $t = \log(n)/\sqrt{n}$ to get
     \begin{equation}
          \Pb*{\left|\frac1n \sum_{\nu=1}^n \langle \vec{w}, \vec{X}^\nu \rangle - \dE \langle \vec{w}, \vec{X} \rangle \right| \geq \frac{\log(n)^{2}}{\sqrt{n}}} \leq 2e^{-c \log(n)^2}.
     \end{equation}
     The result ensues upon noticing that $\frac1n \sum \vec{X}^\nu$ differs from $\vec{g}_i$ by a factor of at most $1/p$.
 \end{proof}

 \paragraph{Decomposing the gradient norm} We now move on to the concentration of the term \review{$\norm{\vec{g}_i}^2$}.
 % \comment{Define $q_i$}. 
 This allows us to write
 \begin{equation}
     \review{\norm{\vec{g}_i}^2 - \dE[\norm{\vec{g}_i}^2}] = \frac{a^2_i}{n^2\review{p}} \left( \underbrace{\sum_{\nu = 1}^n \norm{\vec{X^\nu}}^2 - n\dE[\norm{\vec{X}}^2]}_{S_1} + \underbrace{\sum_{\nu \neq \nu'} \langle \vec{X}^\nu, \vec{X}^{\nu'} \rangle - n(n-1) \norm{\dE \vec{X}}^2}_{S_2} \right)
 \end{equation}
 We shall show the concentration of those two terms sequentially.

 \paragraph{Concentrating the norms} We first focus on $S_1$:
 
 \begin{lemma}\label{lem:app:s1_bound}
     Let $i \in [p]$. There exists a constant $c > 0$ such that with probability $1 - e^{-c\log(n)^2}$,
     \begin{equation}\label{eq:app:s1_bound}
    \Pb*{\left| S_1 \right| \geq \log(n)^{6} d\sqrt{n}} \leq e^{-c \log(n)^2}.
\end{equation}
 \end{lemma}

 \begin{proof}
 The random variable $\norm{\vec{z}}/\sqrt{d}$ is sub-gaussian, and so is $f^\star(\vec{z}^\nu)$. By Lemma \ref{lem:app:orlicz_submult} and the Hölder inequality, the random variable $\norm{\vec{X}^\nu}^2$ satisfies
 \[ \norm{ \norm{\vec{X}^\nu}^2}_{\psi_{1/2}} \leq C \cdot d \quand  \Var\left( \norm{\vec{X}^\nu}^2\right) \leq C \cdot d^2 \]
As a result, we can apply Theorem \ref{thm:app:orlicz_sum} to the random variables $\norm{\vec{X}^\nu}^2 - \dE[\norm{\vec{X}}^2]$, which yields
\begin{equation}
\norm*{\sum_{\nu=1}^n \norm*{\vec{X}^\nu}^2 - n \dE\left[\norm{\vec{X}}^2\right]}_{\psi_{1/2}} \leq c \log(n)^2 d \sqrt{n}.
\end{equation}
Equation \eqref{eq:app:s1_bound} is then a consequence of the Orlicz concentration bound \eqref{eq:app:orlicz_concentration}.
\end{proof}

\paragraph{Decomposing the cross-term} We now move on to $S_2$. To handle this sum, we use the following decoupling result from \cite{pena_1995_decoupling}:
\begin{theorem}
    Let $(f_{ij})_{i, j \in [n]}$ be a set of measurable functions from $\cS^2$ to a Banach space $(B, \norm{\cdot})$, and $(X_1, \dots, X_n), (Y_1, \dots, Y_n)$ two sets of independent random variables such that the laws of $X_i$ and $Y_i$ are the same. Then there exists a constant $C > 0$ such that
    \begin{equation}
        \Pb*{\norm*{\sum_{i \neq j} f_{ij}(X_i, X_j)} \geq t} \leq C \Pb*{\norm*{\sum_{i \neq j} f_{ij}(X_i, Y_j)} \geq \frac tC}
    \end{equation}
\end{theorem}
We apply this theorem to the functions $f_{\nu, \nu'}(\vec{X}^\nu, \vec{X}^{\nu'}) = \langle \vec{X}^\nu, \vec{X}^{\nu'} \rangle - \norm{\dE \vec{X}}^2$. Let $\vec{Y}^\nu$ be an independent copy of the $\vec{X}^\nu$ for $\nu \in [n]$, we then have to estimate
\[\Pb*{\left|\sum_{\nu \neq \nu'} \langle \vec{X}^\nu, \vec{Y}^{\nu'}\rangle - n(n-1) \norm{\dE \vec{X}}^2 \right|  \geq t}.\]
For convenience, let $\bar{x} = \norm{\dE \vec{X}}^2$. Since the $\vec{X}^\nu$ are sub-exponential vectors, the scalar product $\langle \vec{X}^\nu, \vec{Y}^\nu \rangle$ has finite $\psi_{1/2}$-norm. The same bound as Lemma \ref{lem:app:s1_bound} then gives that
\begin{equation}
    \Pb*{\left|\sum_{\nu=1}^n \langle \vec{X}^\nu, \vec{Y}^\nu \rangle - n\bar x^2 \right| \geq \sqrt{n} d \log(n)^6} \leq e^{-c\log(n)^2}
\end{equation}
Hence, to show Proposition \ref{prop:app:update_concentration}, we only need to study the overall sum
\begin{equation}
    \tilde S_{2} := \sum_{\nu , \nu'=1}^n \langle \vec{X}^\nu, \vec{Y}^{\nu'}\rangle - n^2 \bar x^2
\end{equation}
Recall that, as in the proof of Lemma \ref{lem:app:norm_exp}, the vector $\dE \vec{X}$ belongs to the space $V_i = V^\star + \vect(\vec{w}_i^0)$. We thus make the decomposition
\begin{equation}
    \vec{X}^\nu = \vec{X}_i^\nu + \vec{X}_\bot^\nu \quand \vec{Y}^\nu = \vec{Y}_i^\nu + \vec{Y}_\bot^\nu
\end{equation}
where $\vec{X}^\nu_i, \vec{Y}_i^\nu \in V_i$. Hence,
\begin{equation}
    \sum_{\nu , \nu'=1}^n \langle \vec{X}^\nu, \vec{Y}^{\nu'}\rangle - n^2 \bar x^2 = \underbrace{\langle \sum_{\nu=1}^n\vec{X}_i^\nu, \sum_{\nu=1}^n\vec{Y}_i^{\nu}\rangle -n^2 \bar x^2}_{S_2'} + \underbrace{\langle \sum_{\nu=1}^n \vec{X}_\bot^\nu, \sum_{\nu=1}^n\vec{Y}_\bot^{\nu}\rangle}_{S_2''}
\end{equation}
\paragraph{Bounding the last two terms} The main step in bounding $S_2'$ is the following lemma:
\begin{lemma}\label{lem:app:S_11_norm_concentration}
    With probability at least $1 - Ce^{-c\log(n)^2}$, 
    \begin{equation}
        \norm*{\sum_{\nu=1}^n\vec{X}_i^\nu - n \dE \vec{X}} \leq C \sqrt{r} \log(n) \sqrt{n}
    \end{equation}
\end{lemma}
\begin{proof}
    Let $(\vec{u}_1, \dots, \vec{u}_{r+1})$ be an orthonormal basis of $V_i$. Since we have for any vector $\vec{x} \in V_i$
    \[ \norm{\vec{x}}^2 = \sum_{k=1}^{r+1} \langle \vec x, \vec{u}_k \rangle^2, \]
    it suffices to bound such a scalar product with high probability. Each term of the form $\langle \vec{X}_i^\nu - \dE \vec{X}, \vec{u_k} \rangle$ is a sub-exponential random variable with zero mean and bounded variance, and hence by another application of Bernstein's inequality
    \begin{equation}
        \Pb*{\left|\langle \sum_{\nu=1}^n\vec{X}_i^\nu - n \dE \vec{X}, \vec{u}_k \rangle\right| \geq {\log(n)}{\sqrt{n}}} \leq e^{-c\log(n)^2}
    \end{equation}
    The result ensues from a union bound, and the equivalence of norms in finite-dimensional spaces.
\end{proof}
As an easy corollary of this lemma, we get the following bound on $S_2'$:
\begin{corollary}\label{cor:app:S_11_concentration}
    With probability at least $1 - Ce^{-c \log(n)^2}$,
    \begin{equation}
        S_2' = \cO\left(r n\log(n)^2 + \frac{\log(n)^\ell r^{(\ell-1)/2} n \sqrt{n}}{d^{(\ell-1)/2}}\right)
    \end{equation}
\end{corollary}
\begin{proof}
    We use the following decomposition:
    \begin{align*}
    S_2' &= n \langle \sum_{\nu=1}^n\vec{X}_i^\nu - n \dE \vec{X}, \dE \vec{X} \rangle + n\langle \dE \vec{X}, \sum_{\nu=1}^n\vec{Y}_i^\nu - n \dE \vec{X}\rangle + \langle \sum_{\nu=1}^n\vec{X}_i^\nu - n \dE \vec{X}, \sum_{\nu=1}^n\vec{Y}_i^\nu - n \dE \vec{X}\rangle \\
    &\leq n \norm{\dE \vec{X}} \left(\norm*{\sum_{\nu=1}^n\vec{X}_i^\nu - n \dE \vec{X}} +\norm*{\sum_{\nu=1}^n\vec{Y}_i^\nu - n \dE \vec{X}}  \right) + \norm*{\sum_{\nu=1}^n\vec{X}_i^\nu - n \dE \vec{X}}\cdot \norm*{\sum_{\nu=1}^n\vec{Y}_i^\nu - n \dE \vec{X}}
    \end{align*}
    by the Cauchy-Schwarz inequality. The result ensues from the high probability bounds of Lemma \ref{lem:app:S_11_norm_concentration}, as well as the bound on $\norm{\dE \vec{X}}$ from Lemma \ref{lem:app:norm_exp}.
\end{proof}

We finally bound the last term, which closes the proof of Proposition \ref{prop:app:update_concentration}.
\begin{lemma}\label{lem:app:S_12_concentration}
    Let $i \in [p]$. With probability at least $1 - 2e^{-c \log(n)^2} - e^{-c\log(d)^2}$, we have
    \begin{equation}
        |S_2''| \leq 2\log(d) n\sqrt{d}
    \end{equation}
\end{lemma}
\begin{proof}
    Define $\alpha^\nu = \sigma'(\langle \vec{w}_i^0, \vec{z}^\nu \rangle) f^\star(\vec{z}^\nu)$, and $\beta^\nu$ its equivalent for $\vec{Y}^\nu$. Since $\alpha^\nu$ only depends on $\vec{z}_i$, the distribution of $\sum \vec{X}^\nu_\bot$ is the same as $\norm{\alpha} \vec{X}_\bot$, where $\vec{X}_\bot$ is a normal random vector in $V_i^\bot$. Therefore, we have
    \[ S_2'' \overset{d}{=} \norm{\alpha} \cdot \norm{\beta} \cdot \langle \vec{X}_\bot, \vec{Y}_\bot \rangle\]
    for two independent Gaussian vectors $\vec{X}_\bot, \vec{Y}_\bot$.
    Now, both $\norm{\alpha}^2$ and $\norm{\beta}^2$ are the sum of $n$ sub-exponential random variables with bounded variance, and $\langle \vec{X}_\bot, \vec{Y}_\bot \rangle$ is the sum of $d$ such variables. Hence, by Bernstein's inequality, with probability $1 - 2e^{-c\log(n)^2}$,
    \[ \norm{\alpha}^2 \leq n + \log(n)\sqrt{n} \leq 2 n \quand \norm{\beta}^2 \leq 2n,  \]
    and with probability at least $1 - e^{-c\log(d)^2}$ 
    \[ \langle \vec{X}_\bot, \vec{Y}_\bot \rangle \leq \log(d)\sqrt{d},\]
    which ends the proof.
\end{proof}

\subsection{Proof of Theorems \ref{thm:one_step_lower_bound} and \ref{thm:one_step_learning}}

We begin with a proposition that summarizes everything from the two previous sections.

\begin{proposition}\label{prop:app:grad_final}
Let $\ell$ be the leap index of $f^\star$, and assume that $n = \Omega(d^{\ell- \delta})$ for some $\delta > 0$. There is an event with probability at least $1 - cpe^{-\log(d)^2}$ such that for $i\in [p]$ 
% \comment{a) Missing $\eta$ in first term of $(72)$. b) Is $||\pi||^{2(\ell - 1)}$ in $(73)$?}
 \begin{align}
     \norm*{\vec{\pi}_i^1 - \left( \vec{\pi}_i^0 + \frac{\eta a_i}{\sqrt{p}}\: C_{\ell}^\star \times_{1 \dots (\ell-1)} (\vec{w}_i^0)^{\otimes (\ell-1)} \right)} &= \cO\left( \frac{\review{\eta}r^{\ell/2}\polylog(d)}{\review{p}d^{\ell/2}} + \frac{\sqrt{r}\eta \log(d)}{p\sqrt{n}} \right) \label{lem:app:proj_w_final} \\
     \norm*{\vec{w}_i^1}^2 &= \Theta\left( 1 + \frac{\eta^2 X_i \norm*{\vec{\pi_i^0}}\review{^{2(\ell - 1)}}}{p^2} + \frac{\eta^2 d}{np^2} \right)\label{eq:app:norm_w_final}
 \end{align}
 where the $X_i$ are i.i.d random variables as in Lemma \ref{lem:app:norm_exp}.
\end{proposition}
\begin{proof}
    The proof amounts to checking that all the bounds proven so far are of the right order. The first equality is simply a combination of Lemma \ref{lem:app:grad_exp_bounds} and Proposition \ref{prop:app:update_concentration}. For the second part, notice that Lemma \ref{lem:app:norm_exp} implies that
    \[ \dE\left[ \norm*{\vec{w}_i^1}^2 \right] = \Theta\left( 1 + \frac{\eta^2 X_i \norm*{\vec{\pi_i^0}}^{\review{2(\ell - 1)}}}{p^2} + \frac{\eta^2 d}{np^2} \right), \]
    and it is straightforward (albeit tedious) to check that all bounds in Proposition \ref{prop:app:update_concentration} are negligible with respect to the above expectation.
\end{proof}

\paragraph{Proof of Theorem \ref{thm:one_step_lower_bound}} We first consider the case where $n = \Theta(d^{\ell - \delta})$. A simple triangular inequality yields
\[ \norm{\vec{\pi}_i^1} = \cO\left( \norm{\vec{\pi}_i^0} +  \frac{\eta\norm{\vec{\pi}_i^0}^{\ell - 1}}{p} \right) \]
where the second part is due to Lemma \ref{lem:app:norm_exp}. On the other hand, the middle term in \eqref{eq:app:norm_w_final} becomes negligible w.r.t the rightmost one, so we get
\[ \norm*{\vec{w}_i^1} = \Omega\left( 1 + \frac{\eta \,d^{\delta/2}}{p} \right) \]
This implies
\begin{equation}
    \frac{\norm{\vec{\pi}_i^1}}{\norm*{\vec{w}_i^1}} = \cO\left( \max \left(\norm*{\vec{\pi}_i^0}, \frac{\norm*{\vec{\pi}_i^0}^{\ell - 1}}{d^{\delta/2}}  \right) \right) = \cO\left( \frac{\polylog(d)}{d^{(1 \wedge \delta)/2}}\right)
\end{equation}
where the last inequality is due to Lemma \ref{lem:app:bound_m}.

% \comment{Reviewer 1: not clear how this implies the spanning of $V^\star_\ell$. What does it mean for $S$ to be full rank? What does it mean to asymptotically span? There is no dependence on $d$, but dependence on $r$ can be very poor. [Explain that $w_i^0$ for any $u_{j_1}$ it has components along $u_{j_2}$ \dots. This implies that almost surely ou get $u_{j_1}$. Show by independence of the $a_a$. ] }
\paragraph{Proof of Theorem \ref{thm:one_step_learning}} Now, we take $n = \Omega(d^\ell)$, and \review{$\eta = pd^{(\ell-1)/2}$}. Then, the bounds of Proposition \ref{prop:app:grad_final} become
\[ \norm*{\vec{\pi}_i^1 - a_i d^{(\ell-1)/2} C_{\ell}^\star \times_{1 \dots (\ell-1)} (\vec{w}_i^0)^{\otimes (\ell-1)} } = \cO\left( \frac{\sqrt{r}\polylog(d)}{\sqrt{d}}\right) \quand \norm*{\vec{w}_i^1}^2 = \cO(1) \]
Hence, the first part of Theorem 2 is straightforward: from Lemma \ref{lem:app:norm_exp},
\begin{equation}
    \frac{\norm{\vec{\pi}_i^1}}{\norm*{\vec{w}_i^1}} = \Omega\left( a_i^2 X_i \cdot (\sqrt{d} \norm*{\pi_i^0})^{\ell-1} \right),
\end{equation}
which is a random variable with positive expectation. The latter part is not independent from $d$, but it dominates e.g. a variable of the form $\norm{\vec{z}_r}$ where $\vec{z}_r \sim \cN(0, I_r/2)$ with probability $1 - ce^{-\log(d)^2}$.

For the second part, we write using the higher-order SVD of $C_\ell^\star$
\[ C_{\ell}^\star \times_{1 \dots (\ell-1)} (\vec{w}_i^0)^{\otimes (\ell-1)} = \sum_{j_1, \dots, j_\ell = 1}^{r_\ell} S_{j_1, \dots, j_\ell} \,\langle \vec{w}_i^0, \vec{u}_{j_1}^\star \rangle \dots \langle \vec{w}_i^0, \vec{u}_{j_{\ell-1}}^\star \rangle \, \vec{u}_{j_\ell}^\star \]
which belongs to $V_\ell^\star$.
\review{Finally, by the minimality of the HOSVD, for each $j_\ell \in [r_l]$, $\exists [j_1,\cdots, j_{\ell-1}] \in [r_l]^{\ell-1}$ such that $S_{j_1, \dots, j_\ell} \neq 0$. Therefore, the random variable
\begin{equation}
    u^{i}_{j_\ell} = \sum_{j_1, \dots, j_{\ell-1} = 1}^{r_\ell} S_{j_1, \dots, j_\ell} \,\langle \vec{w}_i^0, \vec{u}_{j_1}^\star \rangle \dots \langle \vec{w}_i^0, \vec{u}_{j_{\ell-1}}^\star \rangle,
\end{equation}
 has full support in $\mathbb{R}$. Furthermore, the all-orthogonality of $S$ implies that \citep{de2000multilinear}:
 \begin{equation}
     \sum_{j_1, \dots, j_{\ell-1} = 1}^{r_\ell} S_{j_1, \dots, j}  S_{j_1, \dots,k} = 0,
 \end{equation}
 whenever $j \neq k$. This implies that $[u^{i}_{1},\cdots, u^{i}_{r_\ell}]$ have full-support in $\mathbb{R}^{r_\ell}$.
 Now, since $u^{i}$ are independent for different $i$, the matrix:
 \begin{equation}
    A=\begin{pmatrix}
         u^{2}_{1},\cdots, u^{i}_{r_\ell}\\
          u^{1}_{1},\cdots, u^{i}_{r_\ell}\\
          \vdots \\
          u^{p}_{1},\cdots, u^{i}_{r_\ell}
    \end{pmatrix},
 \end{equation}
is full-rank with probability $1$. The statement of Theorem \ref{thm:one_step_learning} then follows by noting the absolutely continuity of the pushforward measure of $\lambda_{min}(A)$ w.r.t the Lebesgue measure and considering a small enough neighborhood of the origin.
}

\subsection{Spike+Bulk decomposition}

Having proven Theorems \ref{thm:one_step_lower_bound} and \ref{thm:one_step_learning},  we move to investigate the behavior after multiple gradient steps. First, we relate the discussion above to a ``spike+noise" decomposition of the gradient. 
We start from Equation \eqref{eq:grad}:
\begin{align} 
 \vec{g}_i = \frac{a_i}{\sqrt{p}} \cdot \frac1n\sum_{\nu=1}^n \vec{z}^\nu \sigma'(\langle \vec{w}_i, \vec{z}^\nu \rangle) f^\star(\vec{z}^\nu)
\end{align}
Define $\sigma'_{>1}(u) : \R \rightarrow \R$ as the following function:
\begin{equation}
    \sigma'_{>1}(u)= \sigma'(u)-\mu_1,
\end{equation}
so that $\Ea{\sigma'_{>1}(u)}=0$. We have the following decomposition of the gradient:
\review{\begin{align} 
     \vec{g}_i
&=  
    \frac{a_j}{\sqrt{p}}\frac{1}{n}\mu_1\sum_{i=1}^n y_i\bz_i +\underbrace{\frac{1}{n}\frac{a_j}{\sqrt{p}}\mu_1\sum_{\nu=1}^n \sigma'_{>1}((\bz^\nu)^\top\vec{w}^0) \bz^\nu y^\nu}_{\Delta_i},
\end{align}}
or in matrix form:
\begin{equation}\label{eq:spike+bulk}
    \vec{G} = \vec{u}\vec{v}^\top + \Delta,
\end{equation}
where $ \vec{u}=\frac{\mu_1}{\sqrt{p}}\vec{a},\vec{v} = \frac{1}{n}\sum_{i=1}^n  y_i\bz_i$.
A similar decomposition was utilized in \cite{ba2022high} to provide an asymptotic characterization of the training and generalization errors in the regime $n=\Theta(d)$ and step-size $\eta=\cO(\sqrt{p})$. In particular, they show that the presence of this spike for $\eta=\cO(\sqrt{p})$ is not enough to go beyond the linear kernel regime. 

However, as we see below, it is possible to obtain a precise characterization in the feature learning regime $\eta=\Theta(p)$ and generalizing to multiple steps, with stronger concentration  over the structure of $\Delta$. In particular, we prove that $\Delta$ effectively acts as uniform noise that can be incorporated into the initialization $\bW^{(0)}$.

This is expressed through the following Lemma:
\begin{lemma}\label{lem:delta}
With high probability over the initialization  $W^{0}$, as $n,d \rightarrow \infty$ with $n=\Omega(\max{(p,d)})$, the matrix $\Delta$ satisfies the following:
\begin{enumerate}
    \item For any $\vec{u} \in V^\star$,  with $\norm{\bu}=1$, $\langle \Delta_j, \bu \rangle = \cO\left( \frac{\polylog(d)}{p \sqrt{d}} \right)$.
    \item $\norm{\Delta} = \cO(\polylog d/\sqrt{d})$.
     \item For any $i \neq j, i,j \in [p/2]$ , $\Delta_j^\top\Delta_i = \cO\left(\frac{\polylog(d)}{p^2 \sqrt{d}}\right)$,
\end{enumerate}
where we only consider the first half neurons due to the choice of the symmetric initialization in Equation \eqref{eq:sample_archit}.
\end{lemma}

\begin{proof}
Without loss of generality, we assume that $\mu_1 = 0$ and hence that $\Delta_i = \vec{g}_i$.
By Lemma \ref{lem:app:grad_exp}, since $\mu_1 = 0$; we have $\Ea{\Delta_j^\top \bv}=\cO\left( \frac{\polylog(d)}{p \sqrt{d}} \right)$. Furthermore, from Lemma \ref{lem:app:linear_concentration}, we obtain that, with high probability:
\begin{equation}
    \abs{\Delta_j^\top \bv-\Ea{\Delta_j^\top \bv}} = \cO\left( \frac{\log(n)}{p\sqrt{d}} \right).
\end{equation}
This proves Part (i).
Part (ii) follows from Lemma 14 in \cite{ba2022high}.

It remains to show Part (iii). The same proof as in Proposition \ref{prop:app:update_concentration} (Eq. \eqref{eq:app:q_concentration}) implies that, with high probability,
\begin{equation}
    \langle \vec{g_i}, \vec{g_j} \rangle  = \dE[\langle \vec{g_i}, \vec{g_j} \rangle] +  \cO\left( \frac{\polylog(d)}{p^2\sqrt{d}} \right),
\end{equation}
and hence we only need to bound the expectation $\dE[\langle \vec{g_i}, \vec{g_j} \rangle]$. In turn, the decomposition of Equation \eqref{eq:app:exp_norm_g_decomp} still holds, and we get
\begin{equation}
    \dE[\langle \vec{g_i}, \vec{g_j} \rangle]  \leq \norm{\dE[\vec{g_i}]} \norm{\dE[\vec{g_j}]} + \frac{a^2_i}{np^2}\dE\left[ \norm{\vec{z}}^2 \sigma'(\langle \vec{w}_i^0, \vec{z} \rangle)\sigma'(\langle \vec{w}_j^0, \vec{z} \rangle) f^\star(\vec{z})^2\right]
\end{equation}
Since $\mu_1 = 0$, the bound of Lemma \ref{lem:app:norm_exp} becomes
\[ \norm{\dE[\vec{g_i}]} \leq \frac{\pi_i}{p} = \cO\left(\frac{\log(d)}{p\sqrt{d}} \right),\]
and it remains to bound the cross term. The main argument is the following lemma, which is the generalization (with an identical proof) of Lemma D.4 in \cite{arnaboldi.stephan.ea_2023_high}:
\begin{lemma}\label{lem:app:N_point_lipschitz}
     Let $N \geq 0$ be fixed, and $f_1, \dots, f_N$ be a sequence of functions with bounded first and second derivatives. Consider the function on $N\times N$ matrices
    \begin{equation}
        F(\Sigma) = \review{\dE_{\vec{z} \sim \cN(0, \Sigma)}} [f_1(z_1) \dots f_N(z_N)]
    \end{equation}
    Then, for $\Sigma, \Sigma'$ two semidefinite positive matrices with unit diagonal, we have
    \begin{equation}
        \left|F(\Sigma) - F(\Sigma')\right| \leq C \norm{\Sigma-\Sigma'}_\infty.
    \end{equation}
\end{lemma}
Now, we first have by the same arguments as in Lemma \ref{lem:app:norm_exp}
\[ \dE\left[ \norm{\vec{z}}^2 \sigma'(\langle \vec{w}_i^0, \vec{z} \rangle)\sigma'(\langle \vec{w}_j^0, \vec{z} \rangle) f^\star(\vec{z})^2\right] = d \dE\left[\sigma'(\langle \vec{w}_i^0, \vec{z} \rangle)\sigma'(\langle \vec{w}_j^0, \vec{z} \rangle) f^\star(\vec{z})^2\right] + \cO\left( \sqrt{d} \right),\]
so we only to bound the first term of the RHS. Expanding the definition of $f^\star$, the latter is a sum of $k^2$ terms of the form
\begin{equation*}
    \dE[\sigma'(\langle \vec{w}_i^0, \vec{z} \rangle)\sigma'(\langle \vec{w}_j^0, \vec{z} \rangle) \sigma_k^\star(\langle \vec{w}_k^\star, \vec{z} \rangle)\sigma_{k'}^\star(\langle \vec{w}_k^\star, \vec{z} \rangle)],
\end{equation*}
which falls under the framework Lemma \ref{lem:app:N_point_lipschitz} for $N = 4$. In particular, since $\mu_1 = 0$, $F(\Sigma) = 0$ whenever we have $\Sigma_{1i} = \Sigma_{2j} = 0$ for $i \neq 1, j \neq 2$. Hence, by an application of Lemma \ref{lem:app:N_point_lipschitz}, we have 
% \comment{Define $\sigma^\star$}
\begin{equation*}
    \dE[\sigma'(\langle \vec{w}_i^0, \vec{z} \rangle)\sigma'(\langle \vec{w}_j^0, \vec{z} \rangle) \sigma_k^\star(\langle \vec{w}_k^\star, \vec{z} \rangle)\sigma_{k'}^\star(\langle \vec{w}_k^\star, \vec{z} \rangle)] \leq C \max(\langle \vec{w}_i^0, \vec{w}_j^0 \rangle, \pi_i, \pi_j) \leq C \frac{\log(d)}{\sqrt{d}}
\end{equation*}
with high probability, which ends the proof.
\end{proof}

We next prove that the norm of $\bw_i^1$ after the first gradient step posseses a simplified dimension-independent limit:
\begin{lemma}\label{lem:norm_first_step}
Suppose $n=\Theta(d)$. Then, there exists a constant $C$, such that for any neuron $i$, with high-probability as $d \rightarrow \infty$, with step-size $\eta$:
\begin{equation}
    \norm{\bw_i^1}^2 = 1+\eta C a^2_i + \cO(\frac{\polylog d}{\sqrt{d}})  
\end{equation}
\end{lemma}

\begin{proof}
    Recall Equation \ref{eq:app:exp_norm_g_decomp}:
    \begin{equation}
 \dE\left[\norm{\vec{g}_i}^2\right] = \frac{n(n-1)}{n^2}\norm*{\dE[\vec{g}_i]}^2 + \frac{a^2_i}{np^2} \dE\left[ \norm{\vec{z}^\nu}^2 \sigma'(\langle \vec{w}_i, \vec{z}^\nu \rangle)^2 f^\star(\vec{z}^\nu)^2 \right]
   \end{equation}

Lemma \ref{lem:app:norm_exp} implies that $\norm*{\dE[\vec{g}_i]}^2$ is approximately $a^2_i X_i \cdot \frac{\norm*{\vec{\pi}_i^0}^{2(\ell - 1)}}{p^2}$ for a random variable $X_i$. When $\ell=1$, $X_i$ simply reduces to a constant depending only on $g^*$. The second term can be decomposed as:
\begin{align*}
    \frac{a^2_i}{np^2} \dE\left[ \norm{\vec{z}^\nu}^2 \sigma'(\langle \vec{w}_i, \vec{z}^\nu \rangle)^2 f^\star(\vec{z}^\nu)^2 \right]&=  \frac{a^2_i}{np^2} \dE\left[ d \sigma'(\langle \vec{w}_i, \vec{z}^\nu \rangle)^2 f^\star(\vec{z}^\nu)^2 \right]\\&+\dE\left[ (d-\norm{\vec{z}^\nu}^2) \sigma'(\langle \vec{w}_i, \vec{z}^\nu \rangle)^2 f^\star(\vec{z}^\nu)^2 \right]
\end{align*}
Let $m_0^i \in R^r$ denote the vector of overlaps $\langle \vec{w}^0_i, \vec{w}^*_1\rangle,\cdots,\langle \vec{w}^0_i, \vec{w}^*_k\rangle$
By Holder's inequality, the second term is of order $\cO(\frac{1}{\sqrt{d}})$ while through a change of variables, the first term can be expressed as a function of the overlaps $\langle \vec{w}^0_i, \vec{w}^*_j\rangle$ for $j \in [r]$:
\begin{align*}
    \frac{a^2_i}{np^2}\dE\left[ d\sigma'(\langle \vec{w}_i, \vec{z}^\nu \rangle)^2 f^\star(\vec{z}^\nu)^2 \right]= \frac{a^2_id}{np^2}\dE\left[ \sigma'(\langle \vec{w}_i, \vec{z}^\nu \rangle)^2 f^\star(\vec{z}^\nu)^2 \right]
    &=\frac{a^2_id}{np^2}F_{\sigma,g^*}(M_0)\\
    &=\frac{a^2_id}{np^2}F_{\sigma,g^*}(0) + \cO(\frac{1}{\sqrt{d}})\\
\end{align*}

The result then follows by noting that Lemmas \ref{lem:app:student_cross_exp} and \ref{lem:app:linear_concentration} imply that $\eta \langle \vec{w}, \vec{g}_i \rangle = \cO(\frac{\polylog d}{\sqrt{d}})$ with high probability.
\end{proof}

\subsection{Second step: Proof Sketch for Theorem \ref{thm:staircase}}\label{sec:proof_sketch_thm3}

Before providing detailed proof of Theorem \ref{thm:staircase} for general polynomial activation functions, and a general number of steps, we illustrate the essential idea by analyzing the second gradient step. We suppose that $\frac{n}{d}$ is fixed to a constant $\alpha$. Let $\vec{Z}^0$ denote the batch of inputs used for the first gradient step. We condition on $\vec{Z}^0$ and assume that the high-probability events in Lemma \ref{lem:delta} hold. 
We independently sample another batch of $n$ training inputs $\vec{Z}$ and perform the gradient update:
\begin{align} 
 \vec{g}^{1}_j  = -\nabla_{\vec{w}_j} \cL\left(\hat f(\vec{z}^\nu; W^1, \vec{a}), f^\star(\vec{z}^\nu)\right)
\end{align}

However, unlike the first gradient step, the weights $\vec{w}^1$ are no-longer approximately orthonormal across neurons and contain significant correlation along the teacher subspace.

We have:
\begin{equation}
    \vec{w}^1_j = \eta \frac{a_j\mu_1}{\sqrt{p}} \vec{v} + \vec{w}^0_j + \review{\eta}\Delta_j,
\end{equation}
where $\vec{v}= \frac{1}{n}\sum_{i=1}^n y_i\bz_i$.
By theorem \ref{thm:one_step_learning}, we have that the projection of $\vec{v}$ along the target subspace $V^*$  converges in probability to $C_1(f)$. Let $\vec{v}^*=C_1(f)$.
We show that the alignment of $\vec{v}$ along $\vec{v}^*$ affects the components of the second gradient step along the teacher subspace, allowing the gradient to be sensitive to directions linearly coupled with $\vec{v}^*$ in the target function.

We proceed by analyzing the projection of the above update along a direction in the teacher subspace. Let $\bv_j=P_{V^\star}(\vec{w}_j^1)$ and consider the decomposition $\vec{w}_j^1=\bv_j+P^\perp_{V^\star}(\vec{w}_j^1)$.  We further have from Lemma \ref{lem:norm_first_step} that $\norm{P^\perp_{V^\star}(\vec{w}_j^1)}^2_2$ concentrates to a positive bounded value $c_j$ depending only on $a_j$. For each sample, $\vec{z}_i$ let $\kappa_i= \langle \vec{v}_j , \vec{z}_i \rangle$ denote the projection along the ``signal" $\vec{v}_j$.

We now introduce the following function for $z \in \R$:
\begin{equation}
   \sigma_{\kappa,j}(z) = \sigma(c_j z+\kappa).
\end{equation}
Define $\mu_{1,\kappa,j}=\Eb{z \sim \mathcal{N}(0,1)}{\sigma_{\kappa,j}(z)z}$. Further, let:
\begin{equation}
    \sigma'_{>1,\kappa}(u)= \sigma_{\kappa,j}'(u)-\mu_{1,\kappa,j}.
\end{equation}
% To isolate the spike in the second gradient step along the teacher subspace, we introduce the following ``signal+noise" decomposition:
% \begin{align} 
%     \vec{g}^{1}_j 
% &=  
%     \frac{1}{n}a_j\mu_{1,\kappa_i}\sum_{i=1}^n y_i\bx_i +\underbrace{\frac{1}{n}a_j\mu_1\sum_{i=1}^n \sigma'_{>1,\kappa_i}(\bx_i^\top\vec{w}^1) \bx_i y_i}_{\Delta^1_j}.
% \end{align}

From Lemma \ref{lem:app:grad_exp_bounds}, we have that $P_{V^\star}(\vec{v})\overset{\dP}{\longrightarrow} \vec{v}^\star$. Now, let $\vec{u} \in V^\star$ be a direction in the teacher subspace orthogonal to $\vec{v}^\star$. Using, equation \eqref{eq:app:herm_corr}, we have:
\begin{equation}\label{eq:grad_decom}
    \review{\begin{split}
    \Ea{\langle \vec{u}, \vec{g}^{1}_j \rangle} &= a_j\Ea{(f^\star(\bz)-\hat{f}(\bz, \bW^1,\vec{a}))\sigma'(\langle\bz,\vec{w}_j^1 \rangle)(\langle\bz,\vec{u} \rangle)}\\
     &=a_j\Ea{(f^\star(\bz)\mu_{1,\langle\bz,\vec{v}_j\rangle}(\langle \bz_i,\vec{u} \rangle )}-a_j\Ea{\hat{f}(\bz, \bW^1,\vec{a})\sigma'(\langle\bz_j,\vec{w}^1\rangle)(\langle\bz_i,\vec{u} \rangle)}
     \end{split}}
\end{equation}
Where in the first term we took the expectation over $P^\perp_{V^\star}(\vec{\vec{w}^1})$ since it is orthogonal to the teacher-subspace. 

The second term can be expressed as:
\begin{equation}
\langle (\Ea{\hat{f}(\bz, \bW^1,\vec{a})\sigma'(\langle\bz,\vec{w}^1)_j\bz_i}) , \vec{u} \rangle \end{equation}
We have that $\hat{f}(\bz, \bW^1,\vec{a})$ depends only on the directions $\vec{w}_1^1,\cdots,\vec{w}_p^1$. By Lemma \ref{lem:delta}, each of the directions, satisfies $\langle \vec{w}_i, \vec{u} \rangle = \cO(\frac{\polylog(d)}{p\sqrt{d}})$.
Furthermore, one can show that $(\Ea{\hat{f}(\bz, \bW^1,\vec{a})\sigma'(\langle \bz,\vec{w}^1\rangle)_j\bz_i}$ lies in the span of  $\vec{w}_1^1,\cdots,\vec{w}_p^1$. 
Therefore: 
\begin{equation}
    \Ea{\hat{f}(\bz, \bW^1,\vec{a})\sigma'(\langle\bz_i ,\vec{w}^1\rangle)_j\bz_i})^\top\vec{u} \overset{d \rightarrow \infty}{\longrightarrow} 0
\end{equation}

Now, consider the first term i.e $\Ea{f^\star(\bz)\mu_{1,\langle\bz,\vec{v}_j\rangle}(\langle\bz,\vec{u}\rangle)}$.
Denote by $\vec{v}_u^\star$ a unit vector along $\vec{v}_u^\star$.
Let $\vec{v}_u^\star,\vec{u},\vec{u'_1}\cdots,\vec{u'}_{d-2}$ be an orthonormal basis of $\R^d$. Without loss of generality, assume that  $\vec{v}_u^\star,\vec{u},\cdots,\vec{u'}_{r-2}$. span the teacher subspace $V^\star$. We express $y$ using the product Hermite decomposition under the above basis:
\begin{equation}
    y= f^\star(\vec{z}) = \sum_{j_1,\cdots,j_r=1}^\infty \frac{c^\star_{j_1,\cdots,j_r}}{j_1!j_2!\cdots j_r!}\He_{j_1}(\langle\bv_u^\star,\vec{z}\rangle)\He_{j_2}(\langle\vec{u},\vec{z}\rangle)\cdots \He_{j_r}(\langle\vec{u}_{r-2},\vec{z}\rangle).
\end{equation}
Lemma \ref{lem:app:grad_exp} and \ref{lem:delta} imply that $\vec{v}_j \overset{\dP}{\longrightarrow}  c'_j\vec{v}^\star$ where $c'_j$ denotes the constant $\eta a_j \sqrt{p} \alpha \mu_1$ Since $\vec{u} \perp \vec{u'}_{1},\cdots,\vec{u'}_{r-2}$, only the terms in $y$ corresponding to products of the form $\He_{j_1}(\langle\bv^\star_u,\vec{z}\rangle)\He_{j_2}(\langle \vec{u},\vec{z}\rangle)$ contribute to the expectation $\Ea{y\mu_{1,\langle \bz, \vec{v}_j\rangle }(\langle \bz,\vec{u}\rangle)}$ in the limit $d \rightarrow \infty$. Consider the contribution of one such term:
\begin{equation}\label{eq:shifted_hermit_cont}
   \Ea{\He_{j_1}(\langle \bv_u^\star,\vec{z} \rangle)\He_{j_2}(\langle\vec{u},\vec{z}\rangle)\mu_{1,\langle\bz,\vec{v}_j\rangle,j}\langle\bz,\vec{u}\rangle} \rightarrow   \Ea{\He_{j_1}(\langle\bv_u^\star,\vec{z}\rangle)\He_{j_2}(\langle\vec{u},\vec{z}\rangle)\mu_{1,\langle\bz,c'_j\bv^\star\rangle,j}\langle\bz,\vec{u}\rangle}
\end{equation}
Suppose $j_2 \neq 1$, then $\Ea{\He_{j_2}(\langle\vec{u},\vec{z}\rangle)\langle\bz_i,\vec{u}\rangle}=0$. Therefore, the non-zero contributions arise from terms of the form $\He_{j_1}(\langle\bv_u^\star,\vec{z}\rangle)\langle\vec{u},\vec{z}\rangle$. It can be checked that directions $\vec{u}$ having non-zero terms of this form span $U^\star_2$ as defined in Theorem \ref{thm:staircase}. However, in general, the RHS of equation \ref{eq:shifted_hermit_cont} might be $0$ for some choices of $\sigma$ and $a_j$. Moreover, such non-zero contributions might cancel each other for a chosen direction in $U^\star_2$. Furthermore, to obtain high-probability result on the alignment along $U^\star_t$ for a general number of $t$ steps, one needs to quantitatively propagate the expectations and concentration bounds on the projections and norms of W, and show that the magnitude of the projections can be bounded independent of the dimension.
We tackle these issues in the next section and provide a full proof of Theorem \ref{thm:staircase}.

\subsection{Proof of Theorem \ref{thm:staircase}}

The proof proceeds by induction on the number of time-steps $t$. To avoid certain degeneracy conditions in the proof, we restrict ourselves to polynomial activations.
Let $U_t^\star$ be the learned subspace at time-step $t$ according to the definition \ref{def:subspace_conditioning}.

Let $Q_t \in \R^{p \times p}$ denote the overlap matrix for weights of the first-layer neurons at time $t$, i.e.
$Q^t_{i,j}=\langle \bw^t_i,\bw^t_j \rangle \,\, \forall i,j \in [p]$. Let $M_t  \in \R^{r \times p}$ denote the target-network overlap matrix i.e .$M_{i,j}^t=\langle \bw^\star_i,\bw^t_j \rangle \,\, \forall i \in [p], j \in [r]$. Let $\bW^* \in \R^{r \times d}$ denote the matrix with rows $\vec{w_1}^\star,\cdots, \vec{w_r}^\star$.

We denote by $\vec{Z}_t$, the batch of input sampled at time $t \in [T]$. By assumption $\vec{Z}_1,\cdots,\vec{Z}_T$ are independent. Let $\mathcal{F}_t$ denote the natural filtration associated to $\vec{Z}_1,\cdots,\vec{Z}_T$, i.e $\mathcal{F}_t$ is the $\sigma$-algebra generated by $\vec{Z}_1,\cdots,\vec{Z}_t$, and let $\mu_t$ denote the corresponding joint-measure of $\vec{Z}_1,\cdots,\vec{Z}_t$. 
We let $\vec{g}^t_i$ denote the gradient for the $i_{th}$ neuron at time $t$ obtained using the batch $\bZ^{t+1}$.

For any time $t$, let $r_t$ denote the dimension of $U^*_t$ and let $\bW^*_t \in \R^{r_t \times d}$ denote a matrix with rows forming a basis of $U^*_t$, such that $(\bW^*)\top \bW^*_t$ is independent of $d,n$. Thus, $\bW^*_t$ represents a dimension independent basis of $U^*_t$. Let
$\bm \bv_{j,\ba} \in \R^{r_t}$ denote the projections of $\bw^t_j$ along $\bW_t^*$ i.e $\bv_{j,\ba}=\bW_t^*\bw^t_j$. Similarly, for an input $\bz \in \R^d$, we denote the projection of $\bz$ along $\bW_t^*$ by $\kappa=\bW_t^*\bz$.
In what follows, we shall say that a sequence of events $\mathcal{E}_n$ occurs with high-probability as $n,d \rightarrow \infty$ if there exist constants $c, C > 0$ such that $\mathbb{P}(\mathcal{E}_n) \geq 1 - Cpe^{-c\log(n)^2} + Cpe^{-c\log(d)^2}$

At any timestep $t \geq 1$, we prove that the following statements hold with high probability w.r.t $\mu_t$:

\begin{enumerate} 
     \item $Q^t = \tilde{Q}^t_{{\ba}}+\cO(\frac{\text{polylog}d}{\sqrt{d}}), M^t = \tilde{M}^t_{{\ba}}+\cO(\frac{\text{polylog}d}{\sqrt{d}})$, 
    where $\tilde{Q}^t_{{\ba}},\tilde{M}^t_{{\ba}}$ denote dimension-independent matrices with each entry being a polynomial dependent only on $\ba,t$ of $\bw^t_i$, dependent on the second layer $i,\ba$.
    \item Let $\bv \in U_t^\star$, with $\norm{\bv}=1$ be arbitrary. Denote by $\bv^m \in \R^k$, the components of $\bv$ along $\vec{w_1}^\star,\cdots, \vec{w_r}^\star$ i.e $\bv^m=\vec{W^\star}\bv$. Then
    there exists almost surely non-zero random variables $q_{i,t,\bv^m,\ba}$, independent of $d,n$ such that $\langle \bw_i,\bv \rangle = q_{i,t,\bv_m,\ba}+O(\frac{\polylog}{\sqrt{d}})$. Furthermore, $q_{i,t,\bv_m,\ba}$ are non-constant polynomials in $\ba, v_1,\cdots, v_k$.
    \review{
    \item There exists a basis $\bv^{(1)}, \cdots,\bv^{(r_t)}$ of $U_t^\star$ such that the minimum degree in the sequence of polynomials $[q_{i,t,\bv^{(1)}_m,\ba}, \cdots, q_{i,t,\bv^{(r_t)}_m,\ba}]$ is strictly increasing with the minimum degree term in each of the components depending only on $a_i$.
    }
    \item For any $\bv \in {U^\bot}_t^\star \cap V^\star$, $\abs{\langle \bw_i,\bv \rangle}=O(\frac{\polylog}{\sqrt{d}})$, with high probability, for all $i \in [p]$.
\end{enumerate}

\begin{proof}
We proceed by induction over $t$. Suppose that the statements hold at some timestep $t$. We start by proving that
$(i)$ holds at time $t+1$ in expectation:

\begin{lemma}\label{lem:g_exp_t}
$\Ea{Q_t} = \tilde{Q}^t_{{\ba}}+\cO(\frac{\text{polylog}d}{\sqrt{d}})$ and 
    $\Ea{M_t} = \tilde{M}^t_{{\ba}}+\cO(\frac{\text{polylog}d}{\sqrt{d}})$ where each entry of $\tilde{Q}^t_{{\ba}},\tilde{M}^t_{{\ba}}$ is a polynomial of $\ba$ with degree independent of $d,n$.
\end{lemma}

\begin{proof}

Recall that:
\begin{equation}
\begin{split}
    Q_{i,j}^{t+1}&=\langle \bw^{t+1}_i,\bw^{t+1}_j \rangle
    \\ &=Q_{i,j}^{t}+\eta \langle \bg^{t}_i,\bw^{t}_j \rangle+\eta \langle \bw^{t}_i,\bg^{t}_j \rangle+\eta^2 \langle \bg^{t}_i,\bg^{t}_j \rangle.\\
    M_{i,j}^{t+1}&=\langle \bw^\star_i,\bw^{t+1}_j \rangle
    \\ &=M_{i,j}^{t}+\eta \langle \bw^\star_i,\bg^{t}_j \rangle
\end{split}
\end{equation}\label{eq:qmupd}
By the induction hypothesis, the entries of $\Ea{Q^t}, \Ea{M^t}$ converge with high-probability to polynomial limits with error $\cO\left(\frac{\polylog d}{d}\right)$. 
Therefore, it suffices to show that $\Ea{\langle \vec{g}^t_i,\vec{w}^t_j \rangle}$, $\Ea{\langle \vec{w}^t_i,\vec{w}^t_j \rangle},\Ea{\langle \vec{g}^t_i,\vec{g}^t_j \rangle}$ converge to dimension-inpedendent polynomial limits. First, consider the case $i=j$. We have, analogous to Equation \ref{eq:app:exp_norm_g_decomp}

\begin{equation}\label{eq:app:exp_norm_g_t_decomp}
    \dE\left[\norm{\vec{g}^t_i}^2\right] =  \underbrace{\frac{a^2_i}{np^2} \dE\left[ \norm{\vec{z}^\nu}^2 \sigma'(\langle \vec{w}^t_i, \vec{z}^\nu \rangle)^2 (f^\star(\vec{z}^\nu)-\hat f(\vec{z}^\nu; \bW^t, \vec{a}))^2 \right]}_{T_1}+\underbrace{\frac{n(n-1)}{n^2}\norm*{\dE[\vec{g}_i]}^2}_{T_2}
\end{equation}

The first term $T_1$ can be decomposed as follows:
\begin{small}
\begin{align*}
    \frac{a^2_i}{np^2}\dE\left[ \norm{\vec{z}^\nu}^2 \sigma'(\langle \vec{w}^t_i, \vec{z}^\nu \rangle)^2(f^\star(\vec{z}^\nu)-\hat f(\vec{z}^\nu; \bW^t, \vec{a}))^2 \right]= \frac {da^2_i}{np^2}\dE\left[ \sigma'(\langle \vec{w}^t_i, \vec{z}^\nu \rangle)^2 (f^\star(\vec{z}^\nu)-\hat f(\vec{z}^\nu; \bW^t, \vec{a}))^2 \right]\\+\frac{a^2_i}{np^2}\dE\left[(d-\norm{\vec{z}^\nu}^2)\sigma'(\langle \vec{w}^t_i, \vec{z}^\nu \rangle)^2 (f^\star(\vec{z}^\nu)-\hat f(\vec{z}^\nu; \bW^t, \vec{a}))^2 \right].
\end{align*}
\end{small}
Similar to equation \ref{eq:norm_holder}, Holder's inequality implies that conditioned on the event in $\mathcal{F}_t$ of $Q_t, M_t$ being bounded independent of $d,n$, the second term is of order $\cO(\frac{\sqrt{d}}{n})=\cO(\frac{1}{\sqrt{d}})$.
Consider the first term, conditioned on $\mathcal{F}_t$.
\begin{equation}
   \frac {da^2_i}{np^2}\dE\left[ \sigma'(\langle \vec{w}^t_i, \vec{z}^\nu \rangle)^2 (f^\star(\vec{z}^\nu)-\hat f(\vec{z}^\nu; \bW^t, \vec{a}))^2 \vert  \mathcal{F}_t\right]
\end{equation}
By assumption, $\frac {da^2_i}{np^2}=\frac{a^2_i\alpha}{p^2}$ for some constant $\alpha$. Therefore, by definition of $f^\star(\vec{z}^\nu)$ and $\hat f(\vec{z}^\nu; \bW^t, \vec{a})$, the term inside the expectation only depends on the overlaps of $\vec{z}^\nu$  with the neurons and teacher subspace i.e $\langle \vec{w}^t_1, \vec{z}^\nu \rangle$, $\cdots, \langle \vec{w}^t_p, \vec{z}^\nu \rangle$ $ \langle\vec{w}^\star_1, \vec{z}^\nu$,$ \rangle,\cdots, \langle \vec{w}^\star_r, \vec{z}^\nu \rangle$. By a change of variables the above term can therefore be expressed as an expectation w.r.t the $r+j$ correlated variables corresponding to the above overlaps. 

Concretely, we have:
\begin{equation}
    \frac d{np^2}\dE\left[ \sigma'(\langle \vec{w}^t_i, \vec{z}^\nu \rangle)^2 (f^\star(\vec{z}^\nu)-\hat f(\vec{z}^\nu; \bW^t, \vec{a}))^2 \vert \mathcal{F}_t \right]
    = F_g(Q_t,M_t),
\end{equation}
for some function $F:\R \rightarrow \R$

\begin{lemma}
    $F_g$ is a polynomial in $Q_t,M_t$ independent of $d,n$.
\end{lemma}

\begin{proof}
    By assumption, $\sigma'$ and $f^\star$ are polynomials in $\langle \vec{w}^t_1, \vec{z}^\nu \rangle,\cdots$ $\langle \vec{w}^t_p, \vec{z}^\nu \rangle$ and $\langle \vec{w}^\star_1, \vec{z}^\nu \rangle,\cdots,$ $ \langle \vec{w}^\star_r, \vec{z}^\nu \rangle$ respectively. Therefore, $ \sigma'(\langle \vec{w}^t_i, \vec{z}^\nu \rangle)^2 (f^\star(\vec{z}^\nu)-\hat f(\vec{z}^\nu; \bW^t, \vec{a}))^2$ is a polynomial in the zero mean correlated Gaussian variables $\langle \vec{w}^t_1, \vec{z}^\nu \rangle,\cdots, \langle \vec{w}^t_p, \vec{z}^\nu \rangle,\langle \vec{w}^\star_1, \vec{z}^\nu \rangle,\cdots, \langle \vec{w}^\star_r, \vec{z}^\nu \rangle$. Therefore, by Wick's/Isserlis' theorem \citep{janson1997gaussian,polyak2005feynman}, $F_g$ is a polynomial in $Q_t,M_t$.
\end{proof}

 By the induction hypothesis, with high-probability, $Q_t = \tilde{Q}_{t,\ba} + \cO(\frac{\polylog d}{\sqrt{d}})$ and $\tilde{M}_{t,\ba}+ \cO(\frac{\polylog d}{\sqrt{d}})$, where $\tilde{Q}_{t,\ba},\tilde{M}_{t,\ba}$ denote deterministic matriceswith entries being  polynomial functions of $\ba$. By propagating the errors through the polynomial $F_g$, we obtain that
$F_g(Q_t,M_t)=F_g(\tilde{Q}_{t,\ba},\tilde{M}_{t,\ba})+\cO(\frac{\polylog d}{\sqrt{d}})$.

Next, consider the term $T_2$ in Equation \ref{eq:app:exp_norm_g_t_decomp}. By repeatedly applying Stein's Lemma w.r.t terms $\langle \vec{w}^t_i,\bz\rangle$ for $i \in [p]$ and  $\langle \vec{w}^*,\bz\rangle$ for $j \in [r]$, analogous to Lemma \ref{lem:app:grad_exp}, $\dE[\vec{g}_i]$, can be expressed as a linear combination of  $\vec{w}^t_1,\cdots, \vec{w}^t_p$ and $\vec{w}^*_1,\cdots, \vec{w}^*_r$. Concretely, we have:
\begin{align*}
\dE[\vec{g}^t_i]&=a_i\dE\left[ \vec{z} \sigma'(\langle \vec{w}^t_i, \vec{z} \rangle) (f^\star(\vec{z})-\hat f(\vec{z}^\nu; \bW^t, \vec{a}))\right]\\
&= a_i\dE\left[ \vec{z} \sigma'(\langle \vec{w}^t_i, \vec{z} \rangle) f^\star(\vec{z}) \right]-a_i\dE\left[\vec{z} \sigma'(\langle \vec{w}^t_i, \vec{z} \rangle)\hat f(\vec{z}^\nu; \bW^t, \vec{a}))\right]
\end{align*}

Consider the first-term, by Stein's Lemma, we obtain:
\begin{align*}
    \dE\left[  f^\star(\vec{z})\sigma'(\langle \vec{w}_i, \vec{z} \rangle) \vec{z}\right]&=\dE\left[ \vec{z} g^\star(\langle\vec{w_1}^\star,\vec{z}\rangle,\cdots, \langle\vec{w_r}^\star,\vec{z}\rangle)\sigma'(\langle \vec{w}, \vec{z} \rangle)\right]\\
    &=\dE\left[\sigma'(\langle \vec{w}, \vec{z} \rangle) \nabla g^\star(\langle\vec{w_1}^\star,\vec{z}\rangle,\cdots, \langle\vec{w_r}^\star,\vec{z}\rangle)\right]\\&+\dE\left[ \vec{w}_i\sigma''(\langle \vec{w}_i, \vec{z} \rangle) g^\star(\langle\vec{w_1}^\star,\vec{z}\rangle,\cdots, \langle\vec{w_r}^\star,\vec{z}\rangle)\right]
\end{align*}
By chain rule, $\nabla g^\star(\langle\vec{w_1}^\star,\vec{z}\rangle,\cdots, \langle\vec{w_r}^\star,\vec{z}\rangle)$ can be expressed as a linear combination of $\vec{w_1}^\star,\cdots,$ $ \vec{w_r}^\star$ and $\vec{w}_i^t$ with coefficients being polynomials in $\langle\vec{w_1}^\star,\vec{z}\rangle,\cdots, \langle\vec{w_r}^\star,\vec{z}\rangle$ independent of $d$. 

Therefore, by Wick's theorem \citep{janson1997gaussian,polyak2005feynman}, $\dE\left[  f^\star(\vec{z})\sigma'(\langle \vec{w}_i, \vec{z} \rangle) \vec{z}\right]$ can be expressed as $\sum_{k=1}^r p_k(Q_t,M_t)\vec{w}_k^*+h_i(Q_t,M_t)\vec{w}^t_i$,
where $\{p_k\}_{k=1,\cdots,r}$ and $h_i$ are polynomials independent of $d$. 
Similarly, we obtain $\dE\left[\vec{z} \sigma'(\langle \vec{w}^t_i, \vec{z} \rangle)\hat f(\vec{z}^\nu; \bW^t, \vec{a}))\right]$ as a linear combination of $\vec{w_1}^t,\cdots, \vec{w_p}^t$.
Therefore, $\norm{\dE[\vec{g}^t_i]}^2$ conditioned on $\mathcal{F}_t$ is a polynomial in $Q_t, M_t$. Propagating errors from time $t$, we conclude that $T_2$ can be approximated by polynomials in $Q_t, M_t$ with error $\cO(\frac{\polylog d}{\sqrt{d}})$

Similary, the terms $\Ea{\langle \vec{g}^t_i,\vec{g}^t_j \rangle},\Ea{\langle \vec{w}^*_i,\vec{g}^t_j \rangle}$  converge with high-probability to dimension-independent polynomials in $Q_t, M_t$.

\end{proof}

Next, we prove that $(ii), (iii), (iv)$ hold in expectation:
\begin{lemma}
   Let $\bv \in V^\star$, with $\norm{\bv}=1$ be arbitrary with components $\bv^m \in \R^r$ along $\vec{w_1}^\star,\cdots, \vec{w_r}^\star$, then $\Ea{\langle \vec{v}, \vec{g}^{t}_j \rangle} =h_t(j,\vec{v}^m,\vec{a},Q_t,M_t)+\cO(\frac{\polylog d}{p\sqrt{d}})$,
where $h_t(\vec{v}^m,\vec{a})$ satisfies:
\begin{enumerate}
    \item $h_t(j,\vec{v}^m,\vec{a})$  is a non-zero polynomial in $\vec{a}$ if $\bv \in U_{t+1}^\star$.
    \item $h(\vec{v}^m,\vec{a})=0$ otherwise.
    \review{
    \item For any  $\bv \in U_{t+1}^\star$:
    \begin{equation}
        \operatorname{min deg} (h_t(j,\vec{v}^m,\vec{a})) >  \operatorname{min deg} (h_{t-1}(j,\vec{v}^m,\vec{a})), 
    \end{equation}
where $\operatorname{min deg}$ denotes the minimum degree as a polynomial in $\vec{a}$.
    \item There exists a basis $\bv^{(1)}, \cdots,\bv^{(r_t-t_{t-1})}$ of $U_{t+1}^\star \cap (U_t^\star)^{\bot}$ such that the minimum degree in the sequence of polynomials $[h_{i,t,\bv^{(1)}_m,\ba}, \cdots, h_{i,t,\bv^{(r_t-r_{t-1})}_m,\ba}]$ is strictly increasing with the minimum degree term in each of the components depending only on $a_i$.
    }
\end{enumerate}

\end{lemma}

Consider the gradient w.r.t the $j_{th}$ neuron's parameters:
\begin{align} 
 \vec{g}^{t}_j  = -\nabla_{\vec{w}_j} \cL\left(\hat f(\vec{z}^\nu; \bW^t, \vec{a}), f^\star(\vec{z}^\nu)\right)
&=  
    \frac{1}{n}a_j\sum_{\nu=1}^n \bz^\nu (f^\star(\vec{z}^\nu)-\hat f(\vec{z}^\nu; \bW^t, \vec{a}))\sigma'(\langle\bz^\nu, \vec{w}^t_j\rangle)\label{eq:decomposition_gradient_time_t}
\end{align}
\end{proof}

Suppose that $\vec{v} \in U_{t+1}^\star \cap (U_{t}^\star)^\perp$ i.e when $\vec{v}$ is a new direction not yet learned upto time $t$.Using, equation \eqref{eq:decomposition_gradient_time_t}, the expectation $\Ea{\langle \vec{v}, \vec{g}^{t}_j \rangle}$ can be expressed as:
\begin{equation}\label{eq:grad_decom_multi}
    \begin{split}
    \Ea{\langle \vec{v}, \vec{g}^{t}_j \rangle}=\Ea{a_j (f^\star(\vec{z})-\hat f(\vec{z}; \bW^t, \vec{a}))\sigma'(\langle \bz,\vec{w}^t_j\rangle)\langle \bz,\vec{v} \rangle}.
    \end{split}
\end{equation}

Consider the term $\Ea{\hat f(\vec{z}; \bW^t, \vec{a}))\sigma'(\langle \bz,\vec{w}^t_j\rangle)\langle \bz, \vec{v} \rangle}$. Through a change of variables, and Wick's theorem \citep{janson1997gaussian,polyak2005feynman}, one obtains that $\Ea{\hat f(\vec{z}; \bW^t, \vec{a})\sigma'(\langle\bz,\vec{w}^1\rangle)_j\langle\bz,\vec{v}\rangle}$ is a polynomial in $Q$ and the overlaps $\langle \vec{w}^i,\vec{v}\rangle$ for $i \in [p]$ having value $0$ when $\langle \vec{w}^i,\vec{v}\rangle=0$ for all $i \in [p]$. By the induction hypothesis, $\langle \vec{w}^i,\vec{v}\rangle=\cO(\frac{\polylog d}{\sqrt{d}})$ with high probability. Therefore $\Ea{\hat f(\vec{z}; \bW^t, \vec{a}))\sigma'(\langle \bz,\vec{w}^t_j\rangle)\langle \bz, \vec{v} \rangle \vert \mathcal{F}_t}=\cO(\frac{\polylog d}{\sqrt{d}})$ with high probability. Similarly, $\Ea{\hat f(\vec{z}; \bW^t, \vec{a}))\sigma'(\langle \bz,\vec{w}^t_j\rangle)\langle \bz, \vec{v} \rangle\vert \mathcal{F}_t}=\cO(\frac{\polylog d}{\sqrt{d}})$ holds when
$\bv \notin U_{t+1}^\star$.

 Now, consider the term $\Ea{a_j f^\star(\vec{z})\sigma'(\langle\bz^,\vec{w}^t_j\rangle)\langle \bz,\vec{v} \rangle}$. First, using Fubini's theorem, we take the expectation w.r.t the component $\bz^\bot$ of $z$ in ${V^\star}^\bot$. 
 
 Recall that  $\bv_{j,\ba}=\bW_t^*\bw^t_j$ and $\kappa=\bW_t^*\bz$. 
  The resulting expectation converges in probability to a function of $\kappa$:
\begin{equation}
    \Eb{\bz^\bot}{a_j f^\star(\vec{z})\sigma'(\langle\bz,\vec{w}^t_j\rangle)\langle \bz,\vec{v} \rangle}
   =  \Eb{\kappa}{a_jf^\star(\vec{z})f_1(\ba,\kappa)\langle \bz,\vec{v} \rangle},
\end{equation}
where $f_1(\ba,\kappa)$ is defined as follows:
\begin{align*}
f_1(\ba,\kappa)&=\Eb{\bz^\bot}{\sigma'(\bz^\top\vec{w}^t_j)}\\
    &=\Eb{u \sim \mathcal{N}(0,1)}{\sigma'(c_{j,\vec{a}}u+\langle \kappa, \vec{v}_{j,\vec{a}} \rangle)}
\end{align*},
where $c_{j,\vec{a}}$ denotes the norm of $\vec{w}^t_j$ along the orthogonal complement of $V^*$.
$f_1(\ba,\kappa)$ generalizes the ``shifted-hermite" $\mu_{1,\kappa,j}$ defined in the section \ref{sec:proof_sketch_thm3}.
By assumption on $\sigma$, $\sigma'(c_{j,\vec{a}}u+\langle \kappa, \vec{v}_{j,\vec{a}} \rangle)$ is a polynomial in $c_{j,\vec{a}}u,\langle \kappa, \vec{v}_{j,\vec{a}} \rangle$. Furthermore, only the odd terms in $c_{j,\vec{a}}u$ are zero in expectation $u \sim \mathcal{N}(0,1)$. Therefore, $f_1(\ba,\kappa)$ is a polynomial in $\langle \kappa, \vec{v}_{j,\vec{a}} \rangle$ and $c_{j,\vec{a}}$ with only even degree terms in $c_{j,\vec{a}}$. By the induction hypothesis, $c^2_{j,\vec{a}}$ converges in probability to a  polynomial in $\vec{a}$. Therefore $f_1(\ba,\kappa)$ is a polynomial in $\vec{a},\kappa$.

Subsequently, we consider the expectation w.r.t $\langle \bz,\vec{v} \rangle$, at a fixed value of $\kappa$. Define the following function of $\kappa$:
\begin{equation}
    f_2(\kappa) = \Eb{\langle \bz,\vec{v} \rangle}{y\langle \bz,\vec{v} \rangle \vert \kappa}.
\end{equation}

Using the tower law of expectation, we obtain:
\begin{equation}\label{eq:overlap_def}
    \Ea{a_j f^\star(\bz)\sigma'(\bz^\top\vec{w}^t_j)\langle \bz,\vec{v} \rangle}
   =  \Eb{\kappa}{\review{a_j}f_1(a_j,\kappa)f_2(\kappa)},
\end{equation}

% Now, consider the term $\Ea{(y)\sigma'(\bz^\top\vec{w}_j^1)_j(\bz^\top\vec{u})}$.
When 
$\bv \notin U_{t+1}^\star$, $f_2(\kappa)$ is identically $0$ and the above expectation vanishes.

We aim to show that the above expectation does not vanish except for $a_i$ belonging to a zero-measure set. By the definition of subspace conditioning (definition \ref{def:subspace_conditioning}), $\exists \kappa > 0$ such that $\Eb{\langle \bz,\vec{v} \rangle}{f^\star(\bz)\bz \vert \kappa}$ has non-zero overlap with $\vec{v}$. 

Therefore, $f_2(\kappa)$ is not identically zero.
Furthermore, since $f^\star$ is a polynomial by assumption, and $\vec{v} \perp V^\star$, a rotation of basis implies that $f_2$ is a polynomial in $\kappa$.  Let $\mathcal{S}_{y,t}$ be the set of degrees $s \in \N_0$ such that $\Eb{f_2(\kappa)\kappa^s}{\kappa} \neq 0$. Since 
$f_2$ is not identically $0$, we have that 
$\mathcal{S}_{y,t} \neq \phi$.

Now, recall that:
\begin{align*}
    f_1(\ba,\kappa) &= \Eb{u \sim \mathcal{N}(0,1)}{\sigma'(c_{j,\vec{a}}u+\langle \kappa, \vec{v}_{j,\vec{a}} \rangle)}\\
    &=  \Eb{u \sim \mathcal{N}(0,1)}{\sum_{k=0}^{\operatorname{deg}(\sigma)-1}(k+1)b_{k+1} (c_{j,\vec{a}}u+\langle \kappa, \vec{v}_{j,\vec{a}} \rangle)^r}\\
    &=  \sum_{k=0}^{\operatorname{deg} (\sigma)-1}(k+1)b_{k+1}\Eb{u \sim \mathcal{N}(0,1)}{(c_{j,\vec{a}}u+\langle \kappa, \vec{v}_{j,\vec{a}} \rangle)^k}.
\end{align*}
Now, let $s \in \mathcal{S}_{y,t}$ be arbitrary. By assumption,
$\operatorname{deg}(\sigma)-1 \geq s$. 
Let $p_s(\ba)$ denote the coefficient of  $\kappa^s$ in $f_1(\ba,\kappa)$.
Since $c_{j,\vec{a}},\vec{v}_{j,\vec{a}}$ are non-constant polynomials in $\vec{a}$, the coefficient 
 of $\kappa^s$ in $(c_{j,\vec{a}}u+\langle \kappa, \vec{v}_{j,\vec{a}} \rangle)^k$  is a non-constant polynomial in  $\ba$ for any $k$ such that $k-q$ is even. Furthermore, the degree of the coefficient of $\kappa^s$ in $(c_{j,\vec{a}}u+\langle \kappa, \vec{v}_{j,\vec{a}} \rangle)^k$ is strictly increasing in $k$. Therefore, for any $s \in \mathcal{S}_{y,t}$, $p_s(\ba)$ is a non-constant polynomial in $\ba$. Now, consider the term in $p_s(\ba)$ with the least degree in $a_j$. From the definition of $\vec{v}_{j,\vec{a}}$, we have that  $\vec{v}_{j,\vec{a}}=0$ whenever $a_j=0$. Let 
$d_j$ denote the least $s \in \N_0$ such that the coefficient of $a^s_j$ in $\langle \vec{v}_{j,\vec{a}}, \kappa \rangle$ is non-zero. We have that $d_j>0$. Consequently, the minimum degree of $a_j$ in $(c_{j,\vec{a}})^q(\langle \kappa,\vec{v}_{j,\vec{a}} \rangle)^s, $ is $(d_j)^s$ for any $q$. Therefore, the minimum degree of $p_s(\ba)$ is strictly increasing in $s$. This implies that $p_s(\ba)$ are linearly independent for $s=1,\cdots,\operatorname{deg}(\sigma)-1$.

Now, consider the function defined above in Equation \ref{eq:overlap_def}:
\begin{equation}
    h(t,\vec{a})= \Eb{\kappa}{\review{a_j}f_1(\vec{a},\kappa)f_2(\kappa)}.
\end{equation}
By expanding $f_1,f_2$ along 
$\kappa$, the coefficient of $\kappa^s$ for each $s \in \mathcal{S}_{y,t}$ results in a non-constant polynomial in $\vec{a}$. We obtain:
\begin{equation}
     h_t(j,\vec{v}^m,\vec{a})=\sum_{s\in  \mathcal{S}_{y,t}} c_s p_s(\ba),
\end{equation}
where $c_s$ denote non-zero constants independent of $d,n$. Therefore, we have that $h_t(j,\vec{v}^m,\vec{a})$ is a non-constant polynomial in $\vec{a}$. Using Fubini's theorem, we have that the set of zeros of non-zero multivariate polynomials has $0$ measure w.r.t the Lebesque measure (for a generalization, see \citep{mityagin2020zero}), we obtain that $ h_t(j,\vec{v}^m,\vec{a}) \neq 0$ almost surely. 

\review{We now show $(iii)$ and $(iv)$. From the above argument and the inductive hypothesis $(iv)$, we further obtain that the minimum degree term in $p_s(\vec{a})$ depends only on $a_j$ implying that the minimum degree term in $h_t(j,\vec{v}^m,\vec{a})$ depends only on $a_j$. Since $h_{t-1}(j,\vec{v}^m,\vec{a})=0$, we directly obtain $(iii)$.
\\ Now, let $\bu^{(1)}, \cdots,\bu^{(r_t-t_{t-1})}$ be an arbitrary basis of $U^\star_{t+1} \cap U^\star_{t}$. Let $\tilde{\bv}^{(1)}_m$ be the  vector along the coefficients of the minimum degree terms in $[h_{i,t,\bu^{(1)}_m,\ba}, \cdots, h_{i,t,\bu^{(r_t-r_{t-1})}_m,\ba}]$. Similarly, we obtain vectors $\tilde{\bv}^{(2)}_m, \cdots$ corresponding to coefficients of strictly increasing degrees. Applying Gram-Schmidt orthogonalization to $\tilde{\bv}^{(1)}_m,\tilde{\bv}^{(2)}_m, \cdots$, we obtain the desired basis in $(iv)$.}

Now, suppose that $\vec{v} \in  (U_{t}^\star)$, i.e when $\vec{v}$ is an already learned direction. By the induction hypothesis, $\langle \bw^t_i, \vec{v}\rangle$ converges to a non-constant polynomial in $\vec{a}$. Consider the term $ \frac{a_j}{\sqrt{p}}\Ea{\hat f(\vec{z}; \bW^t, \vec{a}))\sigma'(\langle \bz,\vec{w}^t_j\rangle)\langle \bz, \vec{v} \rangle}$ in $\langle \bg^t_i, \vec{v}\rangle$. By expanding  $\hat f(\vec{z}; \bW^t, \vec{a})$ we obtain:
\begin{equation}
    \frac{a_j}{\sqrt{p}}\Ea{\hat f(\vec{z}; \bW^t, \vec{a}))\sigma'(\langle \bz,\vec{w}^t_j\rangle)\langle \bz, \vec{v} \rangle} = \frac{a_j}{p}\sum_{i = 1}^{p} a_i \Ea{\sigma(\langle \vec{w}^t_i, \vec{z} \rangle)\sigma'(\langle \bz,\vec{w}^t_j\rangle)\langle \bz, \vec{v} \rangle}
\end{equation}
The term correspondign to the $j_{th}$ neuron has the form:
\begin{equation}
    \frac{a^2_j}{p}\Ea{\sigma(\langle \vec{w}^t_j, \vec{z} \rangle)\sigma'(\langle \bz,\vec{w}^t_j\rangle)\langle \bz, \vec{v} \rangle}
\end{equation}
By Wick's theorem \citep{janson1997gaussian,polyak2005feynman} and the polynomial assumption on $\sigma$, $\Ea{\sigma(\langle \vec{w}^t_j, \vec{z} \rangle)\sigma'(\langle \bz,\vec{w}^t_j\rangle)\langle \bz, \vec{v} \rangle}$ is a non-zero polynomial in $\langle \bw^t_j, \vec{v}\rangle$. Let $d_j$ be the degree of $a_j$ in $\langle \bw^t_j, \vec{v}\rangle$. Then, the degree of $a_j$ in $\frac{a^2_j}{p}\Ea{\sigma(\langle \vec{w}^t_j, \vec{z} \rangle)\sigma'(\langle \bz,\vec{w}^t_j\rangle)\langle \bz, \vec{v} \rangle}$
is at-least $d_j+2$. Proceeding similarly for the other terms, one can show that the degree of $a_j$ in $\langle \bg^t_i, \vec{v}\rangle$ is strictly larger than in $\langle \bw^t_i, \vec{u}\rangle$. This ensures that $\langle \bw^{t+1}_i, \vec{v}\rangle=\langle \bg^t_i, \vec{v}\rangle+\eta \langle \bg^t_i, \vec{v}\rangle$ remains a non-constant polynomial upto error $\cO(\frac{\polylog d}{\sqrt{d}})$. Therefore, almost surely over $\vec{a}$, a direction is not ``un-learned". \review{The strict increase in degree further implies $(iii)$}

Finally, by decomposing along a general $\vec{v} \in U_{t+1}^\star$, along $U_{t}^\star$ and $U_{t+1}^\star \cap (U_{t}^\star)^\perp$, one obtains that points $(ii)$ and $(iii)$ of the induction statements hold in expectation.

Next, we prove that the events $(i),(ii),(iii)$ hold with high probability. By the induction hypothesis, we have that and the above analysis, we have that:

\begin{lemma}\label{lem:g_conc_t}
Suppose that the induction hypothesis holds at time $t$. 
Then, the following events occur with high-probability for all $i,j \in [p]$
\begin{enumerate}
    \item $\abs{\norm{\vec{g}^{t+1}_i}^2-\Ea{\norm{\vec{g}^{t+1}_i}^2}} 
    = \cO\left(\frac{\polylog{d}}{\sqrt{d}}\right)$ 
    \item $\norm{\langle \vec{g}_i, \vec{g}_j\rangle-\Ea{\langle \vec{g}_i, \vec{g}_j\rangle}} = \cO\left(\frac{\polylog{d}}{\sqrt{d}}\right)$
    \item For any $k \in [r]$, and any unit vector $\vec{w}$
     \begin{equation}
         \left|\langle \vec{w}, \vec{g}_i \rangle - \dE[\langle \vec{w}, \vec{g}_i \rangle]\right| =\cO\left(\frac{\polylog{d}}{\sqrt{d}}\right)
     \end{equation}
\end{enumerate}

\end{lemma}
\begin{proof}

We condition on the event in $\mathcal{F}_t$ that $Q_t, M_t$ are bounded by some constants independent of $d,n$. Subsequently, the proof proceeds similar to  Proposition \ref{prop:app:update_concentration}, with the additional incorporation of the term due to $\hat f(\vec{z}^\nu; \bW^t, \vec{a}))$ in the gradient.

We have:
\begin{equation}
    \vec{g}^t_j =  \frac{1}{n}a_j\sum_{\nu=1}^n \bz_i (f^\star(\vec{z}^\nu)-\hat f(\vec{z}^\nu; \bW^t, \vec{a}))\sigma'(\bz_i^\top\vec{w}^t_j)
\end{equation}
Define:
\begin{equation}
    \vec{X}_i^\nu = \vec{z}^\nu \sigma'(\langle \vec{w}^t_i, \vec{z}^\nu \rangle)(f^\star(\vec{z}^\nu)-\hat f(\vec{z}^\nu; \bW^t, \vec{a})).
 \end{equation}

Analogous to the proof of Proposition \ref{prop:app:update_concentration}, we have:
 \begin{small}
\begin{equation}
     \norm{\vec{g}_i}^2 - \dE[\norm{\vec{g}_i}^2] = \frac{a^2_i}{n^2p^2} \left( \underbrace{\sum_{\nu = 1}^n \norm{\vec{X^\nu}_i}^2 - n\dE[\norm{\vec{X^\nu}}^2]}_{S_1} + \underbrace{\sum_{\nu \neq \nu'} \langle \vec{X}_i^\nu, \vec{X}_i^{\nu'} \rangle - n(n-1) \norm{\dE \langle \vec{X}_i^\nu, \vec{X}_i^{\nu'} \rangle}^2}_{S_2} \right)
 \end{equation}
\end{small}

Similarly, we have:
 \begin{small}
\begin{equation}
    \langle \vec{g}_i, \vec{g}_j \rangle - \dE[\vec{g}_i, \vec{g}_j]= \frac{a^2_i}{n^2p^2} \left( \underbrace{\sum_{\nu = 1}^n \langle \vec{X^\nu}_i, \vec{X^\nu}_j \rangle - n\dE[\langle \vec{X^\nu}_i, \vec{X^\nu}_j \rangle]}_{S'_1} + \underbrace{\sum_{\nu \neq \nu'} \langle \vec{X}_i^\nu, \vec{X}_j^{\nu'} \rangle - n(n-1) (\dE \langle \vec{X}_i^\nu, \vec{X}_j^{\nu'} \rangle)^2}_{S'_2} \right)
\end{equation}
\end{small}
Note that $f^\star(\vec{z}^\nu)$ and $\hat f(\vec{z}^\nu; \bW^t, \vec{a})$ are polynomials in finite-number of correlated Gaussians $\langle \vec{w}^t_1, \vec{z}^\nu \rangle,\cdots, \langle \vec{w}^t_p, \vec{z}^\nu \rangle$, $\langle \vec{w}^\star_1, \vec{z}^\nu \rangle,\cdots, \langle \vec{w}^\star_r, \vec{z}^\nu \rangle$.
Therefore, by repeated applications of Lemma \ref{lem:app:orlicz_submult} and Theorem \ref{thm:app:orlicz_sum}, we obtain that
$\sigma'(\langle \vec{w}^t_i, \vec{z} \rangle), f^\star(\vec{z}^\nu)$ and $\hat f(\vec{z}^\nu; \bW^t, \vec{a}))$ have bounded Orlicz norms of some finite order $\alpha_t$.

Subsequently, similar to Lemma \ref{lem:app:s1_bound}, through Holder's inequality, Lemma \ref{lem:app:orlicz_submult} and Theorem \ref{thm:app:orlicz_sum}, we obtain that $\norm{\vec{X}_i^\nu}^2$, $\langle \vec{X}_i^\nu, \vec{X}_i^{\nu'}\rangle,\langle \vec{X}_i^\nu, \vec{X}_j^{\nu}\rangle,\langle \vec{X}_i^\nu, \vec{X}_j^{\nu'}\rangle$, 
have Orlicz norms of order $\cO(d)$ with $\alpha=\alpha_t$ for some $\alpha_t$ independent of $d$. 

The remaining proof follows by repeating the arguments in Lemmas \ref{lem:app:s1_bound}, \ref{lem:app:S_11_norm_concentration} for Orlicz norms of general order.

Similarly (iii) is obtained by replacing the application of Bernstein's inequality in Lemma \ref{lem:app:linear_concentration} by Theorem \ref{thm:app:orlicz_sum}.
\end{proof}

Lemmas \ref{lem:g_exp_t} and \ref{lem:g_conc_t} together with \ref{eq:qmupd} and the induction hypothesis imply statement $(iii)$ at time $t+1$.

It remains to prove the base case i.e $t=1$. If the leap $\ell > 1$, $U^\star_t={0}$ for all $t \geq 1$. Applying the above arguments then implies that $(i)$ and $(iii)$ hold for all timesteps $t$.

Therefore, we may  assume that $\ell =1$.
At $t=1$, $U^\star_1$ is simply the subspace along $(C_1(f^\star))$. Let   $\bv =  \pm \frac{1}{\norm{C_1(f^\star)}}C_1(f^\star)$ be a vector as per $(ii)$ let $i \in [p]$ be an arbitrary neuron. We have:
\begin{align*}
    \langle \bv, \bw^1_i \rangle &= \langle \bv, \bw^0_i \rangle + \eta \langle \bv, \bg^i \rangle\\
    &= \eta \langle \bv, \bg^i \rangle + \cO(\frac{1}{\sqrt{d}})
\end{align*}
It is straightforward to check that Lemma \ref{lem:app:grad_exp} holds when $\sigma, g^*$ are polynomials, while Lemma \ref{lem:norm_first_step} holds in expectation. Applying the concentration results for Orlicz norms of general order as in Lemma \ref{lem:g_conc_t} imply that Lemma \ref{lem:norm_first_step} also holds in probability for polynomial $\sigma, g^*$.
To establish $(i)$ at time $t=1$, we note that Lemma \ref{lem:norm_first_step} implies that $\Ea{\bw_i}^2$ converges to $1+c a^2_i$ where $c$ is independent of $d,n$.
By Lemma \ref{lem:app:grad_exp} the first term equals with high-probability, $\pm \frac{a_i\mu_1}{p} \norm{C_1(f^\star)} + \cO{(\frac{1}{\sqrt{d}})}$. Since, $\frac{a_i\mu_1}{p} \norm{C_1(f^\star)}$ is a non-constant (linear) polynomial in $a_i$, this proves $(ii)$ for $t=1$.

\review{Lemmas \ref{lem:app:grad_exp} and part $(iv)$ in Lemma \ref{lem:g_conc_t} directly implies $(iv)$ of the induction  statement and  $(ii)$ of Theorem \ref{thm:staircase}.
We now explain how $(ii)$ and $(iii)$ in the induction statements 
imply $(i)$ in Theorem \ref{thm:staircase}. By $(iii)$, the overlaps can be decomposed as:
 \begin{equation}
    \vec{q}^t(i,\ba) \coloneqq [q_{i,t,\bv^{(1)}_m,\ba}, \cdots, q_{i,t,\bv^{(r_t)}_m,\ba}] = \vec{p}^t(i,a_i)+ \delta(i,\ba),
 \end{equation}}
\review{ 
where $\vec{h}^t(i,a_i)=[p_{i,t,\bv^{(1)}_m,a_i}, \cdots, h_{i,t,\bv^{(r_t)}_m,a_i}]$ contains polynomials of strictly increasing minimum degree, depending only on $a_i$ and $\delta(\ba) \in \mathbb{R}^{r_t}$ satisfies:
\begin{equation}\label{eq:deg_ineq}
    \operatorname{min-deg}(\delta(i,\ba)_j) > \operatorname{min-deg}(\vec{h}^t_j(i,a_i)) \ \forall j \in [r_t],
\end{equation}
}
\review{
By the strict increase in degree, we obtain that for any $\vec{v} \in \mathbf{R}^{r_t}$  with $\norm{\vec v}=1$, $\langle \vec{h}^t(i,a_i),\vec{v} \rangle = 0$ for at most $r_t$ values of $a_i$. Therefore, the set of vectors $\vec{h}^t(i,a_i)$ span $\mathbf{R}^{r_t}$. The independence of $a_i$ then implies that for $p > r_t$, the matrix:
\begin{equation}
   \vec{H}^t= \begin{bmatrix}
       \vec{h}^t(i,a_1) \\
       \vec{h}^t(i,a_2) \\
       \vdots \\
       \vec{h}^t(i,a_p)
    \end{bmatrix},
\end{equation}
is almost surely full-rank. Now, the condition \ref{eq:deg_ineq} implies that there exists $\epsilon > 0$ such that for $\abs{a_i} < \epsilon$:
\begin{equation}
    \inf_{\vec{v} \in \mathbf{R}^{r_t},\norm{v}=1} (\abs{\langle \vec{h}^t(i,a_i),\vec{v} \rangle} - \abs{\delta(i,\ba),\vec{v} \rangle}) > 0.
\end{equation}
Therefore, we obtain that conditioned in $\abs{a_i} < \epsilon \forall i \in [p]$,
the overlap matrix
\begin{equation}
   \vec{Q}^t= \begin{bmatrix}
       \vec{q}^t(1,\ba) \\
       \vec{q}^t(2,\ba) \\
       \vdots \\
       \vec{q}^t(p,\ba)
    \end{bmatrix},
\end{equation}
is full-rank almost surely. This implies that the determinant of $(\vec{Q}^t)^\top\vec{Q}^t$ is a non-zero polynomial in $\vec{a}$ and thus $\vec{Q}^t$ is almost surely full-rank. The continuity of the determinant then implies $(i)$ in Theorem \ref{thm:staircase}.}

\subsection{Prediction of the alignment at the second step}

We now utilize the analysis in the previous section to obtain a theoretical prediction for the gradient orientation after two steps for the target function 
defined in bottom right of Figure \ref{fig:multiple_steps} i.e.:
\begin{equation}
    f^{\star}(\vec z) = \sigma^{\star}_1(\langle\vec{w_1}^{\star},\vec{z}\rangle) + \sigma^{\star}_2(\langle\vec{w_2}^{\star},\vec{z}\rangle),
\end{equation}
with $\sigma^{\star}_1(z) = z-z^2$ and $\sigma^{\star}_2(z) = z+ z^2$. Equivalently, the above target function can be expressed in a rotated basis as:
\begin{equation}
     f^{\star}(\vec z)= \sqrt{2}\vec{u_1}^{\star}+2\vec{u_1}^{\star}\vec{u_2}^{\star}.
\end{equation}
Where $\vec{u_1}^{\star}=\frac{1}{\sqrt{2}}(\vec{w_1}^{\star}+\vec{w_2}^{\star})$
and $\vec{u_2}^{\star}=\frac{1}{\sqrt{2}}(\vec{w_1}^{\star}-\vec{w_2}^{\star})$. Therefore $\vec{v}^*=\sqrt{2}\vec{u_1}^{\star}$ with $\norm{\vec{v}^*}=\sqrt{2}$.

We follow the notation defined in the proof sketch in Section \ref{sec:proof_sketch_thm3} and assume that $a_i=\pm \frac{1}{\sqrt{p}}$, while $\alpha= \frac{n}{d}=4$ as in Figure \ref{fig:multiple_steps}.

We have, using Equation \eqref{eq:grad_decom}:

\begin{equation}
    \Ea{\langle \vec{u}, \vec{g}^{1}_j \rangle} \rightarrow \Ea{y\mu_{1,c'_j\langle \bz , \vec{v}^\star \rangle}( \langle \bz_i, \vec{u} \rangle)}-\Ea{\hat{y}^{1}_i\sigma'(\langle \bz_i, \vec{w}^1 \rangle)_j (\langle \bz_i, \vec{u} \rangle)}.
\end{equation}
As explained in the previous section, the second term does not contribute to an alignment towards $\vec{v}^\bot$. Denoting by $\vec{v}_u^\star$, the normalized vector along $\vec{v}^*$,we consider the ratio of the first term when $\vec{u}=\vec{v}_u^\star$ or $\vec{u}=\vec{v}^\bot$.
We obtain:
\begin{equation}
   \langle \vec{g}^1, \vec{v}_u^\star \rangle \approx \Ea{y\mu_{1,c_j'\langle \bz , \vec{v}^\star \rangle,j }\langle \bz_i, \vec{v}_u^\star\rangle},
\end{equation}
Since the first Hermite coefficient $\mu_1$ of the student activation for Relu equals $0.5$, we obtain that $c'j=a_j\eta$.

We assume that $c_j \approx 1$. Therefore, $\mu_{1, \langle \bz , \vec{v}_u^\star \rangle}$ corresponds to the first Hermite coefficient of a translated Relu function and is given by:
\begin{equation}
    \mu_{1,\kappa,j}= (1-\Phi(-\kappa))= \frac{1}{2}(1+\text{erf}(-\kappa/\sqrt{2})).
\end{equation}
Therefore, when $\eta=2$, we obtain that $c'j=2$:
\begin{equation}
    \mu_{1,c'_j\langle \bz , \vec{v}^\star \rangle,j} = \frac{1}{2}(1\pm\text{erf}(\langle \bz,\vec{v}^\star \rangle/\sqrt{2}))=\frac{1}{2}(1\pm\text{erf}(\langle \bz,\vec{v}_u^\star \rangle))
\end{equation}

Let $\Vec{v}_1^{(t=2)}$ and $\Vec{v}_2^{(t=2)}$ denote the projections of the neurons $\vec{w}_j$ with $a_j=1$ and $a_j=-1$ respectively.
Therefore, using the Hermite decomposition of $\text{erf}$, we obtain the following predicted orientations in the setting considered in the right panel of Fig.~\ref{fig:multiple_steps}:
\begin{equation}
\Vec{v}_1^{(t=2)} = (1 - \frac{2}{\sqrt{3 \pi}}) \Vec{w}^\star_1 + (1 + \frac{2}{\sqrt{3 \pi}}) \Vec{w}^\star_2 \qquad 
\Vec{v}_2^{(t=2)}  = (1 + \frac{2}{\sqrt{3 \pi}}) \Vec{w}^\star_1 + (1 - \frac{2}{\sqrt{3 \pi}}) \Vec{w}^\star_2  
\end{equation}
% Useful fact to use:
% \begin{align}
%     \text{1st hermite erf}& \qquad \mu_1 = \frac{2}{\sqrt{3 \pi}} \\
%     \text{1st hermite relu}& \qquad \mu_1 = 0.5 \\
%     \text{relation erf \& cdf of Gaussian}& \qquad  \Phi(x) = \frac{1}{2}[1+ erf(\frac{x}{\sqrt{2}})]\\
% \end{align}
\subsection{Limitations of the Staircase Structure}

We show that a natural class of teacher functions, containing neurons with identical activation functions and uniform second-layer weights does not contain a staircase structure:
\begin{proposition}
Let $y = f^\star(\vec{z}) = \sum_{k=1}^r \sigma^\star(\langle \vec{w}_k^\star, \vec{z} \rangle)$ for some $\sigma^\star$ having leap index $1$, then $U^\star_i=U^\star_1$ for all $i \geq 1$.
\end{proposition}

\begin{proof} 
    For any such target function,  $\vec{v}_u^\star$ is given by $\vec{v}_u^\star=\frac{1}{\sqrt{r}}(\sum_{k=1}^r\vec{w}_k^\star)$. Without loss of generality, assume that $\vec{w}_k^\star = \vec{e}_k$, where $ \vec{e}_k$ denotes the unit vector corresponding to the $k_{th}$ coordinate.
    Now, consider any direction  $\vec{u} \perp \vec{v}^\star$ in the teacher subspace. Such a vector satisfies $\sum_{k=1}^r u_i = 0$.
    Therefore, for any $k \geq 0$, we have:
    \begin{equation}
    \begin{split}
        \Ea{f^\star(\vec{z})H_k(\langle \vec{v}^\star , \vec{z} \rangle )(\langle \vec{u}_i, \vec{z} \rangle )} &= (\sum_{i=1}^p (u_i))(\Ea{f^\star(\vec{z})H_k(\langle \vec{v}_u^\star , \vec{z} \rangle)z_1})\\
        &=0.
    \end{split}
    \end{equation}
    Where we used the symmetry of $f^\star(\vec{z})$ w.r.t permutations of the first $r$ coordinates.
    Therefore, the Hermite decomposition of $f^\star(\vec{z})$ does not contain any term that linearly couples $\vec{u}$ to  $\vec{v}^\star$.
\end{proof}

Therefore, the presence of a staircase structure requires asymmetry between the the dependence of the target function on different directions in the teacher subspace.
\clearpage

\section{Learning the second layer}
\label{sec:appendix:cget_proofs}

\subsection{Proof of Proposition \ref{prop:fixed_width_lower_bound}}
We first prove the finite $p$ case of Proposition \ref{prop:fixed_width_lower_bound}. Let $\vec{a}$ be a second layer vector with $a_i \leq c/\sqrt{p}$, and assume that $W$ only learns a subspace $U \subseteq V^\star$. We write \review{$\dR^d = U \oplus U^\bot_\star \oplus V^{\star\bot}$, where $U^\bot_\star$ is the orthogonal subspace of $U$ in $V^\star$. By assumption, we have $\norm{P_{U^\bot_\star}\vec{w}_i} \leq \eps_d$ for every $i$; where $\eps_d$ is going to zero as $d$ grows.}

For any $\vec{z}\in \dR^d$, we have
\begin{align*}
    \hat f(\vec{z}; W, \vec{a}) &= \sum_{i=1}^p\frac{a_i}{\sqrt{p}} \sigma(\langle \vec{w}_i, P_{U} \vec{z} \rangle + \langle \vec{w}_i, P_{\review{U^\bot_\star}} \vec{z} \rangle  + \langle \vec{w}_i, P_{V^{\star\bot}} \vec{z} \rangle ) \\
    &= \sum_{i=1}^p\frac{a_i}{\sqrt{p}} \sigma(\langle \vec{w}_i, P_{U} \vec{z} \rangle + \langle \vec{w}_i, P_{V^{\star\bot}} \vec{z} \rangle ) + \frac{a_i}{\sqrt{p}}\eps_d \tilde \sigma(\langle \vec{w}_i, \review{P_{U^\bot_\star}} \vec{z} \rangle)
\end{align*}
where $\tilde \sigma$ is a Lipschitz function. We call the first term of the above expression $\tilde f(P_{U} \vec{z}, P_{V^{\star\bot}\vec{z}})$, forgetting the structure of the function $\hat f$. Then, we can write the risk as
\begin{equation}
    \cR(W, \vec{a}) = \dE_{\vec{z}}\left[ \left(f^\star(\vec{z}) - \tilde f(P_{U} \vec{z}, P_{V^{\star\bot}\vec{z}})\right)^2 \right] + O(\eps_d),
\end{equation}
having used the Cauchy-Schwarz inequality to bound the contribution of $\tilde \sigma$. Then, by successive expectations, 
\begin{align*}
    \cR(W, \vec{a}) &= \dE_{P_{V^{\star\bot}}\vec{z}, P_{U} \vec{z}}\left[\dE_{P_{U^\bot}\vec{z}}\left[ \left(f^\star(\vec{z}) - \tilde f(P_{U} \vec{z}, P_{V^{\star\bot}\vec{z}})\right)^2 \bigg \vert P_{V^{\star\bot}}\vec{z}, P_{U} \vec{z} \right]\right] + O(\eps_d), \\
    &\geq \dE_{P_{V^{\star\bot}}\vec{z}, P_{U} \vec{z}}\left[ \inf_f \dE_{P_{U^\bot}\vec{z}}\left[ \left(f^\star(\vec{z}) - f(P_{U} \vec{z}, P_{V^{\star\bot}\vec{z}})\right)^2 \bigg \vert P_{V^{\star\bot}}\vec{z}, P_{U} \vec{z} \right]\right] + O(\eps_d)
\end{align*}
where the infimum is taken over all measurable functions $f: U \times V^{\star \bot} \to \dR$. But this infimum exactly corresponds to the definition of conditional expectation/conditional variance, which is independent from $P_{V^{\star\bot}}\vec{z}$ (since $f^\star$ is). As a result,
\begin{equation}
    \cR(W, \vec{a}) \geq  \dE_{P_{U} \vec{z}}\left[ \Var\left(f^\star(\vec{z})\vert P_{U} \vec{z}\right) \right] + O(\eps_d),
\end{equation}
which implies the statement of Proposition \ref{prop:fixed_width_lower_bound}.

\subsection{Full statement of Theorem \ref{thm:cget}}
We now provide the full statement of Theorem \ref{thm:cget}. It establishes the asymptotic equivalence of the training and generalization errors of the original features and the conditional Gaussian features defined by equation \eqref{eq:ck_kernel}.

Consider the sequence of  vectors $\bv_n \in \R^d$ defined as in Equation \eqref{eq:spike+bulk} by $\vec{v}_n = \frac{1}{n}\sum_{i=1}^n  y_i\bz_i$. For simplicity, we omit the dependence of $\bv_n$ on $n$ and denote each entry by $\bv$. For any vector $\vec{z} \in \R^d$, define the decomposition $\vec{z} = z_{\vec{v}} \vec{v} + \vec{z}^\bot$ and feature maps:
\begin{equation}\phi_{\text{CK}}(\vec{z}) = \sigma(W^{(1)}\vec{z}),
\end{equation}
where $W^{(1)}$ denotes the weight matrix obtained through the application of a single gradient step.

Then the random variable $\phi_{\text{CK}}(\vec{z})$ admits a regular conditional distribution conditioned on the values of $z_{\vec{v}}$ (Theorem 8.37
in \cite{klenke2013probability}). Therefore, the following mean, correlation, and covariance matrix are well-defined:
\begin{equation}
\begin{split}
\mu(z_{\vec{v}}) &= \dE\left[ \phi_{\text{CK}}(\vec{z}) \mid z_{\vec{v}} \right], \quad\Psi(z_{\vec{v}}) = \dE\left[ \phi_{\text{CK}}(\vec{z})(z^\bot)^\top \mid z_{\vec{v}} \right],\\ 
    \Phi(z_{\vec{v}}) &= \Cov\left[ \phi_{\text{CK}}(\vec{z}) \mid z_{\vec{v}} \right] - \Psi(z_{\vec{v}})\Psi(z_{\vec{v}})^\top
\end{split}
\end{equation}

Now, for each value of $z_{\vec{v}}$, define the following random variable:
\begin{equation}\label{eq:cget_features}
    \phi_{\text{CL}}(\vec{z}; \vec{v}) = \mu\left(z_{\vec{v}}\right) + \Psi(z_{\vec{v}}) \vec{z}^\bot +  \Phi(z_{\vec{v}})\vec{\xi}.
\end{equation}
Then $\phi_{\text{CL}}(\vec{z}; \vec{v})$ satisfies:
\begin{equation}
    \dE\left[ \phi_{\text{CL}}(\vec{z}) \mid z_{\vec{v}} \right] = \mu(z_{\vec{v}}) , \dE\left[ \phi_{\text{CL}}(\vec{z})(\vec{z}^\bot)^\top \mid z_{\vec{v}} \right]  = \Psi(z_{\vec{v}}),\\
     \Cov\left[ \phi_{\text{CL}}(\vec{z}) \mid z_{\vec{v}} \right] =  \Cov\left[ \phi_{\text{CK}}(\vec{z}) \mid z_{\vec{v}} \right].
\end{equation}
Therefore, \review{$\phi_{\text{CL}}(\vec{z}; \vec{v})$} is a Gaussian variable having the same conditional mean, covariance as $\phi_{\text{CK}}(\vec{z}; \vec{v})$ and the same corrrelation with $\bz_{\bot}$ as $\phi_{\text{CK}}(\vec{z}; \vec{v})$. Since $\bz_{\bot}$ is Gaussian and independent of $z_{\vec{v}}$, this uniquely characterizes the conditional measure of  $\phi_{\text{CL}}(\vec{z}; \vec{v})$.

Now, consider a set of $n$ training inputs $\bz_1,\cdots,\bz_n$. For each $i \in n$, generate an equivalent feature map $\bm \Phi_{\text{CL}}$ through equation \eqref{alg:gd_training}, with $\vec{\xi}$  being independently sampled for each example. Let $\vec{\Phi}_{\text{CK}}$ and $\vec{\Phi}_{\text{CL}}$ denote matrices in $\R^{n \times p}$ with rows $\phi_{\text{CK}}(\vec{z_i})$ and $\phi_{\text{CL}}(\vec{z_i})$ respectively,

Consider the following minimization problem:
\begin{equation}\label{eq:ck_minima}
\begin{split}
     \min_{\vec{a} \in \dR^p} \frac1n \sum_{\nu = 1}^n \left( \langle \vec{a}, \phi_{\text{CK}}(\vec{z}^\nu)\rangle - f^\star(\vec{z}^\nu) \right)^2 + \lambda \norm{\vec{a}}^2\\
\end{split}
\end{equation}

Define the following constraint set:
 \begin{equation}\label{eq:1d_clt_constraints}
     \cS_p = \Set*{\ba \in \dR^d \given \norm{\ba}_2 \leq R, \quad \norm{\ba}_\infty \leq Cp^{-\eta}}.
 \end{equation} 
We make the following assumption:
\begin{assumption}\label{assump:constraint}
There exist constants $R,C, \eta$ such that the minimizer $\hat{\vec{a}}_{\text{CK}}$  of the optimization problem defined by equation \eqref{eq:ck_minima} lies in $\cS_p$ with high probability as $n,d \rightarrow \infty$.
\end{assumption}
The above assumption can be enforced by utilizing constrained minimization for the second layer. Alternatively, for overparameterized models i.e $p/n>1$, one could utilize the arguments in Theorem 5 of \cite{Montanari2022}. Let $\hR_n^\star(\bm\Phi, \bm y(\bm Z)),\mathcal{R}_g^\star(\bm\Phi, \bm y(\bm Z))$ denote the training and generalization errors respectively with features $\bm\Phi$ and labels $\bm y(\bm Z)$.
\setcounter{theorem}{3}
\begin{theorem}
 Assume that $n, p = \Theta(d)$, and that the vector $\vec v^\star=C_1(f^*)$ defined in Theorem \ref{thm:one_step_learning} is nonzero. Then, the sequence of vectors $\vec v_n = \frac{1}{n}\sum_{i=1}^n  y_i\bz_i \in \R^d$ satisfy:
    \begin{enumerate}
        \item As $n,d \rightarrow \infty$, $P_{V^\star}\vec{v} \overset{\dP}{\longrightarrow} \vec{v}^\star$.
        \item Under Assumption \ref{assump:constraint}, the training and generalization errors obtained through the minimization of the objective \eqref{eq:ck_kernel} for training distribution defined by feature maps $\phi_{\text{CK}}(\vec{z})$ converge in distribution to the corresponding training and generalization errors for features $\phi_{\text{CL}}(\vec{z}; \vec{v})$. 

Concretely, we have that for any bounded Lipschitz function $\Psi: \dR \to \dR$:
\[\lim_{n, p \to \infty} \left| \Ea{\Psi\left(\hR_n^\star(\bm\Phi_{CK}, \bm y(\bm Z))\right)} - \Ea{\Psi\left(\hR_n^\star(\bm\Phi_{\text{CL}}, \bm y(\bm Z))\right)} \right| = 0\]
\[\lim_{n, p \to \infty} \left| \Ea{\Psi\left(\mathcal{R}_g(\bm\Phi_{CK}, \bm y(\bm Z))\right)} - \Ea{\Psi\left(\mathcal{R}_g(\bm\Phi_{\text{CL}}, \bm y(\bm Z))\right)} \right| = 0\]
In particular, for any $\mathcal{E} \in \dR$, and denoting 
 $\overset{\dP}{\longrightarrow}$  the convergence in probability:
\begin{equation}\label{eq:in_proba_convergence}
\begin{split}
\hR_n^\star(\bm\Phi_{CK}, \bm y(\bm Z)) \overset{\dP}{\longrightarrow} \mathcal{E} \quad \text{if and only if} \quad \hR_n^\star(\bm \Phi_{CL}, \bm y(\bm Z)) \overset{\dP}{\longrightarrow} \mathcal{E}\\
\mathcal{R}_g^\star(\bm \Phi_{CK}, \bm y(\bm Z)) \overset{\dP}{\longrightarrow} \mathcal{E} \quad \text{if and only if} \quad \mathcal{R}_g(\bm \Phi_{CL}, \bm y(\bm Z)) \overset{\dP}{\longrightarrow} \mathcal{E},
\end{split}
\end{equation}
\end{enumerate}
\end{theorem}
\setcounter{theorem}{4}

Part (i) follows directly from Lemma \ref{lem:app:grad_exp_bounds}.
To prove the equivalence of training and generalization errors for the given direction, we rely on the framework of one-dimensional CLT (Central Limit Theorem), discussed in \cite{goldt_gaussian_2021}. One-dimensional CLT was recently shown to imply the universality of training and generalization errors for Random feature models in \cite{hu2022universality}. However, in our setting where we train the model, and as verified empirically in \cite{ba2022high}, a naive one-dimensional CLT with equivalent Gaussian features no longer holds.

Instead, we introduce a generalization termed ``conditional one-dimensional CLT", given by the following Lemma:
\begin{lemma}\label{lem:1dclt}
For any Lipschitz function $\varphi: \dR^2 \to \dR$, and $\forall k \in \R$: 
\begin{equation}\label{eq:one_dimensional_clt_main}
    \lim_{n, p \to \infty} \sup_{\btheta_1\in\cS_p,\btheta_2\in\cS^{d-1}} 
    \left| 
    \Ea{\varphi(\btheta_1^\top \phi_{\text{CK}}(\vec{z}),\btheta_2^\top\vec{z})\,\big|\,z_{\vec{v}}=k} - 
   \Ea{\varphi(\btheta^\top\phi_{\text{CL}}(\vec{z}),\btheta_2^\top\vec{z})\,\big|\,z_{\vec{v}}=k}
   \right| = 0
\end{equation}
where $\cS^{d-1}$ denotes the unit sphere in $\R^d$
\end{lemma}
\begin{proof}
For an input $\vec{z} \sim \mathcal{N}(0, I_d)$, we consider the decomposition $\vec{z}=z_{\vec{v}}\vec{v}+\vec{z}^\perp$
We note that conditioned on on $z_{\vec{v}}=k$, $\phi_{\text{CL}}(\vec{z}),\vec{z_\bot}$ is a Gaussian random variable.
Next, consider $\phi_{\text{CK}}(\vec{z})$. Our proof relies on the observation that while the features $\phi_{\text{CK}}(\vec{z})$ have complex non-linear dependence on $z_{\vec{v}}$, for a fixed value of  $z_{\vec{v}}$, they are equivalent to a random-features mapping applied to $\vec{z_\bot}$. Concretely, we have by Lemma \ref{lem:delta} that the weight matrix $\bW^{(1)}$ has the following spike+bulk decomposition (equation \eqref{eq:spike+bulk}):
\begin{equation}
    \bW^{(1)} =\eta \vec{u}\vec{v}^\top+\bW^{(0)} + \eta\Delta,
\end{equation}
where $\vec{u}=\frac{\mu_1}{p}\vec{a}$

Let $\bW^{\perp}$ denote the combined matrix $\bW^{(0)} + \eta\Delta$ with rows $\vec{w}^{\perp}_i$ for $i \in [p]$.

Lemma \ref{lem:norm_first_step} implies that there exist constants $c_i$ for $i \in [p]$ depending only on $a_i$ such that $\norm{\vec{w}^{\perp }_i}^2=c_i + \cO(\frac{\polylog d}{\sqrt{d}})$ with high-probability.
Define the following neuron-wise activation functions:
 \begin{equation}\label{eq:neuron_wise}
       \sigma_{i,z_{\vec{v}}}(u) =  \sigma(c_j u+\eta v_{\vec{z}})-\Eb{u}{\sigma(c_j u+\eta u_i z_{\vec{v}})},
 \end{equation}
where $i \in [p]$ denotes the index of the neuron and the expectation is w.r.t $z \sim \mathcal{N}(0,1)$.
Under the choise of symmetric initialization in Equation \eqref{eq:sample_archit}, it suffices to restrict ourselves to the first half $p/2$ neurons.

For a fixed value of $z_{\vec{v}}$, the feature map $\phi_{\text{CK}}(\bm z) = \sigma(\vec{W}^{1}\vec{z})$ is equivalent to a random features mapping with neurons $\sigma_{i,v_{\vec{z}}}$ applies to inputs $\vec{z}^\bot \in \R^d$ with approximately orthogonal weights $\bW^{\perp}$.
Consider the following events for some positive constants $C_1,C_2,C_3$:
\begin{align*}
    \cA_1 = \Set*{ \sup_{i, j \in [p/2]} \left| {\langle\vec{w}^{\perp }_i}, \vec{w}^{\perp}_j\rangle - c_i\delta_{ij} \right| \leq C_1\left(\frac{\polylog d}{d}\right)^{1/2}} \quad
    \cA_2 = \Set*{\norm{\bW^{\perp}}_{\mathrm{op}} \leq C_3(\polylog d)}
\end{align*}
 
We have, using Lemma \ref{lem:delta} and a union bound, that for $p,d=\Theta(n)$, $\Pr[\cA_1] \overset{n,d \rightarrow \infty}{\longrightarrow} 1$. Furthermore, part (ii) of Lemma 14 in \cite{ba2022high} implies that  $\Pr[\cA_2] \rightarrow 1$.
Next, we utilize Corollary 2  and Lemma 3 in \cite{hu2022universality}. 
Note that the  neuron wise activation functions \eqref{eq:neuron_wise} for a fixed value of $z_{\vec{v}}$ satisfy $\Eb{u}{\sigma_{i,v_{\vec{z}}}(u)}=0$. We relax the requirement of odd-activation in \cite{hu2022universality} by noting that $\phi_{CK},\phi_{CL}$ have exactly equivalent means and covariances as in  Theorem 6 of \cite{dandi2023universality}. 
\end{proof}
The above Lemma states that the one-dimensional projections of $\phi_{\text{CK}}(\vec{z})$ are asymptotically distributed as jointly Gaussian variables with $\vec{z_\bot}$. 

 \subsection{Conditional GET}

We now prove part (ii) of Theorem \ref{thm:cget} using Lemma \ref{lem:1dclt}. This relies on the universality of training and generalization errors between the given distribution and the ``conditional equivalent" distribution. The central idea of the proof again relies on a the isolation of the effects of the ``spikes" and the ``noise" in the features.

The technique presented here is also of independent interest for proving the universality of training, generalization errors in related setups such as with spiked-covariance inputs.

We utilize the following properties of the features :
\begin{lemma}\label{lem:get_lem_1}
For any fixed $z_{\vec{v}}$, the random variable $\phi_{CK}-\bmu(z_{\vec{v}})$ is sub-Gaussian with sub-Gaussian norm independent of $z_{\vec{v}}$ and $n$. 
\end{lemma}
\begin{proof}
The result follows from the assumption of uniform boundedness of the derivative of $\sigma^\star$ and the Lipschitz concentration of Gaussian variables.
\end{proof}
\begin{lemma}\label{lem:get_lem_2}
    There exists a constant $C$ such that the matrix $\bar{\Phi}_{CK}$ with rows $\phi_{CK}-\bmu(z_{\vec{v}})$ satisfies:
    \begin{equation}
        \Pr[\norm{\bar{\Phi}_{CK}} \geq K\sqrt{p}] \leq 2\exp(-Cn)
    \end{equation}
\end{lemma}
\begin{proof}
By Lemma \ref{lem:get_lem_1}, each row of $\bar{\Phi}_{CK}$ is sub-Gaussian. Therefore, the result follows from the concentration of spectral norm of matrices with independent sub-Gaussian rows  (Theorem 5.39 in \cite{vershynin2010introduction}).
\end{proof}

We start by proving certain properties of the optimal parameters $a_i$ upon the training of the second layer:
\begin{lemma}\label{lem:get_lem_3}
Let $\hat{a}_{\text{CK}}(\lambda)$  be the parameters obtained through ridge regression on features $\phi_{\text{CK}}(\vec{z_i})_{\{i=1,\cdots n\}}$ with regularization strength $\lambda$. Then, there exists a constants $C$ such that  with high probability as $n,d \rightarrow \infty$:
\begin{align}
    \frac{1}{n}\sum_{i=1}^n\left(\hat{a}_{\text{CK}}^\top\mu(z_{i,\vec{v}})\right)^2 \leq C
\end{align}
\end{lemma}

\begin{proof}
By assumption, $y_i(\vec{z})=\frac{1}{\sqrt{p}} a^\top \sigma(W \vec{z})$ with $\sigma'$ uniformly bounded. Therefore, from the concentration of Lipschitz functions of gaussian variables, $y_i(\vec{z})$ is sub-Gaussian. Thus $y^2_i(\vec{z})$ are sub-exponential variables. Using Bernstein's inequality \cite{vershynin2018high}, we obtain:
\begin{equation}\label{eq:bernstien}
    \Pr[\frac{1}{n}\sum_{i=1}^n(y_i)^2 -\Ea{(y_i)^2} > K] \leq 2\exp(-\min(c_1K,c_2K^2)n).
\end{equation}
For constants $c_1,c_2$.
\end{proof}

Let $\cA_{\by}$ denote the following event:
\begin{equation}
   \cA_{\by}= \Set*{ \frac{1}{n}\sum_{i=1}^n(y_i)^2 < C_1}.
\end{equation}

By Equation \eqref{eq:bernstien}, we have $\Pr{[\cA_{\by}]} \rightarrow 1$ as $n,d \rightarrow \infty$.

Let $\hat{\cR}(W, \vec{a})$ denote the empirical risk at given values of $\vec{a},W$. We have:
\begin{equation}
   \hat{\vec{a}} = \argmin_{\vec{a}}\hat{\cR}(W, \vec{a})= \argmin_{\vec{a}}{\frac{1}{2n}\sum_{i=1}^n(y_i-\vec{a}^\top\phi_k(\bz_i))^2}.
\end{equation}
We note that when $\vec{a}=\vec{0}$, we have:
\begin{equation}
    \hat{\cR}(W, \vec{0}) =  \frac{1}{n}\sum_{i=1}^n(y_i)^2.
\end{equation}
Since $\hat{\vec{a}}$ minimizes $\hat{\cR}(W, \vec{a})$, we must have:
\begin{equation}
  \hat{\cR}(W, \vec{\hat{a}}) \leq  \hat{\cR}(W, \vec{0}).
\end{equation}
We obtain:
\begin{align*}
    \frac{1}{2n}\sum_{i=1}^n(y_i-\vec{a}^\top\phi_k(\bz_i))^2 &\leq \frac{1}{n}\sum_{i=1}^n(y_i)^2\\
    \implies \frac{1}{2n}\sum_{i=1}^n (\vec{a}^\top\phi_k(\bz_i))^2 &\leq \frac{1}{n}\sum_{i=1}^n \vec{a}^\top\phi_k(\bz_i)y_i\\
     \implies \frac{1}{2n}\sum_{i=1}^n (\hat{\vec{a}}^\top\phi_k(\bz_i))^2 &\leq \sqrt{\frac{1}{n}\sum_{i=1}^n (\vec{a}^\top\phi_k(\bz_i))^2 }\sqrt{\frac{1}{n}\sum_{i=1}^n y_i^2 },
\end{align*}
where the last inequality follows from Cauchy-Schwarz.
Therefore:
\begin{align*}
    \sqrt{\frac{1}{n}\sum_{i=1}^n (\vec{a}^\top\phi_k(\bz_i))^2 } &\leq 2\sqrt{\frac{1}{n}\sum_{i=1}^n y_i^2}\\
    \frac{1}{n}\sum_{i=1}^n (\vec{a}^\top\mu(\bz_i))^2 + \frac{1}{n} \norm{\bar{\Phi}_{\text{CK}}^\top \vec{a}}^2_2  &\leq 4(\frac{1}{n}\sum_{i=1}^n y_i^2).
\end{align*}

Applying Lemma \ref{lem:get_lem_2} and $\Pr{[\cA_{\by}]} \overset{n,d \rightarrow}{\longrightarrow} 1$ then completes the proof.

Next, we prove the universality of the training, generalization error, conditioned on the values of the projections $z_{\vec{v}}$. This can be achieved through a number of techniques such as the Lindeberg's method in \cite{hu2022universality}. We utilize the result of  \cite{Montanari2022}, who apply the interpolation technique to continuously transform the inputs $\vec{x}_i$ to equivalent Gaussian vectors $\vec{g}_i$.

Instead, we interpolate between the features $\phi_{CK}(\vec{z})$ and $\phi_{CL}(\vec{z})$. Define:
\[ \bm u_{t, i} = \bmu(z_{i,\vec{v}}) + \cos(t)(\bm \phi_{CK}(\vec{z}) - \bmu(z_{i,\vec{v}})) + \sin(t)(\bm \phi_{CL}(\vec{z}) -\bmu(z_{i,\vec{v}})), \].

Let $\cA_1$ denote the event:
\begin{equation}
\cA_1= \Set*{\frac{1}{n}\sum_{i=1}^n\left(\hat{a}_{\text{CK}}^\top\mu(z_{i,\vec{v}})\right)^2 \leq C_1}
\end{equation}
Under the above interpolation path, we generalize Theorem 1 in \cite{Montanari2022} to obtain that for any bounded Lipschitz function $\Phi: \dR \to \dR$:
\begin{small}
\begin{equation}\label{eq:conv_uni}
\lim_{n, p \to \infty} \sup_{v_{\bz_1},\cdots,v_{\bz_n}}\left|\Ea{\mathbf{1} _{\cA_{1}}\Phi\left(\hR_n^\star(\bm \Phi_{CK}, \bm y(\bm Z))\right) -\mathbf{1} _{\cA_{1}} \Phi\left(\hR_n^\star(\bm \Phi_{\text{CL}}, \bm y(\bm Z))\right)\mid v_{\bz_1},\cdots,v_{\bz_n} } \right| = 0.
\end{equation}
\end{small}

Below, we explain the modifications to Theorem 1 in \cite{Montanari2022} that allow its applicability to our setting:
\begin{enumerate}
    \item We replace equation (12) in  Assumption 5 tof \cite{Montanari2022} by the conditional 1d-CLT (Lemma \ref{lem:1dclt}). This is similar to the conditioning utilized in \cite{dandi2023universality} for proving the universality in mixture models.
    \item Our target function $y = f^\star(\bz)$ depends on the projection along the spike $v_{\bm z}$ as well as the orthogonal component $\vec{z}^{\bot}$. Since we condition on the values of $v_{\bm z}$, their dependence can be absorbed into the loss function for each input $\vec{z}_i^{\bot}$
    \item While Theorem 1 in \cite{Montanari2022} does not allow a dependence of the labels on the latent variables $\bz$, such a target function can be incorporated by considering the inputs to be the joint variables in $(\bm \Phi_{\text{CK}}(\bz),\bz) \in \R^{p+d}$ and constraining the parameters to have $0$ components along the last $d$ directions.
    \item The event $\cA_1$  and Lemma \ref{lem:get_lem_2} ensure that Lemmas 5 and 6 in \cite{Montanari2022} hold under the presence of variable and unbounded means across samples $\bmu(z_{i,\vec{v}})$.
\end{enumerate}

Next, using the Law of total expectation and Equation \ref{eq:conv_uni}, we obtain:
\[\lim_{n, p \to \infty} \left|\Ea{\mathbf{1} _{\cA_{1}}\Phi\left(\hR_n^\star(\bm \Phi_{CK}, \bm y(\bm Z))\right)} - \Ea{\mathbf{1} _{\cA_{1}}\bm \Phi\left(\hR_n^\star(\bm \phi_{\text{CL}}, \bm y(\bm Z))\right)} \right| = 0.\]
Finally, we note Lemma \ref{lem:get_lem_3} implies that $\Pr[\cA^c_{1}] \rightarrow 0$. Since $\Phi$ is bounded, we have that
\[\lim_{n, p \to \infty} \left|\Ea{\mathbf{1} _{\cA^c_{1}}\Phi\left(\hR_n^\star(\bm \Phi_{CK}, \bm y(\bm Z))\right)} - \Ea{\mathbf{1}_{\cA^c_{1}}\bm \Phi\left(\hR_n^\star(\bm \phi_{\text{CL}}, \bm y(\bm Z))\right)} \right| = 0.\]
This completes the proof of Theorem \ref{thm:cget}.

 \subsection{Generalization Error Lower Bounds: Proof of Corollary \ref{corr:lower_bound}}

From Theorem \ref{thm:cget}, it is sufficient to prove the lower bound for the generalization error corresponding to the equivalent features $\phi_{\text{CL}}(\vec{z})$. Let $\vec{Z}$ denote the input design matrix with rows $\vec{z}_i$.
Similarly, let $\vec{\Xi}$ denote the matrix with rows containing $n$ independent Gaussian vectors, denoting the uncorrelated noise in the equivalent conditional Gaussian features defined by equation \eqref{eq:cget_features}.
We have that $\hat{a}_{\text{CL}}(\lambda,\vec{Z},\vec{\Xi})=\left(\vec{\Phi}_{CL}^\top\vec{\Phi}_{CL} + \frac{\lambda n}{N}\bI\right)^{-1} \vec{\Phi}_{CL}^\top{\vec{y}}$.
The generalization error can then be expressed as:
\begin{align*}
    \cR(W,\hat{a}_{\text{CL}})&=\Eb{\vec{z},\xi}{(f^\star(\vec{z})-\hat{a}_{\text{CL}}(\lambda,\vec{Z},\vec{\Xi})^\top\phi_{\text{CL}}(\vec{z}))^2}\\
    &= \Eb{\xi}{\Eb{\vec{z}}{(f^\star(\vec{z})-\hat{a}_{\text{CL}}(\lambda,\vec{Z},\vec{\Xi})^\top\phi_{\text{CL}}(\vec{z}))^2}},
\end{align*}
where the last line follows from Fubini's theorem.

We note that the predictor $\hat{f}(\vec{z})=\frac{1}{\sqrt{p}}\hat{a}^\top_{\text{CL}}\phi_{\text{CL}}(\vec{z})$ 
is a linear function of $\vec{z^\bot}$ with coefficients dependent on $z_{\vec{v}}$. Therefore, $\hat{f}(\vec{z}) \in \mathcal{P}_{\vec{v},1}$.

For a fixed value of $\xi$, we obtain the following expression for the generalization error:
\begin{align*}
    \Eb{\vec{z}}{(f^\star(\vec{z})-\hat{f}(\vec{z}))^2} &= \norm{f^\star-\hat{f}}^2.
    \\ &=\norm{P_{v,1}(f^\star-\hat{f})}^2+ \norm{P_{v,>1}(f^\star-\hat{f})}^2
    \\ &\geq  \norm{P_{v,>1}(f^\star)^2}^2,
\end{align*}
where we used that $P_{v,>1}(f^\star-\hat{f})=P_{v,>1}(f^\star)$.
Since the projection of $f^\star$ on the orthogonal complement of the teacher subspace is $0$, Corollary \ref{corr:lower_bound} then follows using $P_{V^\star}\vec{v} \overset{\dP}{\longrightarrow} \frac{\mu}{\sqrt{p}}\vec{v}^\star$ and the dominated convergence theorem for the RHS.

%% file: reviewed_jmlr_main.bbl
\begin{thebibliography}{70}
\providecommand{\natexlab}[1]{#1}
\providecommand{\url}[1]{\texttt{#1}}
\expandafter\ifx\csname urlstyle\endcsname\relax
  \providecommand{\doi}[1]{doi: #1}\else
  \providecommand{\doi}{doi: \begingroup \urlstyle{rm}\Url}\fi

\bibitem[Abbe et~al.(2021)Abbe, Boix-Adsera, Brennan, Bresler, and Nagaraj]{abbe2021staircase}
Emmanuel Abbe, Enric Boix-Adsera, Matthew~S Brennan, Guy Bresler, and Dheeraj Nagaraj.
\newblock The staircase property: How hierarchical structure can guide deep learning.
\newblock \emph{Advances in Neural Information Processing Systems}, 34:\penalty0 26989--27002, 2021.

\bibitem[Abbe et~al.(2022)Abbe, Adsera, and Misiakiewicz]{abbe2022merged}
Emmanuel Abbe, Enric~Boix Adsera, and Theodor Misiakiewicz.
\newblock The merged-staircase property: a necessary and nearly sufficient condition for sgd learning of sparse functions on two-layer neural networks.
\newblock In \emph{Conference on Learning Theory}, pages 4782--4887. PMLR, 2022.

\bibitem[Abbe et~al.(2023)Abbe, Boix-Adsera, and Misiakiewicz]{abbe2023sgd}
Emmanuel Abbe, Enric Boix-Adsera, and Theodor Misiakiewicz.
\newblock Sgd learning on neural networks: leap complexity and saddle-to-saddle dynamics, 2023.

\bibitem[Arnaboldi et~al.(2023)Arnaboldi, Stephan, Krzakala, and Loureiro]{arnaboldi.stephan.ea_2023_high}
Luca Arnaboldi, Ludovic Stephan, Florent Krzakala, and Bruno Loureiro.
\newblock From high-dimensional \& mean-field dynamics to dimensionless {ODEs}: {A} unifying approach to {SGD} in two-layers networks, February 2023.
\newblock arXiv:2302.05882 [cond-mat, stat] type: article.

\bibitem[Atanasov et~al.(2022)Atanasov, Bordelon, and Pehlevan]{atanasov2022neural}
Alexander Atanasov, Blake Bordelon, and Cengiz Pehlevan.
\newblock Neural networks as kernel learners: The silent alignment effect.
\newblock In \emph{International Conference on Learning Representations}, 2022.

\bibitem[Ba et~al.(2022)Ba, Erdogdu, Suzuki, Wang, Wu, and Yang]{ba2022high}
Jimmy Ba, Murat~A Erdogdu, Taiji Suzuki, Zhichao Wang, Denny Wu, and Greg Yang.
\newblock High-dimensional asymptotics of feature learning: How one gradient step improves the representation.
\newblock In S.~Koyejo, S.~Mohamed, A.~Agarwal, D.~Belgrave, K.~Cho, and A.~Oh, editors, \emph{Advances in Neural Information Processing Systems}, volume~35, pages 37932--37946. Curran Associates, Inc., 2022.

\bibitem[Bach(2017)]{bach2017breaking}
Francis Bach.
\newblock Breaking the curse of dimensionality with convex neural networks.
\newblock \emph{The Journal of Machine Learning Research}, 18\penalty0 (1):\penalty0 629--681, 2017.

\bibitem[{Ben Arous} et~al.(2021){Ben Arous}, Gheissari, and Jagannath]{BenArous2021}
Gerard {Ben Arous}, Reza Gheissari, and Aukosh Jagannath.
\newblock Online stochastic gradient descent on non-convex losses from high-dimensional inference.
\newblock \emph{Journal of Machine Learning Research}, 22\penalty0 (106):\penalty0 1--51, 2021.

\bibitem[Ben~Arous et~al.(2022)Ben~Arous, Gheissari, and Jagannath]{ben2022high}
Gerard Ben~Arous, Reza Gheissari, and Aukosh Jagannath.
\newblock High-dimensional limit theorems for sgd: Effective dynamics and critical scaling.
\newblock \emph{Advances in Neural Information Processing Systems}, 35:\penalty0 25349--25362, 2022.

\bibitem[Berthier et~al.(2023)Berthier, Montanari, and Zhou]{berthier2023learning}
Raphaël Berthier, Andrea Montanari, and Kangjie Zhou.
\newblock Learning time-scales in two-layers neural networks, 2023.

\bibitem[Bietti et~al.(2022)Bietti, Bruna, Sanford, and Song]{Bietti2022}
Alberto Bietti, Joan Bruna, Clayton Sanford, and Min~Jae Song.
\newblock Learning single-index models with shallow neural networks.
\newblock In S.~Koyejo, S.~Mohamed, A.~Agarwal, D.~Belgrave, K.~Cho, and A.~Oh, editors, \emph{Advances in Neural Information Processing Systems}, volume~35, pages 9768--9783. Curran Associates, Inc., 2022.

\bibitem[Bordelon and Pehlevan(2023)]{bordelon2023dynamics}
Blake Bordelon and Cengiz Pehlevan.
\newblock Dynamics of finite width kernel and prediction fluctuations in mean field neural networks, 2023.

\bibitem[Bordelon et~al.(2020)Bordelon, Canatar, and Pehlevan]{bordelon20a}
Blake Bordelon, Abdulkadir Canatar, and Cengiz Pehlevan.
\newblock Spectrum dependent learning curves in kernel regression and wide neural networks.
\newblock In Hal~Daumé III and Aarti Singh, editors, \emph{Proceedings of the 37th International Conference on Machine Learning}, volume 119 of \emph{Proceedings of Machine Learning Research}, pages 1024--1034. PMLR, 13--18 Jul 2020.

\bibitem[Bosch et~al.(2023)Bosch, Panahi, and Hassibi]{bosch2023precise}
David Bosch, Ashkan Panahi, and Babak Hassibi.
\newblock Precise asymptotic analysis of deep random feature models, 2023.

\bibitem[Boursier et~al.(2022)Boursier, Pillaud-Vivien, and Flammarion]{boursier2022gradient}
Etienne Boursier, Loucas Pillaud-Vivien, and Nicolas Flammarion.
\newblock Gradient flow dynamics of shallow relu networks for square loss and orthogonal inputs.
\newblock \emph{Advances in Neural Information Processing Systems}, 35:\penalty0 20105--20118, 2022.

\bibitem[Boyd et~al.(2006)Boyd, Ghosh, Prabhakar, and Shah]{boyd_2006_randomized}
S.~Boyd, A.~Ghosh, B.~Prabhakar, and D.~Shah.
\newblock Randomized gossip algorithms.
\newblock \emph{IEEE Transactions on Information Theory}, 52\penalty0 (6):\penalty0 2508--2530, June 2006.
\newblock ISSN 1557-9654.
\newblock \doi{10.1109/TIT.2006.874516}.

\bibitem[Canatar et~al.(2021)Canatar, Bordelon, and Pehlevan]{Canatar2021}
Abdulkadir Canatar, Blake Bordelon, and Cengiz Pehlevan.
\newblock Spectral bias and task-model alignment explain generalization in kernel regression and infinitely wide neural networks.
\newblock \emph{Nature Communications}, 12\penalty0 (1):\penalty0 2914, May 2021.
\newblock ISSN 2041-1723.
\newblock \doi{10.1038/s41467-021-23103-1}.

\bibitem[Chizat and Bach(2018)]{chizat2018global}
Lenaic Chizat and Francis Bach.
\newblock On the global convergence of gradient descent for over-parameterized models using optimal transport.
\newblock \emph{Advances in neural information processing systems}, 31, 2018.

\bibitem[Chizat et~al.(2019)Chizat, Oyallon, and Bach]{chizat_2019_lazy}
Lénaïc Chizat, Edouard Oyallon, and Francis Bach.
\newblock On {Lazy} {Training} in {Differentiable} {Programming}.
\newblock In \emph{Advances in {Neural} {Information} {Processing} {Systems}}, volume~32. Curran Associates, Inc., 2019.

\bibitem[Cui et~al.(2021)Cui, Loureiro, Krzakala, and Zdeborov\'{a}]{Cui2021}
Hugo Cui, Bruno Loureiro, Florent Krzakala, and Lenka Zdeborov\'{a}.
\newblock Generalization error rates in kernel regression: The crossover from the noiseless to noisy regime.
\newblock In M.~Ranzato, A.~Beygelzimer, Y.~Dauphin, P.S. Liang, and J.~Wortman Vaughan, editors, \emph{Advances in Neural Information Processing Systems}, volume~34, pages 10131--10143. Curran Associates, Inc., 2021.

\bibitem[Cui et~al.(2022)Cui, Loureiro, Krzakala, and Zdeborov{\'a}]{Cui2022}
Hugo Cui, Bruno Loureiro, Florent Krzakala, and Lenka Zdeborov{\'a}.
\newblock Error rates for kernel classification under source and capacity conditions, 2022.

\bibitem[Damian et~al.(2023)Damian, Nichani, Ge, and Lee]{damian_2023_smoothing}
Alex Damian, Eshaan Nichani, Rong Ge, and Jason~D. Lee.
\newblock Smoothing the {Landscape} {Boosts} the {Signal} for {SGD}: {Optimal} {Sample} {Complexity} for {Learning} {Single} {Index} {Models}.
\newblock Technical report, May 2023.
\newblock arXiv:2305.10633 [cs, math, stat] type: article.

\bibitem[Damian et~al.(2022)Damian, Lee, and Soltanolkotabi]{damian2022neural}
Alexandru Damian, Jason Lee, and Mahdi Soltanolkotabi.
\newblock Neural networks can learn representations with gradient descent.
\newblock In Po-Ling Loh and Maxim Raginsky, editors, \emph{Proceedings of Thirty Fifth Conference on Learning Theory}, volume 178 of \emph{Proceedings of Machine Learning Research}, pages 5413--5452. PMLR, 02--05 Jul 2022.

\bibitem[Dandi et~al.(2023)Dandi, Stephan, Krzakala, Loureiro, and Zdeborová]{dandi2023universality}
Yatin Dandi, Ludovic Stephan, Florent Krzakala, Bruno Loureiro, and Lenka Zdeborová.
\newblock Universality laws for gaussian mixtures in generalized linear models, 2023.

\bibitem[De~Lathauwer et~al.(2000)De~Lathauwer, De~Moor, and Vandewalle]{de2000multilinear}
Lieven De~Lathauwer, Bart De~Moor, and Joos Vandewalle.
\newblock A multilinear singular value decomposition.
\newblock \emph{SIAM journal on Matrix Analysis and Applications}, 21\penalty0 (4):\penalty0 1253--1278, 2000.

\bibitem[Devroye et~al.(2013)Devroye, Gy{\"o}rfi, and Lugosi]{devroye2013probabilistic}
Luc Devroye, L{\'a}szl{\'o} Gy{\"o}rfi, and G{\'a}bor Lugosi.
\newblock \emph{A probabilistic theory of pattern recognition}, volume~31.
\newblock Springer Science \& Business Media, 2013.

\bibitem[Dhifallah and Lu(2020)]{Dhifallah2020}
Oussama Dhifallah and Yue~M. Lu.
\newblock A precise performance analysis of learning with random features, 2020.

\bibitem[Dietrich et~al.(1999)Dietrich, Opper, and Sompolinsky]{Dietrich1999}
Rainer Dietrich, Manfred Opper, and Haim Sompolinsky.
\newblock Statistical mechanics of support vector networks.
\newblock \emph{Phys. Rev. Lett.}, 82:\penalty0 2975--2978, Apr 1999.
\newblock \doi{10.1103/PhysRevLett.82.2975}.

\bibitem[Donhauser et~al.(2021)Donhauser, Wu, and Yang]{Donhauser2021}
Konstantin Donhauser, Mingqi Wu, and Fanny Yang.
\newblock How rotational invariance of common kernels prevents generalization in high dimensions.
\newblock In Marina Meila and Tong Zhang, editors, \emph{Proceedings of the 38th International Conference on Machine Learning}, volume 139 of \emph{Proceedings of Machine Learning Research}, pages 2804--2814. PMLR, 18--24 Jul 2021.

\bibitem[Dudeja and Hsu(2018)]{pmlr-v75-dudeja18a}
Rishabh Dudeja and Daniel Hsu.
\newblock Learning single-index models in gaussian space.
\newblock In Sébastien Bubeck, Vianney Perchet, and Philippe Rigollet, editors, \emph{Proceedings of the 31st Conference On Learning Theory}, volume~75 of \emph{Proceedings of Machine Learning Research}, pages 1887--1930. PMLR, 06--09 Jul 2018.
\newblock URL \url{https://proceedings.mlr.press/v75/dudeja18a.html}.

\bibitem[Gerace et~al.(2020)Gerace, Loureiro, Krzakala, Mezard, and Zdeborova]{gerace_generalisation_2020}
Federica Gerace, Bruno Loureiro, Florent Krzakala, Marc Mezard, and Lenka Zdeborova.
\newblock Generalisation error in learning with random features and the hidden manifold model.
\newblock In Hal~Daumé III and Aarti Singh, editors, \emph{Proceedings of the 37th International Conference on Machine Learning}, volume 119 of \emph{Proceedings of Machine Learning Research}, pages 3452--3462. PMLR, 13--18 Jul 2020.

\bibitem[Ghorbani et~al.(2019)Ghorbani, Mei, Misiakiewicz, and Montanari]{Ghorbani2019}
Behrooz Ghorbani, Song Mei, Theodor Misiakiewicz, and Andrea Montanari.
\newblock Limitations of lazy training of two-layers neural network.
\newblock In H.~Wallach, H.~Larochelle, A.~Beygelzimer, F.~d\textquotesingle Alch\'{e}-Buc, E.~Fox, and R.~Garnett, editors, \emph{Advances in Neural Information Processing Systems}, volume~32. Curran Associates, Inc., 2019.

\bibitem[Ghorbani et~al.(2020)Ghorbani, Mei, Misiakiewicz, and Montanari]{Ghorbani2020}
Behrooz Ghorbani, Song Mei, Theodor Misiakiewicz, and Andrea Montanari.
\newblock When do neural networks outperform kernel methods?
\newblock In H.~Larochelle, M.~Ranzato, R.~Hadsell, M.F. Balcan, and H.~Lin, editors, \emph{Advances in Neural Information Processing Systems}, volume~33, pages 14820--14830. Curran Associates, Inc., 2020.

\bibitem[Goldt et~al.(2022)Goldt, Loureiro, Reeves, Krzakala, Mezard, and Zdeborova]{goldt_gaussian_2021}
Sebastian Goldt, Bruno Loureiro, Galen Reeves, Florent Krzakala, Marc Mezard, and Lenka Zdeborova.
\newblock The gaussian equivalence of generative models for learning with shallow neural networks.
\newblock In Joan Bruna, Jan Hesthaven, and Lenka Zdeborova, editors, \emph{Proceedings of the 2nd Mathematical and Scientific Machine Learning Conference}, volume 145 of \emph{Proceedings of Machine Learning Research}, pages 426--471. PMLR, 16--19 Aug 2022.

\bibitem[Gotze et~al.(2019)Gotze, Sambale, and Sinulis]{Gotze2019ConcentrationIF}
Friedrich Gotze, Holger Sambale, and Arthur Sinulis.
\newblock Concentration inequalities for polynomials in $\alpha$-sub-exponential random variables.
\newblock \emph{Electronic Journal of Probability}, 2019.

\bibitem[Goyal et~al.(2017)Goyal, Doll{\'a}r, Girshick, Noordhuis, Wesolowski, Kyrola, Tulloch, Jia, and He]{goyal2017accurate}
Priya Goyal, Piotr Doll{\'a}r, Ross Girshick, Pieter Noordhuis, Lukasz Wesolowski, Aapo Kyrola, Andrew Tulloch, Yangqing Jia, and Kaiming He.
\newblock Accurate, large minibatch sgd: Training imagenet in 1 hour.
\newblock \emph{arXiv preprint arXiv:1706.02677}, 2017.

\bibitem[Grad(1949)]{grad_1949_note}
Harold Grad.
\newblock Note on {N}-dimensional hermite polynomials.
\newblock \emph{Communications on Pure and Applied Mathematics}, 2\penalty0 (4):\penalty0 325--330, 1949.
\newblock ISSN 1097-0312.
\newblock \doi{10.1002/cpa.3160020402}.

\bibitem[Greub(2012)]{greub2012multilinear}
W.H. Greub.
\newblock \emph{Multilinear Algebra}.
\newblock Grundlehren der mathematischen Wissenschaften. Springer Berlin Heidelberg, 2012.
\newblock ISBN 9783662007952.

\bibitem[Hu and Lu(2022)]{hu2022universality}
Hong Hu and Yue~M Lu.
\newblock Universality laws for high-dimensional learning with random features.
\newblock \emph{IEEE Transactions on Information Theory}, 2022.

\bibitem[Jacot et~al.(2021)Jacot, Ged, {\c{S}}im{\c{s}}ek, Hongler, and Gabriel]{jacot2021saddle}
Arthur Jacot, Fran{\c{c}}ois Ged, Berfin {\c{S}}im{\c{s}}ek, Cl{\'e}ment Hongler, and Franck Gabriel.
\newblock Saddle-to-saddle dynamics in deep linear networks: Small initialization training, symmetry, and sparsity.
\newblock \emph{arXiv preprint arXiv:2106.15933}, 2021.

\bibitem[Janson(1997)]{janson1997gaussian}
Svante Janson.
\newblock \emph{Gaussian hilbert spaces}.
\newblock Number 129. Cambridge university press, 1997.

\bibitem[Kalimeris et~al.(2019)Kalimeris, Kaplun, Nakkiran, Edelman, Yang, Barak, and Zhang]{kalimeris2019sgd}
Dimitris Kalimeris, Gal Kaplun, Preetum Nakkiran, Benjamin Edelman, Tristan Yang, Boaz Barak, and Haofeng Zhang.
\newblock Sgd on neural networks learns functions of increasing complexity.
\newblock \emph{Advances in neural information processing systems}, 32, 2019.

\bibitem[Klenke(2013)]{klenke2013probability}
Achim Klenke.
\newblock \emph{Probability theory: a comprehensive course}.
\newblock Springer Science \& Business Media, 2013.

\bibitem[Ledoux and Talagrand(1991)]{ledoux_1991_probability}
Michel Ledoux and Michel Talagrand.
\newblock \emph{Probability in {Banach} {Spaces}: {Isoperimetry} and {Processes}}.
\newblock Springer-Verlag, 1991.
\newblock ISBN 9780387520131.
\newblock Google-Books-ID: juC1QgAACAAJ.

\bibitem[Li et~al.(2020)Li, Fan, Tse, and Lin]{li2020review}
Li~Li, Yuxi Fan, Mike Tse, and Kuo-Yi Lin.
\newblock A review of applications in federated learning.
\newblock \emph{Computers \& Industrial Engineering}, 149:\penalty0 106854, 2020.

\bibitem[Loureiro et~al.(2021)Loureiro, Gerbelot, Cui, Goldt, Krzakala, Mezard, and Zdeborov\'{a}]{loureiro_learning_2021}
Bruno Loureiro, Cedric Gerbelot, Hugo Cui, Sebastian Goldt, Florent Krzakala, Marc Mezard, and Lenka Zdeborov\'{a}.
\newblock Learning curves of generic features maps for realistic datasets with a teacher-student model.
\newblock In M.~Ranzato, A.~Beygelzimer, Y.~Dauphin, P.S. Liang, and J.~Wortman Vaughan, editors, \emph{Advances in Neural Information Processing Systems}, volume~34, pages 18137--18151. Curran Associates, Inc., 2021.

\bibitem[Loureiro et~al.(2022)Loureiro, Gerbelot, Refinetti, Sicuro, and Krzakala]{loureiro22a}
Bruno Loureiro, Cedric Gerbelot, Maria Refinetti, Gabriele Sicuro, and Florent Krzakala.
\newblock Fluctuations, bias, variance \&; ensemble of learners: Exact asymptotics for convex losses in high-dimension.
\newblock In Kamalika Chaudhuri, Stefanie Jegelka, Le~Song, Csaba Szepesvari, Gang Niu, and Sivan Sabato, editors, \emph{Proceedings of the 39th International Conference on Machine Learning}, volume 162 of \emph{Proceedings of Machine Learning Research}, pages 14283--14314. PMLR, 17--23 Jul 2022.

\bibitem[Mei and Montanari(2022)]{mei_generalization_2022}
Song Mei and Andrea Montanari.
\newblock The generalization error of random features regression: Precise asymptotics and the double descent curve.
\newblock \emph{Communications on Pure and Applied Mathematics}, 75\penalty0 (4):\penalty0 667--766, 2022.

\bibitem[Mei et~al.(2018)Mei, Montanari, and Nguyen]{mei2018mean}
Song Mei, Andrea Montanari, and Phan-Minh Nguyen.
\newblock A mean field view of the landscape of two-layer neural networks.
\newblock \emph{Proceedings of the National Academy of Sciences}, 115\penalty0 (33):\penalty0 E7665--E7671, 2018.

\bibitem[Mei et~al.(2019)Mei, Misiakiewicz, and Montanari]{mei2019mean}
Song Mei, Theodor Misiakiewicz, and Andrea Montanari.
\newblock Mean-field theory of two-layers neural networks: dimension-free bounds and kernel limit.
\newblock In \emph{Conference on learning theory}, pages 2388--2464. PMLR, 2019.

\bibitem[Mei et~al.(2022)Mei, Misiakiewicz, and Montanari]{Mei2023}
Song Mei, Theodor Misiakiewicz, and Andrea Montanari.
\newblock Generalization error of random feature and kernel methods: Hypercontractivity and kernel matrix concentration.
\newblock \emph{Applied and Computational Harmonic Analysis}, 59:\penalty0 3--84, 2022.
\newblock ISSN 1063-5203.
\newblock \doi{https://doi.org/10.1016/j.acha.2021.12.003}.
\newblock Special Issue on Harmonic Analysis and Machine Learning.

\bibitem[Mityagin(2020)]{mityagin2020zero}
Boris~Samuilovich Mityagin.
\newblock The zero set of a real analytic function.
\newblock \emph{Mathematical Notes}, 107\penalty0 (3-4):\penalty0 529--530, 2020.

\bibitem[Montanari and Saeed(2022)]{Montanari2022}
Andrea Montanari and Basil~N. Saeed.
\newblock Universality of empirical risk minimization.
\newblock In Po-Ling Loh and Maxim Raginsky, editors, \emph{Proceedings of Thirty Fifth Conference on Learning Theory}, volume 178 of \emph{Proceedings of Machine Learning Research}, pages 4310--4312. PMLR, 02--05 Jul 2022.

\bibitem[Naveh and Ringel(2021)]{Naveh2021}
Gadi Naveh and Zohar Ringel.
\newblock A self consistent theory of gaussian processes captures feature learning effects in finite cnns.
\newblock In M.~Ranzato, A.~Beygelzimer, Y.~Dauphin, P.S. Liang, and J.~Wortman Vaughan, editors, \emph{Advances in Neural Information Processing Systems}, volume~34, pages 21352--21364. Curran Associates, Inc., 2021.

\bibitem[Opper and Urbanczik(2001)]{Opper2001}
M.~Opper and R.~Urbanczik.
\newblock Universal learning curves of support vector machines.
\newblock \emph{Phys. Rev. Lett.}, 86:\penalty0 4410--4413, May 2001.
\newblock \doi{10.1103/PhysRevLett.86.4410}.

\bibitem[Pena and Montgomery-Smith(1995)]{pena_1995_decoupling}
Victor H. de~la Pena and S.~J. Montgomery-Smith.
\newblock Decoupling {Inequalities} for the {Tail} {Probabilities} of {Multivariate} \${U}\$-{Statistics}.
\newblock \emph{The Annals of Probability}, 23\penalty0 (2):\penalty0 806--816, April 1995.
\newblock ISSN 0091-1798, 2168-894X.
\newblock \doi{10.1214/aop/1176988291}.

\bibitem[Petrini et~al.(2022)Petrini, Cagnetta, Vanden-Eijnden, and Wyart]{petrini2022learning}
Leonardo Petrini, Francesco Cagnetta, Eric Vanden-Eijnden, and Matthieu Wyart.
\newblock Learning sparse features can lead to overfitting in neural networks, 2022.

\bibitem[Polyak(2005)]{polyak2005feynman}
Michael Polyak.
\newblock Feynman diagrams for pedestrians and mathematicians.
\newblock \emph{Graphs and patterns in mathematics and theoretical physics}, 73:\penalty0 15--42, 2005.

\bibitem[Rotskoff and Vanden-Eijnden(2022)]{rotskoff2018trainability}
Grant Rotskoff and Eric Vanden-Eijnden.
\newblock Trainability and accuracy of artificial neural networks: An interacting particle system approach.
\newblock \emph{Communications on Pure and Applied Mathematics}, 75\penalty0 (9):\penalty0 1889--1935, 2022.
\newblock \doi{https://doi.org/10.1002/cpa.22074}.

\bibitem[Rudi and Rosasco(2017)]{rudi2017generalization}
Alessandro Rudi and Lorenzo Rosasco.
\newblock Generalization properties of learning with random features.
\newblock \emph{Advances in neural information processing systems}, 30, 2017.

\bibitem[Saad and Solla(1995)]{saad.solla_1995_line}
David Saad and Sara~A. Solla.
\newblock On-line learning in soft committee machines.
\newblock \emph{Physical Review E}, 52\penalty0 (4):\penalty0 4225--4243, October 1995.
\newblock \doi{10.1103/PhysRevE.52.4225}.

\bibitem[Schröder et~al.(2023)Schröder, Cui, Dmitriev, and Loureiro]{schroder2023deterministic}
Dominik Schröder, Hugo Cui, Daniil Dmitriev, and Bruno Loureiro.
\newblock Deterministic equivalent and error universality of deep random features learning, 2023.

\bibitem[Seroussi et~al.(2023)Seroussi, Naveh, and Ringel]{Seroussi2023}
Inbar Seroussi, Gadi Naveh, and Zohar Ringel.
\newblock Separation of scales and a thermodynamic description of feature learning in some cnns.
\newblock \emph{Nature Communications}, 14\penalty0 (1):\penalty0 908, Feb 2023.
\newblock ISSN 2041-1723.
\newblock \doi{10.1038/s41467-023-36361-y}.

\bibitem[Simon et~al.(2022)Simon, Dickens, Karkada, and DeWeese]{simon2022eigenlearning}
James~B. Simon, Madeline Dickens, Dhruva Karkada, and Michael~R. DeWeese.
\newblock The eigenlearning framework: A conservation law perspective on kernel regression and wide neural networks, 2022.

\bibitem[Spigler et~al.(2020)Spigler, Geiger, and Wyart]{Spigler2020}
Stefano Spigler, Mario Geiger, and Matthieu Wyart.
\newblock Asymptotic learning curves of kernel methods: empirical data versus teacher–student paradigm.
\newblock \emph{Journal of Statistical Mechanics: Theory and Experiment}, 2020\penalty0 (12):\penalty0 124001, dec 2020.
\newblock \doi{10.1088/1742-5468/abc61d}.

\bibitem[van~der Vaart and Wellner(1996)]{vaart_1996_weak}
Aad van~der Vaart and Jon Wellner.
\newblock \emph{Weak {Convergence} and {Empirical} {Processes}: {With} {Applications} to {Statistics}}.
\newblock Springer Science \& Business Media, March 1996.
\newblock ISBN 9780387946405.
\newblock Google-Books-ID: OCenCW9qmp4C.

\bibitem[Vershynin(2010)]{vershynin2010introduction}
Roman Vershynin.
\newblock Introduction to the non-asymptotic analysis of random matrices.
\newblock \emph{arXiv preprint arXiv:1011.3027}, 2010.

\bibitem[Vershynin(2018)]{vershynin2018high}
Roman Vershynin.
\newblock \emph{High-dimensional probability: An introduction with applications in data science}, volume~47.
\newblock Cambridge university press, 2018.

\bibitem[Xiao et~al.(2022)Xiao, Hu, Misiakiewicz, Lu, and Pennington]{xiao2022precise}
Lechao Xiao, Hong Hu, Theodor Misiakiewicz, Yue Lu, and Jeffrey Pennington.
\newblock Precise learning curves and higher-order scalings for dot-product kernel regression.
\newblock \emph{Advances in Neural Information Processing Systems}, 35:\penalty0 4558--4570, 2022.

\bibitem[Zdeborov{\'a}(2020)]{zdeborova2020understanding}
Lenka Zdeborov{\'a}.
\newblock Understanding deep learning is also a job for physicists.
\newblock \emph{Nature Physics}, 16\penalty0 (6):\penalty0 602--604, 2020.

\end{thebibliography}
